\documentclass[10pt,journal,compsoc]{IEEEtran}
\usepackage{array}
\usepackage{textcomp}
\usepackage{stfloats}
\usepackage{verbatim}
\usepackage{cite}
\usepackage{amsmath,amssymb,amsfonts,amsthm}
\usepackage{algorithm}
\usepackage{algorithmicx}
\usepackage[noend]{algpseudocode}
\usepackage{graphicx}
\usepackage{textcomp}
\usepackage{xcolor}
\usepackage{url}
\usepackage{enumitem}
\usepackage{hyperref}
\usepackage{multicol}
\usepackage{multirow}
\usepackage{booktabs}
\usepackage{subfigure}
\usepackage{todonotes}
\usepackage{xcolor}
\usepackage{xspace}
\usepackage[normalem]{ulem}

\theoremstyle{plain}
\newtheorem{theorem}{Theorem}[section]
\newtheorem{proposition}[theorem]{Proposition}

\theoremstyle{definition}

\theoremstyle{remark}

\newcommand{\propositionname}{Prop.}

\newcommand{\rulesep}{\unskip\ \vrule\ }

\newcommand{\method}{\textsc{Diversify}\xspace}
\newcommand{\methodmcp}{{\method-MCP}\xspace}
\newcommand{\methodmal}{{\method-MAH}\xspace}
\usepackage{todonotes}
\usepackage{xcolor}

\begin{document}

\title{\method: A General Framework for Time Series Out-of-distribution Detection and Generalization}

\author{Wang~Lu,
Jindong~Wang,
Xinwei Sun,
Yiqiang Chen,~\IEEEmembership{Senior Member,~IEEE},
Xiangyang Ji,~\IEEEmembership{Member,~IEEE},
Qiang Yang,~\IEEEmembership{Fellow,~IEEE},
and~Xing~Xie,~\IEEEmembership{Fellow,~IEEE}

\IEEEcompsocitemizethanks{\IEEEcompsocthanksitem
Wang Lu and Xiangyang Ji are with Tsinghua University. E-mail: luw12@tsinghua.org.cn,  xyji@tsinghua.edu.cn.

\IEEEcompsocthanksitem Jindong Wang and Xing Xie are with Microsoft Research Asia. E-mail:
\{jindong.wang, xingx\}@microsoft.com.

\IEEEcompsocthanksitem Xinwei Sun is with Fudan University, Shanghai 200437, China. E-mail: sunxinwei@fudan.edu.cn.

\IEEEcompsocthanksitem Yiqiang Chen is with Beijing Key Laboratory of Mobile Computing and Pervasive Devices, CAS. E-mail: yqchen@ict.ac.cn.

\IEEEcompsocthanksitem Qiang Yang is with Hong Kong University of Science and Technology, Hong Kong. E-mail: qyang@cse.ust.hk.

\IEEEcompsocthanksitem Correspondence to: Jindong Wang (jindong.wang@microsoft.com).
}

\thanks{Manuscript received April 19, 2021; revised August 16, 2021.}}

\markboth{Journal of \LaTeX\ Class Files,~Vol.~14, No.~8, August~2021}%
{Shell \MakeLowercase{\textit{et al.}}: A Sample Article Using IEEEtran.cls for IEEE Journals}

\IEEEpubid{0000--0000/00\$00.00~\copyright~2021 IEEE}

\IEEEtitleabstractindextext{%
\begin{abstract}
Time series remains one of the most challenging modalities in machine learning research.
The out-of-distribution (OOD) detection and generalization on time series tend to suffer due to its non-stationary property, i.e., the distribution changes over time.
The \emph{dynamic} distributions inside time series pose great challenges to existing algorithms to identify invariant distributions since they mainly focus on the scenario where the domain information is given as prior knowledge.
In this paper, we attempt to exploit subdomains within a whole dataset to counteract issues induced by non-stationary for generalized representation learning.
We propose \textbf{\method}, a general framework, for OOD detection and generalization on dynamic distributions of time series.
\method takes an iterative process: it first obtains the \emph{`worst-case'} latent distribution scenario via adversarial training, then reduces the gap between these latent distributions.
We implement \method via combining existing OOD detection methods according to either extracted features or outputs of models for detection while we also directly utilize outputs for classification.
In addition, theoretical insights illustrate that \method is theoretically supported.
Extensive experiments are conducted on seven datasets with different OOD settings across gesture recognition, speech commands recognition, wearable stress and affect detection, and sensor-based human activity recognition.
Qualitative and quantitative results demonstrate that \method learns more generalized features and significantly outperforms other baselines.
\end{abstract}

}

\maketitle

\section{Introduction}
\IEEEPARstart{T}{ime} series analysis is one of the most challenging problems in machine learning~\cite{xiao2022dynamic,tang2021omni}. 
For years, there have been tremendous efforts for time series classification, such as hidden Markov models~\cite{fulcher2014highly}, RNN-based methods~\cite{husken2003recurrent}, and Transformer-based approaches~\cite{li2019enhancing,DBLP:conf/icml/DrouinMC22}.
Time series has wide applications in a wide spectrum of applications, e.g., industrial process~\cite{wang2023attention}, stock predicting~\cite{du2021adarnn}, and clinical and remote health~\cite{di2023explainable}.
There are several active research areas related to time series, including classification, forecasting, clustering, multivariate analysis, and high-frequency time series analysis.

The primary focus of this paper is to learn general representations for time series for better out-of-distribution (OOD) detection~\cite{yang2021generalized} and generalization~\cite{wang2021generalizing}.
The main difference between OOD detection and generalization is the different types of distribution shift that cause the problem: label shift and feature shift, according to variables that change in distributions.
On the one hand, label shift normally occurs in classes, which means unseen targets can contain classes not present in the training data, which is studied extensively under the name of anomaly detection and OOD detection\footnote{We noticed that there are some methods performing OOD detection under feature shift~\cite{lehner20223d,vidit2023clip}, but this paper mainly deals with OOD detection caused by label shift, following existing work~\cite{yang2021generalized,yang2022openood}.}; 
On the other hand, feature shift typically happens in inputs and the corresponding research field is OOD generalization that has been extensively studied~\cite{wang2021generalizing}.

OOD detection attempts to solve label shift and it has also attracted much attention recently~\cite{wang2022out,wang2022out2,ren2022out,fang2022out}.
OOD detection can be viewed as a special classification task that distinguishes between in-distribution (ID) classes and OOD classes~\cite{yang2021generalized,yang2022openood}.
For example, Hendrycks et al.~\cite{hendrycks2018deep} trained OOD detectors against an auxiliary dataset of outliers to improve deep OOD detection while BATS~\cite{zhu2022boosting} rectified the feature into its typical set and calculated the OOD score with the typical features to achieve reliable uncertainty estimation.
As for OOD generalization, existing approaches assume the existence of several predefined domains and then endeavor to bridge gaps among domains to learn domain-invariant representations that can be seamlessly transferred to the unseen target distribution.
The key in existing algorithms is to exploit the given domain information (i.e., domain index) to guide the domain-invariant representation learning.
For instance, GILE~\cite{qian2021latent} learned to automatically disentangle domain-agnostic and domain-specific features for generalizable sensor-based cross-person activity recognition while SDMix~\cite{lu2022semantic} provided a semantic data augmentation method to solve a similar problem.
These methods still heavily rely on domain information.

Can we directly adopt existing OOD detection and generalization algorithms for time series?
Unfortunately, the answer is no.
Non-stationary property~\cite{kuznetsov2015learning}, i.e. statistical features changing over time, bring new challenges to time series OOD detection and generalization.
Besides common spatial shifts in feature space, non-stationary property leads to another features shifts named temporal shifts that can occur in one class for the same subject at different times.
Temporal shifts are often \emph{latent}, \emph{dynamic}, and \emph{variable}, which makes it difficult to pre-split data and leads manual splits according to priors inaccurate.
Moreover, different from computer vision, few time-series datasets are well pre-partitioned.
To the best of our knowledge, no work studies OOD representation by considering temporal shifts for time series OOD detection and generalization simultaneously.

\begin{figure*}[t!]
    \centering
    \includegraphics[width=\textwidth]{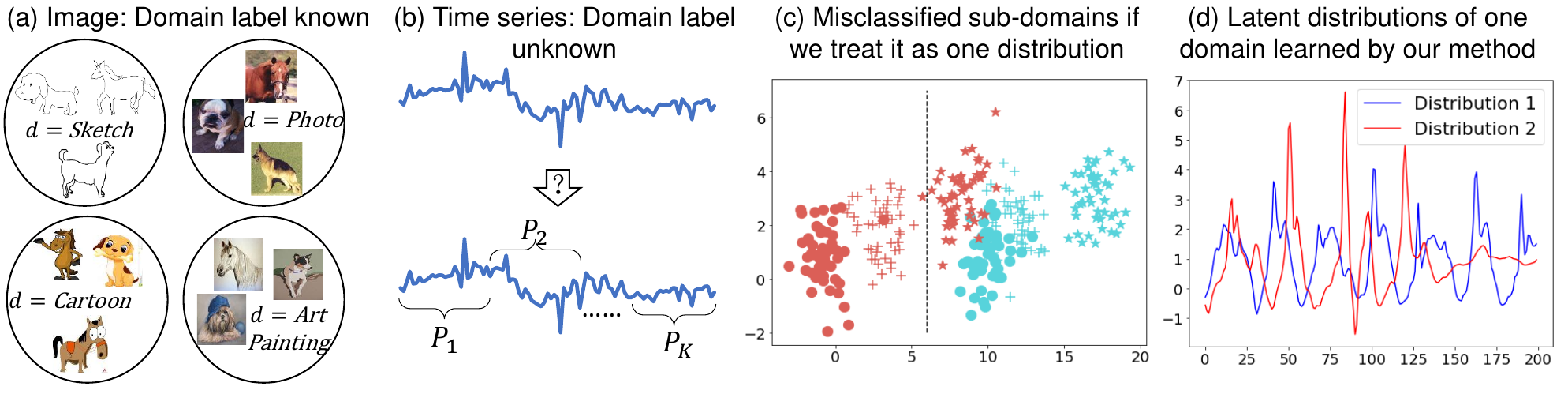}
    \vspace{-.2in}
    \caption{Illustration of \method: (a) Domain generalization for image data requires known domain labels. (b) Domain labels are unknown for time series. (c) If we treat the time series data as one single domain, the sub-domains are misclassified. Different colors and shapes correspond to different classes and domains. Axes represent data values. (d) Finally, our \method can effectively learn the latent distributions. X-axis represents data numbers while Y-axis represents values.}
    \label{fig1-intro}
\end{figure*}


\figurename~\ref{fig1-intro} shows an illustrative example.
OOD generalization in image classification often involves several domains whose domain labels are static and known (subfigure (a)), which can be employed to build OOD models.
However, \figurename~\ref{fig1-intro}(b) shows that in EMG time series data~\cite{lobov2018latent}, the distribution is changing dynamically over time and its domain information is \emph{unavailable}.
If no attention is paid to exploring its \emph{latent} distributions (i.e., sub-domains), predictions may fail in face of diverse sub-domain distributions (subfigure (c)).
This will dramatically impede existing OOD algorithms due to their reliance on domain information.




In this work, we propose \method, a general representation learning framework for time series OOD detection and generalization by characterizing the latent distributions inside the data.
The ultimate idea of \method is to characterize the latent distributions inside time series without domain labels, which can then be leveraged to perform OOD detection and generalization.
Concretely speaking, \method consists of a min-max adversarial game: on one hand, it learns to segment the time series data into several latent sub-domains by maximizing the segment-wise distribution gap to preserve diversities, i.e., the `worst-case' distribution scenario; on the other hand, it learns domain-invariant representations by reducing the distribution divergence between the obtained latent domains.
Such latent distributions naturally exist in time series, e.g., the activity data from multiple people follow different distributions.
Additionally, our experiments show that even the data of one person still has such diversity: it can also be split into several latent distributions.
After obtaining the characterized latent distributions, we can establish different implementations for downstream purposes such as OOD detection and generalization.
Specifically for OOD detection, we provide two implementations, where \methodmal utilizes the Mahalanobis distance with learned generalized representations while \methodmcp makes use of logit outputs of models.
With a simple softmax activation, \method can be easily used for generalization.
Since \method can provide better representations and better predictions, all proposed implementations can significantly outperform other methods.

This paper extends our previous paper published at ICLR 2023~\cite{wang2022out}, which focuses only on OOD generalization.
Compared to the previous version, this version makes substantial extensions by formulating \method as a general framework for both OOD detection and generalization, and then develops novel algorithms for OOD detection with more experiments and analysis.

To summarize, our contributions are four-fold: 
\begin{itemize}[leftmargin=*]
\setlength\itemsep{0em}
    \item \textbf{General framework:} We propose a general framework, \method, to solve OOD detection and generalization simultaneously. \method can identify the latent distributions and learn generalized representations. We provide the theoretical insights behind \method to analyze its design philosophy. 
    \item \textbf{Specific implementations:} 
    For detection, we provide two implementations, \methodmal and \methodmcp. 
    For classification, we directly utilize outputs of \method with softmax activation.
    \item \textbf{Superior performance and insightful results:} Qualitative and quantitative results demonstrate the superiority of \method in several challenging scenarios: difficult tasks, significantly diverse datasets, and limited data. More importantly, \method can successfully characterize the latent distributions within a time series dataset.
    \item \textbf{Extensibility:} Besides implementations proposed in the paper, \method is an extensible framework, which means it can implement with more methods, e.g. ODIN~\cite{liang2018enhancing}. Thereby, \method can be applied to more applications and be further improved with more latest methods.
\end{itemize}

The remainder of this paper is organized as follows. 
We will introduce related work in \sectionautorefname~\ref{sec-rela} and elaborate on the proposed method and offer a clear summary in \sectionautorefname~\ref{sec-method}.  
Then, the experimental implementations and results are presented to demonstrate the superiority of \method for detection and generalization in \sectionautorefname~\ref{sec-exp-detect} and \sectionautorefname~\ref{sec-exp-cls} respectively while experimental analyses are provided in \sectionautorefname~\ref{sec-analay-exp}.
\sectionautorefname~\ref{sec-limit} provides some limitations and discussions.
Finally, conclusions and some possible future directions can be found in \sectionautorefname~\ref{sec-conl}. 

\section{Related Work}
\label{sec-rela}
\subsection{Time series analysis}
Time series classification is a challenging problem.
Existing researches mainly focus on the modeling of temporal relations using either RNN-based methods~\cite{dennis2019shallow} or the recently-proposed Transformer architecture~\cite{DBLP:conf/icml/DrouinMC22}.
MiniRocket~\cite{dempster2021minirocket} transformed input time series using
convolutional kernels and used the transformed features to train a linear classifier. 
MI-ShaRNN~\cite{dennis2019shallow} was proposed to induce long-term dependencies, and yet
admit parallelization. 
In this architecture, the first layer split inputs and ran several independent RNNs while the second layer consumed the outputs using a second RNN.
PatchTST~\cite{nie2022time} segmented time series into subseries-level patches which served as input tokens to Transformer and utilized channel-independence where each channel contained a single univariate time series that shared the same embedding and Transformer weights across all the series.
More related work and details can be found in the following surveys~\cite{ismail2019deep,benidis2022deep,wen2022transformers}.
However, few studies pay attention to latent subdomains in time series to learn generalized features for OOD detection and generalization.

\subsection{Domain/OOD generalization}
Transfer learning~\cite{pan2009survey,silver2013lifelong} trains a model on a source task and aims to enhance the performance of the model on a different but related target task.
Domain generalization (DG)/out-of-distribution generalization can be viewed as a branch of transfer learning but aims at learning models that can be generalized to unseen targets whose distribution can be different from training~\cite{wang2021generalizing,zhou2022domain}.
Existing methods can be categorized into three groups, namely: data manipulation~\cite{shankar2018generalizing,xu2021fourier}, representation learning~\cite{li2018deep,zhang2021deep}, and learning strategy~\cite{rame2022fishr,kim2021selfreg}.
Data manipulation mainly focuses on manipulating the inputs and generating more diversified data or representations to enhance models' generalizability.
CROSSGRAD\cite{shankar2018generalizing} utilized a Bayesian network to model dependence between label, domain, and input instance, and it parallelly a label and a domain classifier on examples perturbed by loss gradients of each other's objectives.
FACT~\cite{xu2021fourier} assumed that the Fourier phase information contained high-level semantics and was not easily by domain shifts.
And thereby, it utilized an amplitude mix that linearly interpolated between the amplitude spectrum of two images to force the model to capture phase information.
Representation learning is the most popular category in domain generalization and it can be further split into two sub-groups, domain-invariant representation learning and feature disentanglement. 
CIAN~\cite{li2018deep}, a conditional invariant adversarial network, learned class-wise adversarial network for DG.
StableNet~\cite{zhang2021deep} utilized a novel nonlinear feature decorrelation approach based on Random Fourier features with linear computational complexity and it could effectively partial out the irrelevant features and leverage truly relevant features for classification.
Learning strategy focuses on exploiting special learning strategy to promote generalization capability.
Fishr~\cite{rame2022fishr} introduced a new regularization and enforced domain invariance in the space of the gradients via the gradient covariance similar to CORAL~\cite{sun2016deep}.
SelfReg~\cite{kim2021selfreg} proposed a new regularization method for domain generalization based on contrastive learning and it only utilized positive data pairs.
More details can be found in the survey~\cite{wang2021generalizing}.
We will introduce more related work and illustrate their difference from ours.

Most domain/OOD generalization methods typically assume the availability of domain labels for training~\cite{peng2019domain,zhang2022towards}.
Specifically, \cite{Matsuura2020DomainGU} also studied DG without domain labels by clustering with the style features for images, which is not applied to time series and is not end-to-end trainable.
Single domain generalization is similar to our problem setting which also involves one training domain~\cite{fan2021adversarially,li2021progressive,wang2021learning,zhu2021crossmatch}.
However, they treated the single domain as one distribution and did not explore latent distributions.
Multi-domain learning is similar to DG which also trains on multiple domains but also tests on training distributions.
\cite{deecke2022visual} proposed sparse latent adapters to learn from unknown domain labels, but their work does not consider the min-max worst-case distribution scenario and optimization.
In domain adaptation, \cite{wang2020continuously} proposed the notion of domain index and further used variational models to learn them~\cite{xu2023domain}, but took a different modeling methodology since they did not consider min-max optimization.
Mixture models~\cite{rasmussen1999infinite} are models representing the presence of subpopulations within an overall population, e.g., Gaussian mixture models.
Our approach has a similar formulation but does not use generative models.
Subpopulation shift is a new setting~\cite{koh2021wilds} that refers to the case where the training and test domains overlap, but their relative proportions differ.
Our problem does not belong to this setting since we assume that these distributions do not overlap.
Distributionally robust optimization~\cite{delage2010distributionally} shares a similar paradigm with our work, whose paradigm is also to seek a distribution that has the worst performance within a range of the raw distribution.
GroupDRO~\cite{sagawa2019distributionally} studied DRO at a group level.
However, we study the internal distribution shift instead of seeking a global distribution close to the original one.
To our best knowledge, there is only one recent work \cite{du2021adarnn} that studied time series from the distribution level.
However, AdaRNN is a two-stage non-differential method that is tailored for RNN and it is mainly designed for prediction.

\subsection{OOD detection}
OOD detection has been extensively studied with a plethora of methods developed in the past few years.
In simple terms, OOD detection aims to find OOD samples that belong to the classes not present in training data.
Existing methods can be categorized into several groups, post-hoc methods, training-time regularization, training with Outlier Exposure, and some other methods.
Comprehensive surveys can be found in \cite{yang2021generalized,yang2022openood}.

Post-hoc~\cite{liang2018enhancing,song2022rankfeat,djurisic2022extremely} is one of the most popular directions for its simplicity and extensibility.
Odin~\cite{liang2018enhancing} utilized temperature scaling and added small perturbations to the input to separate the softmax score distributions between ID and OOD images for more effective detection while RankFeat~\cite{song2022rankfeat} removed the rank-1 matrix composed of the largest singular value and the associated singular vectors from the high-level feature.
Another post-hoc method, ASH~\cite{djurisic2022extremely} removed  a large portion of a sample’s activation at a late layer and simply adjusted the rest at inference time.
Compared to post-hoc methods, training-time regularization methods require training with custom-designed goals~\cite{ming2022exploit,tao2022non}.
CIDER~\cite{ming2022exploit} jointly optimized a dispersion loss and a compactness loss to promote strong ID-OOD separability for exploiting better hyperspherical embeddings while NPOS~\cite{tao2022non} generated artificial OOD training data and facilitated learning a reliable decision boundary between ID and OOD data.
Training with Outlier Exposure~\cite{yu2019unsupervised,yang2021semantically} makes use of a set of collected OOD samples during training to learn the discrepancy between ID and OOD.
\cite{yu2019unsupervised} proposed a two-head deep convolutional neural network (CNN)
and maximized the discrepancy between the two classifiers to detect OOD inputs while \cite{yang2021semantically} proposed unsupervised dual grouping (UDG) to leverage an external unlabeled set for the joint modeling of ID and OOD data.
Few studies consider feature shifts in OOD detection.
\cite{yang2023full} introduced full-spectrum OOD detection and took into account both covariate shift and semantic shift but it mainly focused on computer vision.
For time series OOD detection, there still lacks effective techniques to address both two types of distribution shifts simultaneously.

\section{Methodology}
\label{sec-method}
\subsection{Problem Formulation}
A time-series training dataset $\mathcal{D}^{tr}$ can be often pre-processed using sliding window\footnote{Sliding window is a common technique to segment one time series data into fixed-size windows. Each window is a minimum instance. We focus on fixed-size inputs for its popularity in time series~\cite{das1998rule}.} to $N$ inputs: $\mathcal{D}^{tr}=\{(\mathbf{x}_i, y_i)\}_{i=1}^N$, where $\mathbf{x}_i \in \mathcal{X} \subset \mathbb{R}^p$ is the $p$-dimensional instance  and $y_i \in \mathcal{Y} = \{1, \ldots, C\}$ is its label.
We use $\mathbb{P}^{tr}(\mathbf{x},y)$ on $\mathcal{X} \times \mathcal{Y}$ to denote the joint distribution of the training dataset.
Our goal is to learn a generalized model from $\mathcal{D}^{tr}$ to predict well on an \emph{unseen} target dataset, $\mathcal{D}^{te}$, which is inaccessible in training.
In our problem, the training and test datasets have the same input but different distributions, i.e., $\mathcal{X}^{tr} = \mathcal{X}^{te}$, but $\mathbb{P}^{tr}(\mathbf{x},y) \neq \mathbb{P}^{te}(\mathbf{x},y)$.

\textbf{OOD detection:}
The testing datasets contain more classes than training datasets, i.e. $\mathcal{Y}^{tr} \subset \mathcal{Y}^{te}$.
We denote the classes present in the training datasets as ID classes, $C_{ID}=\{1, 2, \cdots, C_n\}$, while the rest classes that only exist in the testing datasets are the OOD class, $C_{OOD}=\{C_n+1\}$.
We aim to train a model $h$ from $\mathcal{D}^{tr}$ to detect OOD classes and achieve minimum error on $\mathcal{D}^{te}$for ID classes.

\textbf{OOD generalization:}
the training and test datasets share the same output space, i.e. $\mathcal{Y}^{tr} = \mathcal{Y}^{te}$.
We aim to train a model $h$ from $\mathcal{D}^{tr}$ to achieve minimum error on $\mathcal{D}^{te}$.

\subsection{Motivation}

\textbf{What are domain and feature distribution shifts in time series?}
Time series may consist of several unknown latent distributions (domains).
For instance, data collected by sensors of three persons may belong to two different distributions due to their dissimilarities.
This can be termed a spatial distribution shift.
Surprisingly, we even find temporal distribution shifts that distributions of one person can also change at different times.
Those shifts widely exist in time series, as suggested by \cite{zhang2021robust, ragab2022conditional}.
\figurename~\ref{fig1-categ} gives an example where \figurename~\ref{fig1-categ}(a) indicates that the distribution in EMG time series data~\cite{lobov2018latent} is changing dynamically over time and its domain information is \emph{unavailable} while \figurename~\ref{fig1-categ}(b) shows that the sensor data collected during walking follow different distributions across different persons. 
In addition, \figurename~\ref{fig1-categ}(c) provides an example of label shifts, where standing, running, and cycling can be in-distribution (ID) classes present in the training data while falling down is an OOD class only present in the targets.

\textbf{Latent domain characterization is indispensable for OOD detection and generalization.}
Due to the non-stationary property, naive approaches that treat time series as one distribution fail to capture domain-invariant (OOD) features since they ignore the diversities inside the dataset.
In \figurename~\ref{fig1-intro}(d), we assume the training domain contains two sub-domains (circle and plus points).
Directly treating it as one distribution via existing OOD approaches may generate the black margin.
Red star points are misclassified to the green class when predicting on the OOD domain (star points) with the learned model.
Thus, multiple diverse latent distributions in time series should be characterized to learn better OOD features which are essential factors affecting performance of both OOD detection and generalization when encountering non-stationary.
We name distribution shifts in \figurename~\ref{fig1-intro}(b) spatial distribution shifts, for which we can group data into different domains according to some specific characteristics, e.g., persons, positions, and some other factors.
However, in real scenarios, the information can be missing or not suitable for grouping and we only have access to a whole dataset without splits.


\begin{figure}[t!]
    \centering
    \includegraphics[width=0.48\textwidth]{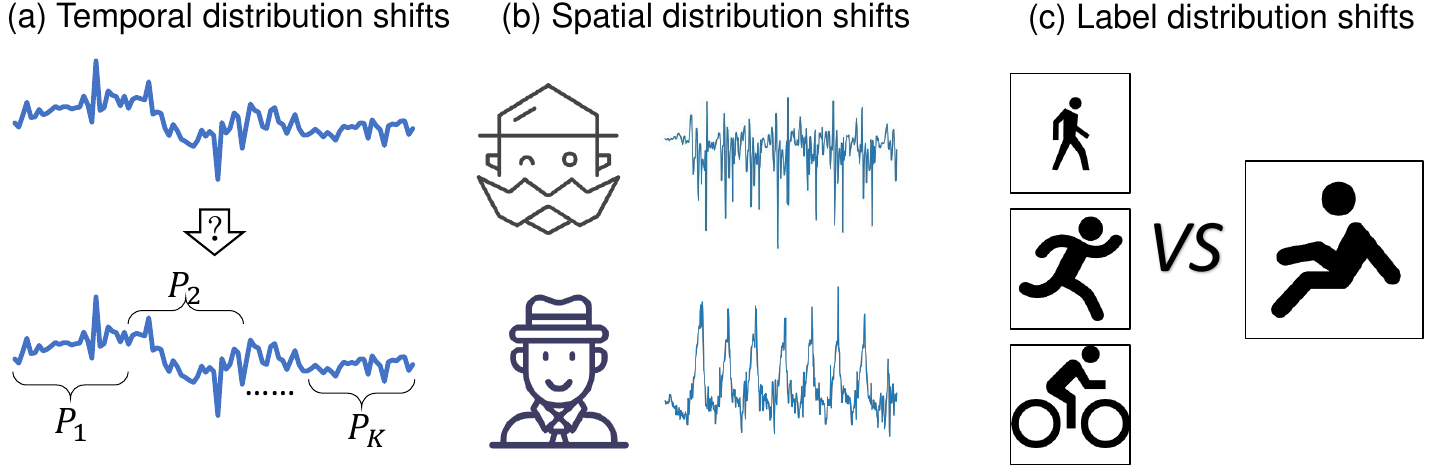}
    \vspace{-.2in}
    \caption{Categories of distribution shifts.}
    \label{fig1-categ}
    \vspace{-.2in}
\end{figure}

\textbf{A brief formulation of latent domain characterization.}
Following the above discussions, according to feature shifts, a time series may consist of $K$ unknown latent domains\footnote{We assume $K \in [N]$ in this work, as a smaller $K$ is too coarse to show the diversities in distributions while a larger $K$ brings difficulties to optimization. $K$ is tuned here but we hope to learn it in the future.}\footnote{A domain is a set of data samples following a certain distribution and we use them interchangeably hereafter.} rather than a fixed one, i.e., $\mathbb{P}^{tr}(\mathbf{x},y) = \sum_{i=1}^{K} \pi_i \mathbb{P}^{i}(\mathbf{x},y)$,
where $\mathbb{P}^{i}(\mathbf{x},y)$ is the distribution of the $i$-th latent one with weight $\pi_i$, $\sum_{i=1}^K \pi_i = 1$.\footnote{We use the notations $\pi_i$ and $\mathbb{P}^i$ to only describe the problem, but do not formalize it.}
There could be infinite ways to obtain $\mathbb{P}^i$s and our goal is to learn the `worst-case' distribution scenario where the distribution divergence between each $\mathbb{P}^i$ and $\mathbb{P}^j$ is maximized.
Why the `worst-case' scenario?
It will maximally preserve the diverse information of each latent distribution, thus benefiting generalization.

\subsection{\method}
\label{sec-div-fra}
\begin{figure*}[t!]
    \centering
    \includegraphics[width=0.8\textwidth]{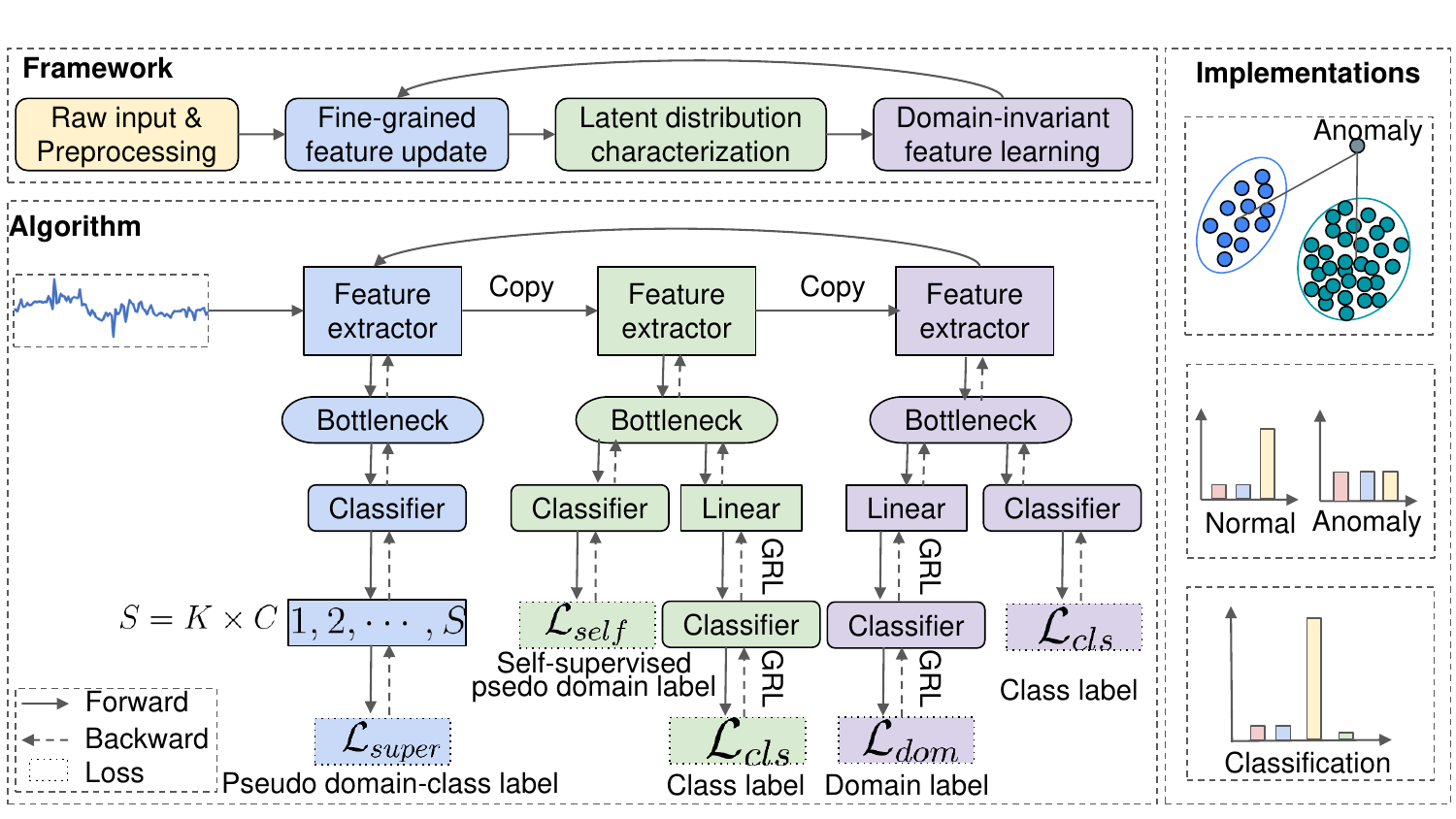}
    \caption{The framework of \method.}
    \label{fig:frame}
\end{figure*}

In this paper, we propose \method to learn OOD representations for time series OOD detection and generalization.
The core of \method is to characterize the latent distributions and then minimize the distribution divergence between each two.
Concretely speaking, an iterative process is utilized: it first obtains the 'worst-case' distribution scenario from a given dataset, then bridges the distribution gaps between each pair of latent distributions.
It mainly contains four steps, where step $2 \sim 4$ are iterative:
\begin{enumerate}[leftmargin=*]

    \item Pre-processing: this step adopts the sliding window to split the entire training dataset into fixed-size windows. We argue that the data from one window is the smallest domain unit.
    \item Fine-grained feature update: this step updates the feature extractor using the proposed pseudo domain-class labels as the supervision.
    \item Latent distribution characterization: it aims to identify the domain label for each instance to obtain the latent distribution information.
    It maximizes the different distribution gaps to enlarge diversity.
    \item Domain-invariant representation learning: this step utilizes pseudo domain labels from the last step to learn domain-invariant representations and train a generalizable model. 
\end{enumerate}

\textbf{Fine-grained Feature Update.}
Before characterizing latent distributions, we perform fine-grained feature updates to obtain fine-grained representation.
As shown in \figurename~\ref{fig:frame} (blue), we propose a new concept: pseudo domain-class label to fully utilize the knowledge contained in domains and classes, which serves as the supervision for the feature extractor.
Features are more fine-grained w.r.t. domains and labels, instead of only attached to domains or labels.

At the first iteration, there is no domain label $d'$ and we simply initialize $d'=0$ for all samples.
We treat per category per domain as a new class with label $s\in \{1,2,\cdots, S\}$.
We have $S = K \times C$\footnote{For detection, replace $C$ with $C_n$.} where $K$ is the pre-defined number of latent distributions that can be tuned in experiments.
We perform pseudo domain-class label assignment to get discrete values for supervision: $s = d'\times C + y.$ 

Let $h_f^{(2)}, h_b^{(2)}, h_c^{(2)}$ be feature extractor, bottleneck, and classifier, respectively (we use superscripts to denote step number).
Then, the supervised loss is computed using the cross-entropy loss $\ell$:
\begin{equation}
    \mathcal{L}_{super} = 
    \mathbb{E}_{(\mathbf{x},y)\sim \mathbb{P}^{tr}} \ell \left(h_{c}^{(2)}(h_{b}^{(2)}(h_f^{(2)}(\mathbf{x}))),s \right).
    \label{eqa:step1}
\end{equation}

\textbf{Latent Distribution Characterization.}
This step characterizes the latent distributions contained in one dataset.
As shown in \figurename~\ref{fig:frame} (green), we propose an adapted version of adversarial training to disentangle the domain labels from the class labels.
However, there are no actual domain labels provided, which hinders such disentanglement.
Inspired by \cite{caron2018deep}, we employ a self-supervised pseudo-labeling strategy to obtain domain labels.

First, we attain the centroid for each domain with class-invariant features:
\begin{equation}
    \tilde{\mu}_k = \frac{\sum_{\mathbf{x}_i \in \mathcal{X}^{tr}} 
    \delta_k(h_{c}^{(3)}(h_{b}^{(3)}(h_f^{(3)}(\mathbf{x}_i)))) h_{b}^{(3)}(h_f^{(3)}(\mathbf{x}_i))}{\sum_{\mathbf{x}_i\in \mathcal{X}^{tr}} 
    \delta_k(h_{c}^{(3)}(h_{b}^{(3)}(h_f^{(3)}(\mathbf{x}_i))))},
    \label{eqa:initialc}
\end{equation}
where $h_f^{(3)}, h_b^{(3)}, h_c^{(3)}$ are feature extractor, bottleneck, and classifier, respectively.
$\tilde{\mu}_k$ is the initial centroid of the $k^{th}$ latent domain while $\delta_{k}$ is the $k^{th}$ element of the logit soft-max output.
Then, we obtain the pseudo domain labels via the nearest centroid classifier using a distance function $D$:
\begin{equation}
    \tilde{d}'_i = \arg\min_k D(h^{(3)}_{b}(h^{(3)}_f(\mathbf{x}_i)), \tilde{\mu}_k).
    \label{eqa:initiald}
\end{equation}

Then, we compute the centroids and obtain the updated pseudo domain labels:
\begin{equation}
\begin{aligned}
    \mu_k = \frac{\sum_{\mathbf{x}_i\in \mathcal{X}^{tr}} 
    \mathbb{I}(\tilde{d}'_i=k) h^{(3)}_{b}(h^{(3)}_f(\mathbf{x}))}{\sum_{\mathbf{x}_i\in \mathcal{X}^{tr}} 
    \mathbb{I}(\tilde{d}'_i=k)},\\
    d'_i = \arg\min_k D(h^{(3)}_b(h^{(3)}_f(\mathbf{x}_i)), \mu_k),
\end{aligned}
    \label{eqa:psd}
\end{equation}
where $\mathbb{I}(a)=1$ when $a$ is true, otherwise $0$.
After obtaining $d'$, we can compute the loss of step 2:
\begin{equation}
\begin{aligned}
    \mathcal{L}_{self} + \mathcal{L}_{cls} =&
 \mathbb{E}_{(\mathbf{x},y)\sim \mathbb{P}^{tr}} \ell(h^{(3)}_c(h^{(3)}_b(h^{(3)}_f(\mathbf{x}))),d')\\
&    + \ell(h^{(3)}_{adv}(R_{\lambda_1}(h^{(3)}_b(h^{(3)}_f(\mathbf{x})))),y),
\end{aligned}
    \label{eqa:step2}
\end{equation}
where $h^{(3)}_{adv}$ is the discriminator for step 3 that contains several linear layers and one classification layer.
$R_{\lambda_1}$ is the gradient reverse layer with hyperparameter $\lambda_1$~\cite{ganin2016domain}.
After this step, we can obtain the pseudo domain label $d'$ for $\mathbf{x}$.

\textbf{Domain-invariant Representation Learning.}
After obtaining the latent distributions, we learn domain-invariant representations for generalization.
In fact, this step (purple in \figurename~\ref{fig:frame}) is simple: we borrow the idea from DANN~\cite{ganin2016domain} and directly use adversarial training to update the classification loss $\mathcal{L}_{cls}$ and domain classifier loss $\mathcal{L}_{dom}$ using gradient reversal layer (GRL) (a common technique that facilitates adversarial training via reversing gradients)~\cite{ganin2016domain}:
\begin{equation}
\begin{aligned}
    \mathcal{L}_{cls} + \mathcal{L}_{dom} =& \mathbb{E}_{(\mathbf{x},y)\sim \mathbb{P}^{tr}} \ell(h_{c}^{(4)}(h_{b}^{(4)}(h_f^{(4)}(\mathbf{x}))),y)\\
&    + \ell(h_{adv}^{(4)}(R_{\lambda_2}(h_{b}^{(4)}(h_f^{(4)}(\mathbf{x})))),d'),
\end{aligned}
    \label{eqa:step3}
\end{equation}
where $\ell$ is the cross-entropy loss and $R_{\lambda_2}$ is the gradient reverse layer with hyperparameter $\lambda_2$~\cite{ganin2016domain}.
We will omit the details of GRL and adversarial training here since they are common techniques in deep learning.

\textbf{Training and Complexity.}
We repeat these steps until convergence or max epochs.
Different from existing methods, the last two steps only optimize the last few independent layers.
Most of the trainable parameters are shared between modules, indicating that \method has the same model size as existing methods.
The modules from the last step are utilized for inference.

\subsection{\method for OOD detection}

\method attempts to exploit subdomains and learn generalized representations.
In this section, we provide two implementations of \method, \methodmal and \methodmcp, for time series OOD detection.
In these two implementations, we first train the model with steps in \sectionautorefname~\ref{sec-div-fra} and then combine post-hoc methods using representations and logits respectively when testing.

\textbf{\methodmal.}
Mahalanobis distance-based confidence score is a popular metric to detect OOD samples~\cite{lee2018simple} and it is mainly influenced by the features/representations given by the model.
We can obtain the class conditional Gaussian distributions with respect to features of the deep models under Gaussian discriminant analysis (GDA) which result in a confidence score based on the Mahalanobis distance.
\methodmal obtains the features of the corresponding samples using the outputs of the bottleneck,
\begin{equation}
\mathbf{z}=h_{b}^{(4)}(h_f^{(4)}(\mathbf{x})).
\end{equation}

\methodmal assumes that the class-conditional distributions of $\mathbf{z}$ follow multivariate Gaussian distributions, 
\begin{equation}
\mathbb{P}(\mathbf{z}| y=c) \sim \mathcal{N} (\mathbf{\mu}_c, \Sigma),
\end{equation}
where $\mu_c$ is the mean of multivariate Gaussian distribution of ID class $c$ while $\Sigma$ is a tied covariance matrix,
\begin{equation}
    \begin{aligned}
        \mathbf{\mu}_c &= \frac{1}{N_c}\sum_{i:y_i=c} \mathbf{z}_i,\\
        \Sigma&=\frac{1}{N}\sum_c \sum_{i:y_i=c}(\mathbf{z}_i -\mathbf{\mu}_c)(\mathbf{z}_i - \mathbf{\mu}_c)^T,
    \end{aligned}
\end{equation}
where $N_c$ denotes the number of samples in class $c$.
According to a simple theoretical connection between GDA and the softmax classifier \cite{lasserre2006principled,murphy2012machine,lee2018simple}, the posterior distribution defined by the generative classifier under GDA with tied covariance assumption is equivalent to the softmax classifier.
Now, we can utilize the Mahalanobis distance between a test sample $\mathbf{x}$ and the
closest class-conditional distribution as a confidence score,
\begin{equation}
    M(\mathbf{x})=\max_c - (\mathbf{z}-\mathbf{\mu}_c)^T\Sigma^{-1}(\mathbf{z}-\mathbf{\mu}_c).
\end{equation}

The larger $M(\mathbf{x})$ is, the more likely $\mathbf{x}$ belongs to class $c$, and thereby $\mathbf{x}$ is more likely to be an ID sample.
Conversely, a small $M(\mathbf{x})$ illustrates that the sample can be an OOD sample.
Since \methodmal does not rely on the predictions of the model, it can avoid the impact of high-confidence outputs from the deep learning models.

\textbf{\methodmcp.}
Maximum class probability (MCP)~\cite{hendrycks2018deep} is a popular baseline for OOD detection.
It is extremely influenced by the predictions of the model since it utilizes the logit outputs of the model directly, which means that better predictions can bring better performance with MCP.
Combining with MCP, we implement \methodmcp.

\methodmcp obtains an estimation vector of data as:
\begin{equation}
\mathbf{y}'=h_{c}^{(4)}(h_{b}^{(4)}(h_f^{(4)}(\mathbf{x}))).
\end{equation}

Via a softmax activation, $\mathbf{y}'$ can be converted into a vector ranging from $0$ to $1$, $\tilde{\mathbf{y}}' = \mathrm{softmax}(\mathbf{y}')$, which can reflect the confidence of the model's confidence to some extent.
Commonly, we can view $\tilde{\mathbf{y}}'$ as a probability estimation of classes.
The probability $v = \max_y \mathbb{P}(y|\mathbf{x})$ of the most likely class is the final prediction and it can also serve as an ID score.
The larger $v$ is, the more confident the model is, and thereby the sample is more likely to be an ID sample.
Conversely, a small $v$ indicates possible OOD inputs.
Since modern NNs have been shown to often make over-confident softmax outputs, we can also utilize a temperature hyperparameter in the softmax activation to generate smoother predictions~\cite{hinton2015distilling},\footnote{In this paper, we simply set $T=1$.}

\begin{equation}
        \tilde{\mathbf{y}'} = \mathrm{softmax}(\mathbf{y}',T), \text{ and }
        \mathbf{\tilde{y}}'_i=\frac{\exp( \mathbf{y}'_i/T)}{\sum_j \exp( \mathbf{y}'_j/T)}.
\end{equation}

\subsection{\method for OOD generalization}
\method can also be utilized for classification.
We can directly utilize the outputs of step 4 with a softmax activation to obtain predictions,
\begin{equation}
\mathbf{y}'=\mathrm{softmax}(h_{c}^{(4)}(h_{b}^{(4)}(h_f^{(4)}(\mathbf{x})))).
\end{equation}

\subsection{Theoretical Insights}

\textbf{Preliminary}
For a distribution $\mathbb{P}$ with an ideal binary labeling function $h^*$ and a hypothesis $h$, we define the error $\varepsilon_{\mathbb{P}}(h)$ in accordance with~\cite{ben2010theory} as: 
\begin{equation}
    \label{eqa:def-eps}
    \varepsilon_{\mathbb{P}}(h) = \mathbb{E}_{\mathbf{x}\sim \mathbb{P}} |h(\mathbf{x}) - h^*(\mathbf{x})|.
\end{equation}

We also give the definition of $\mathcal{H}$-divergence according with~\cite{ben2010theory}.
Given two distributions $\mathbb{P}, \mathbb{Q}$ over a space $\mathcal{X}$ and a hypothesis class $\mathcal{H}$, 
\begin{equation}
    \label{eqa:def-hdive}
    d_{\mathcal{H}}(\mathbb{P},\mathbb{Q}) = 2\sup_{h\in \mathcal{H}} |Pr_{\mathbb{P}}(I_h)- Pr_{\mathbb{Q}}(I_h)|,
\end{equation}
where $I_h = \{\mathbf{x}\in\mathcal{X}|h(\mathbf{x})=1\}$.
We often consider the $\mathcal{H}\Delta \mathcal{H}$-divergence in~\cite{ben2010theory} where the symmetric difference
hypothesis class $\mathcal{H}\Delta \mathcal{H}$ is the set of functions characteristic to disagreements between hypotheses.

\begin{theorem}
\label{thm:da,ben}
(Theorem 2.1 in~\cite{sicilia2021domain}, modified from Theorem 2 in~\cite{ben2010theory}). Let $\mathcal{X}$ be a space and $\mathcal{H}$ be a class of hypotheses corresponding to this space. Suppose $\mathbb{P}$ and $\mathbb{Q}$ are distributions over $\mathcal{X}$. Then for any $h\in \mathcal{H}$, the following holds
\begin{equation}
    \label{eqa:da}
    \varepsilon_{\mathbb{Q}}(h) \leq \lambda'' + \varepsilon_{\mathbb{P}}(h)+ \frac{1}{2} d_{\mathcal{H}\Delta \mathcal{H}}(\mathbb{Q}, \mathbb{P})
\end{equation}
with $\lambda''$ the error of an ideal joint hypothesis for $\mathbb{Q}, \mathbb{P}$.
\end{theorem}

\theoremautorefname~\ref{thm:da,ben} provides an upper bound on the target-error.
$\lambda''$ is a property of the dataset and hypothesis class and is often ignored.
\theoremautorefname~\ref{thm:da,ben} demonstrates the necessity to learn domain invariant features.

\textbf{Why \method can learn better representations and achieve better predictions?}
We present some theoretical insights to show that our approach is well motivated.

\begin{proposition}
\label{prop:hdist}
Let $\mathcal{X}$ be a space and $\mathcal{H}$ be a class of hypotheses corresponding to this space. Let $\mathbb{Q}$ and the collection $\{\mathbb{P}_i \}_{i=1}^K$ be distributions over $\mathcal{X}$ and let $\{\varphi_i \}_{i=1}^K$ be a collection of non-negative coefficients with $\sum_i \varphi_i = 1$. Let $\mathcal{O}$ be a set of distributions s.t. $\forall \mathbb{S}\in \mathcal{O}$, the following holds
\begin{equation}
    \label{eqa:range}
     d_{\mathcal{H}\Delta \mathcal{H}}(\sum_i \varphi_i\mathbb{P}_i, \mathbb{S}) \leq \max_{i,j} d_{\mathcal{H}\Delta \mathcal{H}}(\mathbb{P}_i,\mathbb{P}_j).
\end{equation}
Then, for any $h\in \mathcal{H}$, 
\begin{equation}
    \label{eqa:bound}
    \begin{aligned}
    \varepsilon_{\mathbb{Q}}(h)\leq & \lambda' + \sum_i \varphi_i \varepsilon_{\mathbb{P}_i}(h) + \frac{1}{2}\min_{\mathbb{S}\in\mathcal{O}}  d_{\mathcal{H}\Delta \mathcal{H}}(\mathbb{S}, \mathbb{Q})\\ & + \frac{1}{2}\max_{i,j} d_{\mathcal{H}\Delta \mathcal{H}}(\mathbb{P}_i, \mathbb{P}_j),
    \end{aligned}
\end{equation}
where $\lambda'$ is the error of an ideal joint hypothesis.
$\varepsilon_{\mathbb{P}}(h)$ is the error for a hypothesis $h$ on a distribution $\mathbb{P}$.
$d_{\mathcal{H}\Delta \mathcal{H}} (\mathbb{P}, \mathbb{Q})$ is $\mathcal{H}$-divergence which measures differences in distribution~\cite{ben2010theory}. 
\end{proposition}

\begin{proof}
On one hand, with \theoremautorefname~\ref{thm:da,ben}, we have 
\begin{equation}
    \label{eqa:p1s1}
    \begin{aligned}
    \varepsilon_{\mathbb{Q}}(h) \leq \lambda'_1 +\varepsilon_{\mathbb{S}}(h)+ \frac{1}{2}d_{\mathcal{H}\Delta \mathcal{H}}(\mathbb{S}, \mathbb{Q}), \forall h\in \mathcal{H}, \forall\mathbb{S}\in \mathcal{O}.
    \end{aligned}
\end{equation}

On the other hand, with \theoremautorefname~\ref{thm:da,ben}, we have 
\begin{equation}
    \label{eqa:p1s2}
    \begin{aligned}
    \varepsilon_{\mathbb{S}}(h) \leq & \lambda'_2 +\varepsilon_{\sum_i \varphi_i\mathbb{P}_i}(h)+ \\ &\frac{1}{2}d_{\mathcal{H}\Delta \mathcal{H}}(\sum_i \varphi_i\mathbb{P}_i, \mathbb{S}), \forall h\in \mathcal{H}.
    \end{aligned}
\end{equation}

Since $d_{\mathcal{H}\Delta \mathcal{H}}(\sum_i \varphi_i\mathbb{P}_i, \mathbb{S}) \leq \max_{i,j} d_{\mathcal{H}\Delta \mathcal{H}}(\mathbb{P}_i,\mathbb{P}_j)$, and $\varepsilon_{\sum_i \varphi_i\mathbb{P}_i}(h) = \sum_i \varphi_i\varepsilon_{\mathbb{P}_i}(h)$, we have
\begin{equation}
    \label{eqa:p1s3}
    \begin{aligned}
    \varepsilon_{\mathbb{Q}}(h) \leq & \lambda' + \sum_i \varphi_i \varepsilon_{\mathbb{P}_i}(h)+ \frac{1}{2}d_{\mathcal{H}\Delta \mathcal{H}}(\mathbb{S}, \mathbb{Q})\\ &+\frac{1}{2}\max_{i,j}d_{\mathcal{H}\Delta \mathcal{H}}(\sum_i \varphi_i\mathbb{P}_i, \mathbb{S}), \forall h\in \mathcal{H}, \forall\mathbb{S}\in \mathcal{O},
    \end{aligned}
\end{equation}
where $\lambda'=\lambda'_1+\lambda'_2$.
\equationautorefname~\ref{eqa:p1s3} for all $\mathbb{S}\in \mathcal{O}$ holds. Therefore, we complete the proof.
\end{proof}

The first item in \eqref{eqa:bound}, $\lambda'$, is often neglected since it is small in reality.
The second item, $\sum_i \varphi_i \varepsilon_{\mathbb{P}_i}(h)$, exists in almost all methods and can be minimized via supervision from class labels with cross-entropy loss in \eqref{eqa:step3}.
Our main purpose is to minimize the last two items in \eqref{eqa:bound}.
Here $\mathbb{Q}$ corresponds to the unseen out-of-distribution target domain.

The last term $\frac{1}{2} \max_{i,j}d_{\mathcal{H}\Delta \mathcal{H}}(\mathbb{P}_i,\mathbb{P}_j)$ is common in OOD theory which measures the maximum differences among source domains.
This corresponds to step 4 in our approach.

Finally, the third item, $\frac{1}{2}\min_{\mathbb{S} \in \mathcal{O}} d_{\mathcal{H}\Delta \mathcal{H}}(\mathbb{S},\mathbb{Q})$, explains why we exploit sub-domains.
Since our goal is to learn a model which can perform well on an unseen target domain, we cannot obtain $\mathbb{Q}$.
To minimize $\frac{1}{2}\min_{\mathbb{S} \in \mathcal{O}} d_{\mathcal{H}\Delta \mathcal{H}}(\mathbb{S},\mathbb{Q})$, we can only enlarge the range of $\mathcal{O}$.
We have to $\max_{i,j}d_{\mathcal{H}\Delta \mathcal{H}}(\mathbb{P}_i,\mathbb{P}_j)$ according to \eqref{eqa:range}, corresponding to step 3 in our method which tries to segment the time series data into several latent sub-domains by maximizing the segment-wise distribution gap to preserve diversities, i.e. the `worst-case' distribution scenario.
Better representations and predictions bring improvements on OOD detection and generalization.

\section{Experiments on OOD detection}
\label{sec-exp-detect}
We perform evaluations on three diverse time series detection tasks: gesture recognition, wearable stress\&affect detection, and sensor-based activity recognition.
\tablename~\ref{tb-data-harss} shows the statistical information on datasets that we use.

\begin{table}[htbp]
\centering
\caption{Information on datasets.}
\label{tb-data-harss}
\resizebox{0.4\textwidth}{!}{%
\begin{tabular}{crrrr}
\toprule
Dataset& Subjects&Sensors&Classes& Samples\\
\midrule
EMG&36&1&7&33,903,472\\
WESAD&15&8&4&63,000,000\\
DSADS&8&3&19&1,140,000\\
USC-HAD&14&2&12&5,441,000\\
UCI-HAR&30&2&6&1,310,000\\
PAMAP2&9&3&18&3,850,505\\
\bottomrule
\end{tabular}%
}
\end{table}

\subsection{Setup}
We utilize the sliding window technique to split data.
As its name suggests, this technique involves taking a subset of data from a given array or sequence.
Two main parameters of the sliding window technique are the window size, describing a subset length, and the step size, describing moving forward distance each time.

Time series OOD detection algorithms with feature shifts are currently less studied and we combine existing DG methods with MCP and Mahalanobis distance.
We compare with five state-of-the-art methods.
DANN~\cite{ganin2016domain} is a method that utilizes adversarial training to force the discriminator unable to classify domains for better domain-invariant features. It requires domain labels and splits data in advance while ours is a universal method.
CORAL~\cite{sun2016deep} utilizes the covariance alignment in feature layers for better domain-invariant features. It also requires domain labels and splits data in advance.
GroupDRO~\cite{sagawa2019distributionally} is a method that seeks a global distribution with the worst performance within a range of the raw distribution for better generalization. Ours study the internal distribution shift instead of seeking a global distribution close to the original one.
ANDMask~\cite{parascandolo2020learning} is a gradient-based optimization method that belongs to special learning strategies.

For fairness, all methods use a feature net with two blocks and each block has one convolution layer, one pooling layer, and one batch normalization layer, following \cite{wang2019deep}.
All methods are implemented with PyTorch~\cite{paszke2019pytorch}.
The maximum training epoch is set to 150.
The Adam optimizer with weight decay $5\times 10^{-4}$ is used.
The learning rate for the rest methods is $10^{-2}$ or $10^{-3}$.
For the pooling layer, we utilize MaxPool2d.
The kernel size is $(1,2)$ ad the stride is 2.

Some OOD methods require the domain labels known in training while ours does not, which is more challenging and practical.
For these methods that require domain labels, we randomly assign domain labels to them in batches.
We conduct the training-domain-validation strategy and the training data are split by $8:2$ for training and validation.
We tune all methods to report the average best performance of three trials. 
$K$ in \method is treated as a hyperparameter and we tune it to record the best OOD performance.\footnote{There might be no optimal $K$ for a dataset. We perform a grid search in $[2,10]$ to get the best performance.}
We utilize three metrics on the testing datasets for evaluation, including ID accuracy which evaluates generalized classification capability and the Area Under the Receiver Operating Characteristic curve (AUROC) and the Area Under the Precision-Recall curve (AUPR) for evaluating OOD detection capability following \cite{hendrycks2016baseline}.

\subsection{Gesture Recognition}
First, we evaluate \method on EMG for gestures Data Set~\cite{lobov2018latent}.
Electromyography (EMG) is a typical time-series data that is based on bioelectric signals.
EMG for gestures Data Set~\cite{lobov2018latent} contains raw EMG data recorded by MYO Thalmic bracelet.
The bracelet is equipped with eight sensors equally spaced around the forearm that simultaneously acquire myographic signals.
EMG data are scene and device-dependent, which means the same person may generate different data when performing the same activity with the same device at a different time (i.e., distribution shift across time~\cite{wilson2020multi,purushotham2016variational}) or with different devices at the same time.
Data from 36 subjects are collected while they performed a series of static hand gestures and the number of instances is $~40,000-50,000$ recordings in each column.
It contains 7 classes and we select 6 common classes performed by all subjects for our experiments.
Ulnar deviations is selected as OOD class while the rests serve as the ID classes.

For EMG, we set the window size to 200 and the step size to 100, which means there exist $50\%$ overlaps between two adjacent samples.
We normalize each sample with $\tilde{\mathbf{x}} = \frac{\mathbf{x} - \min \mathbf{X} }{\max \mathbf{X} - \min \mathbf{X}}$.
$\mathbf{X}$ contains all $\mathbf{x}$.
The final dimension is $8\times 1 \times 200$.
We randomly divide 36 subjects into four domains (i.e., $0, 1, 2, 3$) without overlapping and each domain contains data of 9 persons.

\begin{table}[t!]
\centering
\caption{ID accuracy, AUROC, and AUPR on EMG. \textbf{Bold} means the best while \underline{underline} means the second-best.}
\label{tab1-emg-detect}
\resizebox{0.5\textwidth}{!}{
\begin{tabular}{llllll|ll|ll}
\toprule
Targets                & 0              & 1              & 2              & 3              & AVG            & \multicolumn{2}{c|}{AVG (+MCP)}    & \multicolumn{2}{c}{AVG (+MAH)}    \\
Metrics                & \multicolumn{5}{c|}{ID ACC}                                                         & AUROC          & AUPR           & AUROC          & AUPR           \\
\midrule
ERM                    & 62.59          & 70.56          & 77.45          & 69.96          & 70.14          & 54.85          & 85.98          & 49.28          & 85.17          \\
CORAL                  & 66.88          & \underline{82.71}    & \underline{82.53}    & 73.88          & \underline{76.50}    & 61.42          & 89.02          & \underline{64.32}    & \underline{90.71}    \\
DANN                   & 69.69          & 76.17          & 80.74          & \underline{75.80}    & 75.60          & 61.54          & 88.31          & 58.79          & 89.08          \\
GroupDRO               & \underline{74.19}    & 77.90          & 81.53          & 69.18          & 75.70          & \underline{66.09}    & \underline{90.33}    & 62.33          & 90.59          \\
ANDMask                & 65.61          & 78.30          & 75.73          & 55.02          & 68.67          & 51.67          & 84.71          & 59.75          & 88.37          \\
\method & \textbf{81.43} & \textbf{90.39} & \textbf{87.26} & \textbf{86.19} & \textbf{86.32} & \textbf{70.76} & \textbf{92.66} & \textbf{78.38} & \textbf{94.65} \\ \bottomrule
\end{tabular}}
\end{table}

The results are shown in \tableautorefname~\ref{tab1-emg-detect} and we gain the following observation.
1) Our method achieves the best performance on each task and has an improvement of about $10\%$ on average compared to the second-best method, which demonstrates that our method has a good capability of ID classification.
2) For methods with MCP, \methodmcp achieves the best AUROC and AUPR on average.
Compared to the second-best method, \methodmcp has improvements of $4.67\%$ and $2.33\%$ respectively.
For methods with Mahalanobis distance \methodmal achieves the best AUROC and AUPR on average.
Compared to the second-best method, \methodmal has improvements of $14.06\%$ and $3.94\%$ respectively.
These results demonstrate that \method has a good ability to detect anomalies.
3) In most situations, ID accuracy has a positive relation to AUROC and AUPR, but there also exist some counterexamples.
CORAL achieves better ID accuracy but worse AUROC(MCP) on average than GroupDRO. 
Therefore, we might need to select different hyperparameters and methods for different purposes.
4) Compared to \methodmcp, \methodmal has another remarkable improvement but some methods, e.g. ERM-MAH and ERM-MCP, perform more terribly.
These results demonstrate that our method can learn better generalized representations.
And for time series OOD detection task, Mahalanobis distance is not influenced by the over-confidence of deep models while MCP might require further tuning.

\subsection{Wearable Stress and Affect Detection}
We further evaluate \method on a larger dataset, Wearable Stress and Affect Detection (WESAD)~\cite{schmidt2018introducing}.
WESAD is a public dataset that contains physiological and motion data of 15 subjects with $63,000,000$ instances.
We utilize sensor modalities of chest-worn devices including electrocardiogram, electrodermal activity, electromyogram, respiration, body temperature, and three-axis acceleration.
We split 15 subjects into four domains.
We utilize the same preprocessing as EMG and select class stress as the OOD class.

\begin{table}[t!]
\centering
\caption{ID accuracy, AUROC, and AUPR on WESAD. \textbf{Bold} means the best while \underline{underline} means the second-best.}
\label{tab1-wesad-detect}
\resizebox{0.5\textwidth}{!}{
\begin{tabular}{llllll|ll|ll}
\toprule
Targets                & 0              & 1              & 2              & 3              & AVG            & \multicolumn{2}{c|}{AVG (+MCP)}    & \multicolumn{2}{c}{AVG (+MAH)}    \\
Metrics                & \multicolumn{5}{c|}{ID ACC}                                                         & AUROC          & AUPR           & AUROC          & AUPR           \\
\midrule
ERM                    & \textbf{52.83} & 60.77          & 59.01          & 55.53          & 57.04          & 49.66          & 80.54          & 59.68          & 83.56          \\
CORAL                  & \underline{51.66}    & 55.22          & 50.19          & 52.49          & 52.39          & \underline{69.98}    & 87.97          & 72.67          & 90.94          \\
DANN                   & 49.22          & 57.12          & 56.47          & \underline{59.96}    & 55.69          & 60.34          & 85.97          & 65.97          & 86.86          \\
GroupDRO               & 51.63          & \underline{67.12}    & \textbf{63.96} & 54.01          & \underline{59.18}    & 66.20          & \underline{88.32}    & \underline{77.92}    & \underline{92.12}    \\
ANDMask                & 50.20          & 54.80          & 57.30          & 56.62          & 54.73          & 53.30          & 81.30          & 66.12          & 84.17          \\
\method & 48.79          & \textbf{73.18} & \underline{60.49}    & \textbf{75.01} & \textbf{64.37} & \textbf{77.24} & \textbf{92.57} & \textbf{88.68} & \textbf{96.47} \\ \bottomrule
\end{tabular}}
\end{table}

The results are shown in \tableautorefname~\ref{tab1-wesad-detect} and we have the following observation.
1) Similar to EMG, our method has the best ID accuracy, AUROC, and AUPR on average, which demonstrates that our method has excellent abilities of classification and OOD detection.
2) For some tasks, our method obtains worse ID accuracy than other methods, which can be caused by two reasons.
On the one hand, our method relies on adapted DANN to exploit subdomains and learn representations. 
When DANN performs terribly, our method can be influenced.
On the other hand, $K$ is a hyperparameter on the current method, and we can miss the best results due to limited searches.
3) No matter how the method performs in accuracy, \methodmal performs the best, which indicates that the features \method extracted are outstanding.
Moreover, compare to \methodmcp, \methodmal still has another remarkable improvement.

\subsection{Sensor-based Human Activity Recognition}
Finally, we construct three diverse OOD settings by leveraging four sensor-based human activity recognition datasets: DSADS~\cite{barshan2014recognizing}, USC-HAD~\cite{zhang2012usc}, UCI-HAR~\cite{anguita2012human}, and PAMAP2~\cite{reiss2012introducing}.
UCI daily and sports dataset (\textbf{DSADS})~\cite{barshan2014recognizing} consists of 19 activities collected from 8 subjects wearing body-worn sensors on 5 body parts. 
USC-SIPI human activity dataset (\textbf{USC-HAD})~\cite{zhang2012usc} is composed of 14 subjects (7 male, 7 female, aged 21 to 49) executing 12 activities with a sensor tied on the front right hip.
\textbf{UCI-HAR}~\cite{anguita2012human} is collected by 30 subjects performing 6 daily living activities with a waist-mounted smartphone. 
\textbf{PAMAP2}~\cite{reiss2012introducing} contains data from 18 activities, performed by 9 subjects wearing 3 sensors.
These datasets are collected from different people and positions using accelerometer and gyroscope, with $11,741,000$ instances in total.
For DSADS, we directly utilize data split by the providers.
The final dimension shape is $45\times 1\times 125$.
$45 = 5\times3\times 3$ where $5$ means five positions, the first $3$ means three sensors, and the second $3$ means each sensor has three axes.
For USC-HAD, the window size is 200 and the step size is 100.
The final dimension shape is $6\times 1\times 200$.
For PAMAP2, the window size is 200 and the step size is 100.
The final dimension shape is $27\times 1\times 200$.
For UCI-HAR, we directly utilize data split by the providers.
The final dimension shape is $6\times 1\times 128$.
(1) \textbf{X-person generalization (Cross-person)} aims to learn generalized models for different persons. 
This setting utilizes DSADS, USC-HAD, and PAMAP2.
Within each dataset, we randomly split the data into four groups. 
For DSADS, we choose running, ascending stairs, descending stairs, rope jumping, and playing basketball as OOD classes.
For USC-HAD, we choose Running Forward and Jumping Up as OOD classes.
For PAMAP2, we choose running, Nordic walking, and rope jumping as OOD classes. 
(2) \textbf{X-position generalization (Cross-position)} aims to learn generalized models for different sensor positions. 
This setting uses DSADS and data from each position denotes a different domain.
Thereby, a sample is split into five samples in the first dimension and the final dimension shape is $9\times 1\times 125$.
We choose running, ascending stairs, descending stairs, rope jumping, and playing basketball as OOD classes.
(3) \textbf{X-dataset generalization (Cross-dataset)} aims to learn generalized models for different datasets.
This setting uses all four datasets, and each dataset corresponds to a different domain.
Six common classes are selected.
Two sensors from each dataset that belong to the same position are selected and data is down-sampled to have the same dimension.
The final dimension shape is $6\times 1\times 50$.
We choose ascending and descending as OOD classes.

\begin{table*}[!t]
\centering
\caption{ID accuracy under X-Person, X-Position, and X-Dataset settings. \textbf{Bold} means the best while \underline{underline} means the second-best.}
\label{tab1-har-inacc}
\resizebox{0.95\textwidth}{!}{
\begin{tabular}{cllllll|lllll|lllll}
\toprule
\multicolumn{1}{l}{Task}    & Datasets               & \multicolumn{5}{c|}{DSADS}                                                          & \multicolumn{5}{c|}{USC-HAD}                                                                                                & \multicolumn{5}{c}{PAMAP2}                                                         \\
\multirow{7}{*}{X-Person}   & Targets                & 0              & 1              & 2              & 3              & AVG            & 0              & 1                                              & 2                      & 3              & AVG            & 0              & 1              & 2              & 3              & AVG            \\
\midrule                            & ERM                    & 86.37          & 72.32          & 80.77          & \underline{76.85}    & 79.08          & \textbf{81.85} & 58.99                                          & 71.61                  & 58.25          & 67.68          & 86.84          & \underline{80.62}    & 49.65          & 85.64          & 75.69          \\
                            & CORAL                  & 85.06          & \underline{81.25}    & \underline{89.82}    & 72.98          & \underline{82.28}    & 78.44          & 57.71                                          & 74.26                  & \underline{60.35}    & 67.69          & 89.82          & 79.72          & 59.58          & 84.56          & 78.42          \\
                            & DANN                   & \underline{87.14}    & 77.44          & 81.49          & 76.13          & 80.55          & 80.69          & 57.50                                          & 72.35                  & 59.77          & 67.58          & 89.92          & 77.50          & 51.83          & 87.85          & 76.78          \\
                            & GroupDRO               & 85.18          & 76.25          & 83.75          & 75.71          & 80.22          & 79.71          & \underline{62.50}                                    & \underline{74.30}            & 58.43          & \underline{68.73}    & \textbf{91.32} & 78.23          & 60.68          & 88.73          & \textbf{79.74} \\
                            & ANDMask                & 85.00          & 75.89          & 84.40          & 71.13          & 79.11          & 76.77          & 59.98                                          & 72.13                  & 56.24          & 66.28          & 89.82          & 77.82          & \underline{50.80}    & \underline{90.16}    & 77.15          \\
                            & \method & \textbf{87.68} & \textbf{82.26} & \textbf{90.30} & \textbf{86.85} & \textbf{86.77} & \underline{81.61}    & \textbf{66.08}                                 & \textbf{75.82}         & \textbf{71.15} & \textbf{73.66} & \underline{91.25}    & \textbf{82.70} & \textbf{62.62} & \textbf{90.93} & \textbf{81.88} \\ \midrule
\multirow{7}{*}{X-Position} & Targets                & 0              & 1              & 2              & 3              & \multicolumn{1}{l}{4}              & \multicolumn{1}{l|}{AVG}            & \multicolumn{1}{c}{\multirow{7}{*}{X-Dataset}} & Targets                & DSADS          & \multicolumn{1}{l}{USC-HAD}        & \multicolumn{1}{l}{HAR}            & PAMAP2         & AVG            &                &                \\ \midrule
                            & ERM                    & 33.38          & 21.62          & \underline{35.43}    & 22.50          & \multicolumn{1}{l}{19.88}          & \multicolumn{1}{l|}{26.56}          & \multicolumn{1}{c}{}                           & ERM                    & 37.11          & \multicolumn{1}{l}{53.05}          & \multicolumn{1}{l}{58.84}          & 34.34          & 45.83          &                &                \\
                            & CORAL                  & \underline{41.79}    & 27.77          & 32.16          & \underline{28.14}    & \multicolumn{1}{l}{\underline{26.07}}    & \multicolumn{1}{l|}{\underline{31.18}}    & \multicolumn{1}{c}{}                           & CORAL                  & 40.00          & \multicolumn{1}{l}{51.47}          & \multicolumn{1}{l}{59.57}          & 62.18          & 53.31          &                &                \\
                            & DANN                   & 40.04          & \underline{29.64}    & 30.74          & 22.59          & \multicolumn{1}{l}{23.93}          & \multicolumn{1}{l|}{29.39}          & \multicolumn{1}{c}{}                           & DANN                   & \underline{40.96}    & \multicolumn{1}{l}{54.49}          & \multicolumn{1}{l}{\underline{61.44}}    & 66.24          & 55.78          &                &                \\
                            & GroupDRO               & 31.64          & 24.02          & 32.74          & 29.17          &\multicolumn{1}{l}{20.94}          & \multicolumn{1}{l|}{27.70}          & \multicolumn{1}{c}{}                           & GroupDRO               & 40.00          & \multicolumn{1}{l}{\underline{55.77}}    & \multicolumn{1}{l}{60.99}          & \underline{72.10}    & \underline{57.22}    &                &                \\
                            & ANDMask                & 31.93          & 24.03          & 33.41          & 23.93          & \multicolumn{1}{l}{23.45}          & \multicolumn{1}{l|}{27.35}          & \multicolumn{1}{c}{}                           & ANDMask                & 39.57          & \multicolumn{1}{l}{48.08}          & \multicolumn{1}{l}{58.39}          & 61.68          & 51.93          &                &                \\
                            & \method & \textbf{43.10} & \textbf{30.57} & \textbf{37.57} & \textbf{31.76} & \multicolumn{1}{l}{\textbf{26.89}} &\multicolumn{1}{l|}{ \textbf{33.98}} & \multicolumn{1}{c}{}                           & \method & \textbf{70.70} & \multicolumn{1}{l}{\textbf{60.48}} & \multicolumn{1}{l}{\textbf{64.00}} & \textbf{74.28} & \textbf{67.36} &                &                           \\ \bottomrule
\end{tabular}}
\end{table*}

\begin{table}[htbp]
\centering
\caption{Average AUROC and AUPR under X-Person, X-Position, and X-Dataset settings. \textbf{Bold} means the best.}
\label{tab1-har-au}
\resizebox{0.5\textwidth}{!}{
\begin{tabular}{clllclll}
\toprule
\multicolumn{1}{l}{}         & Datasets                   & \multicolumn{2}{c}{DSADS}       & \multicolumn{2}{c}{USC-HAD}                         & \multicolumn{2}{c}{PAMAP2}      \\
\multirow{13}{*}{X-Person}   & Metrics                    & AUROC          & AUPR           & \multicolumn{1}{l}{AUROC}          & AUPR           & AUROC          & AUPR           \\ \midrule
                             & ERM-MCP                    & 66.82          & 80.52          & \multicolumn{1}{l}{15.50}          & 78.71          & 45.95          & 91.85          \\
                             & CORAL-MCP                  & 75.54          & 89.33          & \multicolumn{1}{l}{23.00}          & 80.59          & 60.79          & 94.84          \\
                             & DANN-MCP                   & 70.11          & 85.56          & \multicolumn{1}{l}{\underline{28.82}}    & 82.45          & \underline{65.59}    & \underline{96.01}    \\
                             & GroupDRO-MCP               & \underline{78.02}    & \underline{90.42}    & \multicolumn{1}{l}{27.98}          & \underline{82.47}    & 65.50          & 95.59          \\
                             & ANDMask-MCP                & 69.60          & 85.41          & \multicolumn{1}{l}{17.71}          & 79.35          & 51.15          & 93.51          \\
                             & \methodmcp  & \textbf{82.38} & \textbf{92.62} & \multicolumn{1}{l}{\textbf{37.50}} & \textbf{85.55} & \textbf{75.02} & \textbf{97.16} \\
                             & ERM-MAH                    & 65.51          & 84.28          & \multicolumn{1}{l}{98.67}          & 99.80          & 74.59          & 96.52          \\
                             & CORAL-MAH                  & 67.96          & 85.02          & \multicolumn{1}{l}{99.22}          & 99.89          & \underline{88.31}    & \underline{98.89}    \\
                             & DANN-MAH                   & 63.33          & 78.88          & \multicolumn{1}{l}{\underline{99.36}}    & \underline{99.91}    & 86.15          & 98.61          \\
                             & GroupDRO-MAH               & \underline{73.60}    & \underline{87.88}    & \multicolumn{1}{l}{99.25}          & 99.89          & 87.37          & 98.77          \\
                             & ANDMask-MAH                & 51.51          & 75.11          & \multicolumn{1}{l}{98.74}          & 99.82          & 85.71          & 98.57          \\
                             & \method-MAH & \textbf{93.90} & \textbf{97.36} & \multicolumn{1}{l}{\textbf{99.56}} & \textbf{99.94} & \textbf{90.37} & \textbf{99.14} \\ \midrule
\multirow{13}{*}{X-Position} & Metrics                    & AUROC          & AUPR           & \multirow{13}{*}{X-Dataset}        & AUROC          & AUPR           &                \\ \midrule
                             & ERM-MCP                    & 42.23          & 67.65          &                                    & 40.03          & 75.91          &                \\
                             & CORAL-MCP                  & \underline{43.43}    & 68.75          &                                    & 44.60          & 77.80          &                \\
                             & DANN-MCP                   & 43.20          & 68.33          &                                    & 41.21          & 77.97          &                \\
                             & GroupDRO-MCP               & 42.95          & \underline{69.11}    &                                    & \underline{47.91}    & \underline{79.24}    &                \\
                             & ANDMask-MCP                & 39.14          & 66.57          &                                    & 38.97          & 74.38          &                \\
                             & \methodmcp  & \textbf{54.58} & \textbf{76.78} &                                    & \textbf{60.23} & \textbf{85.67} &                \\
                             & ERM-MAH                    & \underline{87.92}    & \underline{95.40}    &                                    & 55.66          & 80.49          &                \\
                             & CORAL-MAH                  & 87.35          & 95.26          &                                    & \underline{75.17}    & \underline{90.77}    &                \\
                             & DANN-MAH                   & 87.04          & 95.01          &                                    & 65.07          & 85.46          &                \\
                             & GroupDRO-MAH               & 87.51          & 95.24          &                                    & 73.66          & 89.70          &                \\
                             & ANDMask-MAH                & 86.33          & 94.37          &                                    & 53.83          & 77.03          &                \\
                             & \method-MAH & \textbf{91.85} & \textbf{97.04} &                                    & \textbf{85.50} & \textbf{94.86} &                     \\ \bottomrule
\end{tabular}}
\end{table}

The ID classification results are shown in \tableautorefname~\ref{tab1-har-inacc} while OOD detection results are shown in \tableautorefname~\ref{tab1-har-au}.
We have the following observation from these results.
1) Similar to EMG, our method has the best ID accuracy, AUROC, and AUPR on average, which demonstrates that our method has excellent abilities of classification and OOD detection.
2) When the task is difficult: in the X-Person setting, USC-HAD may be the most difficult task.
Although it has more samples, it contains 14 subjects with only two sensors in one position, which may bring more difficulty in learning.
The results prove the above argument that all methods perform terribly on this benchmark while ours has the largest improvement.
3) When datasets are significantly more diverse: compared to X-Person and X-Position settings, X-Dataset may be more difficult since all datasets are totally different and samples are influenced by subjects, devices, sensor positions, and some other factors.
In this setting, our method is substantially better than others.
4) Limited data: for tasks on DSADS, the number of training data is limited,
In this case, enhancing diversity can still bring a remarkable improvement and our method can boost the performance.
5) For time series OOD detection, it is better to utilize Mahalanobis distance since it is less influenced by over-confidence in deep learning models.
For USC-HAD, \methodmal even achieves an improvement of over $62\%$ on average AUROC compared to \methodmcp.

\section{Experiments on OOD generalization}
\label{sec-exp-cls}
We perform evaluations on four diverse time series classification tasks: gesture recognition, speech commands recognition, wearable stress\&affect detection, and sensor-based activity recognition.
Settings and implementations are similar to OOD detection.
Besides the comparison methods mentioned above, we add more latest methods designed for classification.
Mixup~\cite{zhang2018mixup} is a method that utilizes interpolation to generate more data for better generalization. Ours mainly focuses on generalized representation learning.
RSC~\cite{huang2020self} is a self-challenging training algorithm that forces the network to activate features as much as possible by manipulating gradients. It belongs to gradient operation-based DG while ours is to learn generalized features.
GILE~\cite{qian2021latent} is a disentanglement method designed for cross-person human activity recognition. It is based on VAEs and requires domain labels.
AdaRNN~\cite{du2021adarnn} is a method with a two-stage that is non-differential and is tailored for RNN. A specific algorithm is designed for splitting. Ours is universal and is differential with better performance.
Per-segment accuracy is the evaluation metric.

\subsection{Gesture Recognition}
First, we evaluate \method on EMG for gestures Data Set~\cite{lobov2018latent}.
We randomly divide 36 subjects into four domains (i.e., $0, 1, 2, 3$).
\figurename~\ref{fig:emg} shows that with the same backbone, our method achieves the best average performance and is $\mathbf{4.3}\%$ better than the second-best method.
\method even outperforms AdaRNN which has a stronger backbone.

\begin{table*}[t!]
\centering
\caption{Accuracy on cross-person generalization. ``Target'' $0 \sim 4$ denotes the unseen test set.}
\resizebox{.8\textwidth}{!}{%
\begin{tabular}{l|ccccc|ccccc|ccccc|c}
\toprule
                       & \multicolumn{5}{c|}{DSADS}                                                          & \multicolumn{5}{c|}{USC-HAD}                                                       & \multicolumn{5}{c|}{PAMAP2}                                                      &   ALL             \\
                       
Target                 & 0              & 1              & 2              & 3              & AVG              & 0              & 1             & 2              & 3              & AVG              & 0              & 1             & 2              & 3              & AVG           & AVG              \\ \midrule
ERM                    & 83.1          & 79.3           & 87.8          & 71.0          & 80.3           & 81.0          & 57.7         & 74.0          & 65.9          & 69.7          & \underline{90.0}    & 78.1         & 55.8          & 84.4          & 77.1       & 75.7          \\
DANN                   & 89.1          & 84.2          & 85.9          & 83.4          & 85.6          & 81.2          & 57.9         & \underline{76.7}    & 70.7    & \underline{71.6} & 82.2          & 78.1         & 55.4          & 87.3          & 75.7       & 77.7          \\
CORAL                  & \underline{91.0}          & 85.8          & 86.6          & 78.2          & 85.4          & 78.8          & 58.9         & 75.0          & 53.7          & 66.6          & 86.2          & 77.8         & 49.0             & \underline{87.8}    & 75.2        & 75.7          \\
Mixup                  & 89.6          & 82.2          &89.2    & \textbf{86.9} & \underline{87.0}    & 80.0          & \textbf{64.1}   & 74.3          & 61.3          & 69.9          & 89.4          & \underline{80.3}    & 58.4          & 87.7          & \underline{79.0} & \underline{78.6}    \\
GroupDRO               & \textbf{91.7}    & \underline{85.9}    & 87.6          & 78.3          & 85.9          & 80.1          & 55.5         & 74.7          & 60.0          & 67.6          & 85.2          & 77.7         & 56.2          & 85.0          & 76.0       & 76.5          \\
RSC                    & 84.9          & 82.3          & 86.7          & 77.7          & 82.9          & \underline{81.9}    & 57.9         & 73.4          & 65.1          & 69.6          & 87.1          & 76.9         & \underline{60.3} & \textbf{87.8} & 78.0       & 76.9          \\
ANDMask                & 85.0          & 75.8          & 87.0          & 77.6          & 81.4          & 79.9          & 55.3         & 74.5          & 65.0          & 68.7          & 86.7          & 76.4         & 43.6          & 85.6          & 73.1       & 74.4          \\
GILE                   & 81.0             & 75.0             & 77.0             & 66.0             & 74.7          & 78,0             & 62.0            & 77.0             & 63.0             & 70.0             & 83.0             & 68.0            & 42.0             & 76.0             & 67.5        & 70.7          \\
AdaRNN & 80.9 &	75.5 &	\textbf{90.2} &	75.5 &	80.5 &78.6 &	55.3 &	66.9 &	\textbf{73.7}&	68.6 &81.6 &	71.8 &	45.4 &	82.7	 &70.4 &73.2\\
\method & 90.4	&\textbf{86.5}&	\underline{90.0}&	\underline{86.1}&	\textbf{88.2} &\textbf{82.6}	&\underline{63.5}&	\textbf{78.7}&	\underline{71.3}&	\textbf{74.0}&	\textbf{91.0}&	\textbf{84.3}&	\textbf{60.5}	&87.7	&\textbf{80.8}&	\textbf{81.0}
\\
\bottomrule
\end{tabular}}
\label{tab:my-table-crosspeople}
\end{table*}
\begin{table*}[t!]
\centering
\caption{Classification accuracy on cross-position, cross-dataset, and one-to-another generalization.}
\label{tab:my-table-otherhar}
\resizebox{.9\textwidth}{!}{%
\begin{tabular}{l|cccccc|ccccc|cccc}
\toprule
                       & \multicolumn{6}{c|}{Cross-position generalization}                                                                      & \multicolumn{5}{c|}{Cross-dataset generalization}                                  & \multicolumn{4}{c}{One-Person-To-Another}                                          \\

Target                 & 0              & 1              & 2              & 3              & 4              & AVG             & 0              & 1              & 2              & 3              & AVG &DSADS&USC-HAD&PAMAP2&AVG              \\ \midrule
ERM                    & 41.5          & 26.7          & 35.8          & 21.4          & 27.3          & 30.6          & 26.4          & 29.6          & 44.4          & 32.9          & 33.3   &51.3&	46.2	&53.1&50.2
           \\
DANN                   & 45.4          & 25.3          & 38.1          & 28.9          & 25.1          & 32.6          & 29.7          & 45.3          & \underline{46.1} & 43.8          & \underline{41.2}     &-&-&- &-    \\
CORAL                  & 33.2          & 25.2          & 25.8          & 22.3          & 20.6          & 25.4          & 39.5          & 41.8          & 39.1           & 36.6          & 39.2      &-&-&-&-               \\
Mixup                  & \textbf{48.8}    & \textbf{34.2} & 37.5          & \underline{29.5}     & \underline{29.9}    & \underline{36.0}    & 37.3          & \textbf{47.4}    & 40.2          & 23.1          & 37.0   &\underline{62.7}&	46.3	&58.6&55.8
       \\
GroupDRO               & 27.1          & 26.7          & 24.3          & 18.4          & 24.8          & 24.3          & \textbf{51.4} & 36.7          & 33.2           & 33.8           & 38.8    &51.3&	48.0	&53.1    &50.8           \\
RSC                    & 46.6          & 27.4          & 35.9          & 27.0          & 29.8          & 33.3          & 33.1           & 39.7           & 45.3   & \underline{45.9}    & 41.0& 59.1&	\underline{49.0}&	\underline{59.7}&\underline{55.9}
                \\
ANDMask                & 47.5          & 31.1          & \underline{39.2}    & 30.2          & 29.9           & 35.6          & 41.7          & 33.8          & 43.2          & 40.2          & 39.7  &57.2&	45.9&	54.3&52.5
            \\
\method & \underline{47.7} &	\underline{32.9}	&\textbf{44.5}&	\textbf{31.6}	&\textbf{30.4}&	\textbf{37.4}&	\underline{48.7}&	\underline{46.9}	&\textbf{49.0}&	\textbf{59.9}	&\textbf{51.1}&\textbf{67.6}&	\textbf{55.0}&	\textbf{62.5}&\textbf{61.7}

\\
\bottomrule
\end{tabular}}
\end{table*}

\subsection{Speech Commands}

\begin{figure}[t!]
    \centering
    \subfigure[EMG]{
        \includegraphics[height=0.11\textwidth]{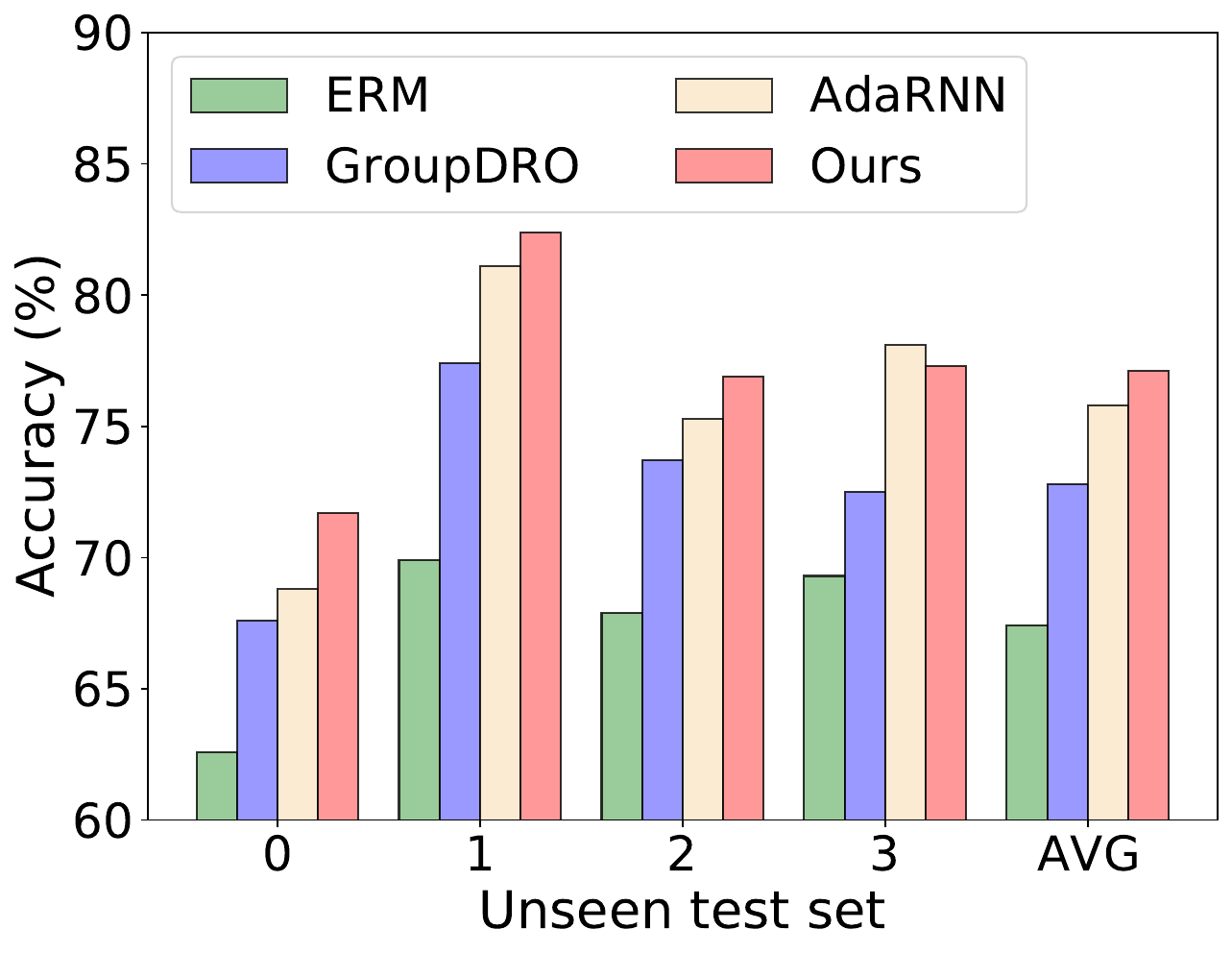}
        \label{fig:emg}
    }
    \subfigure[Speech commands]{
        \includegraphics[height=0.11\textwidth]{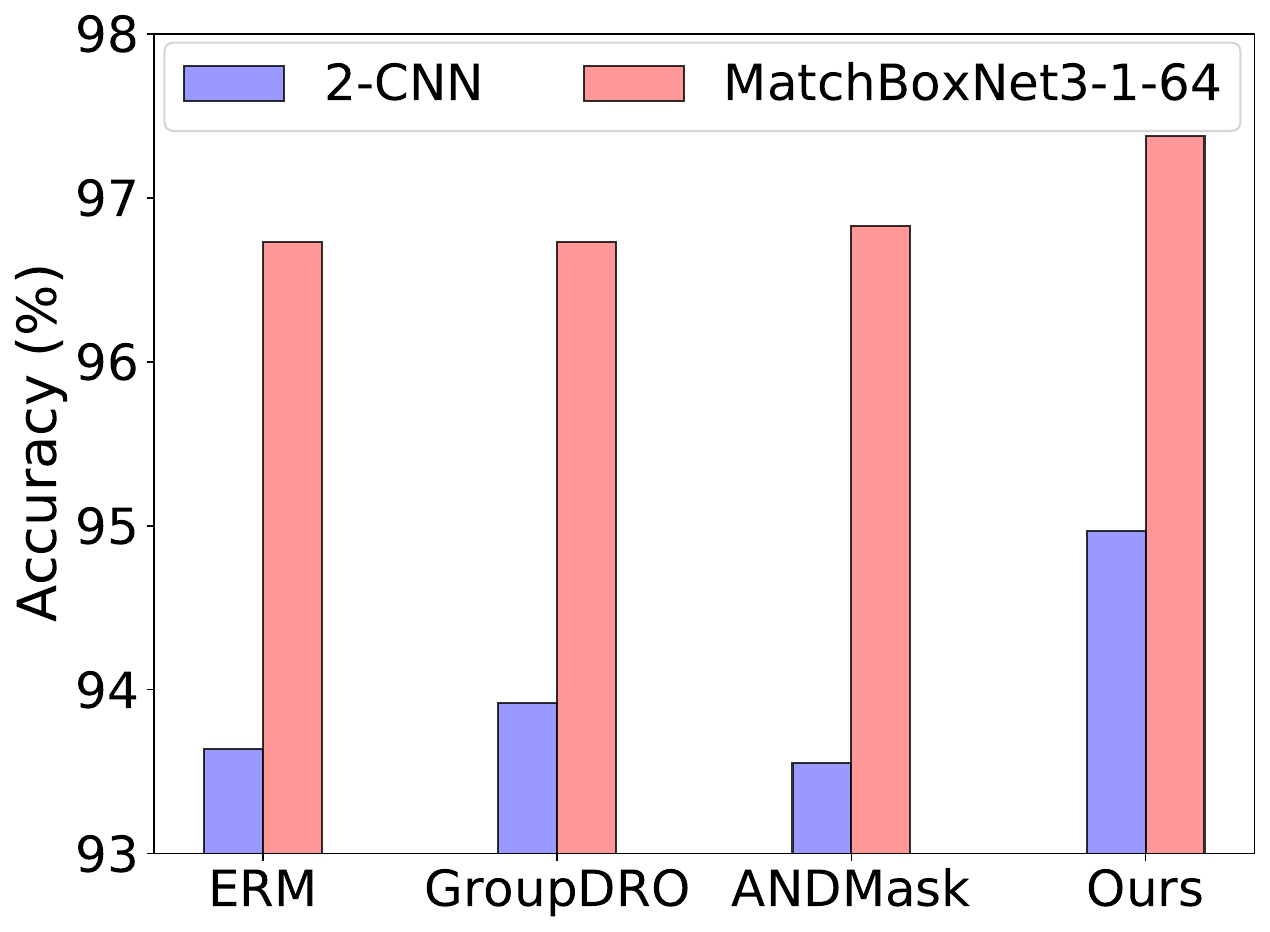}
        \label{fig:speech}
    }
    \subfigure[WESAD]{
        \includegraphics[height=0.11\textwidth]{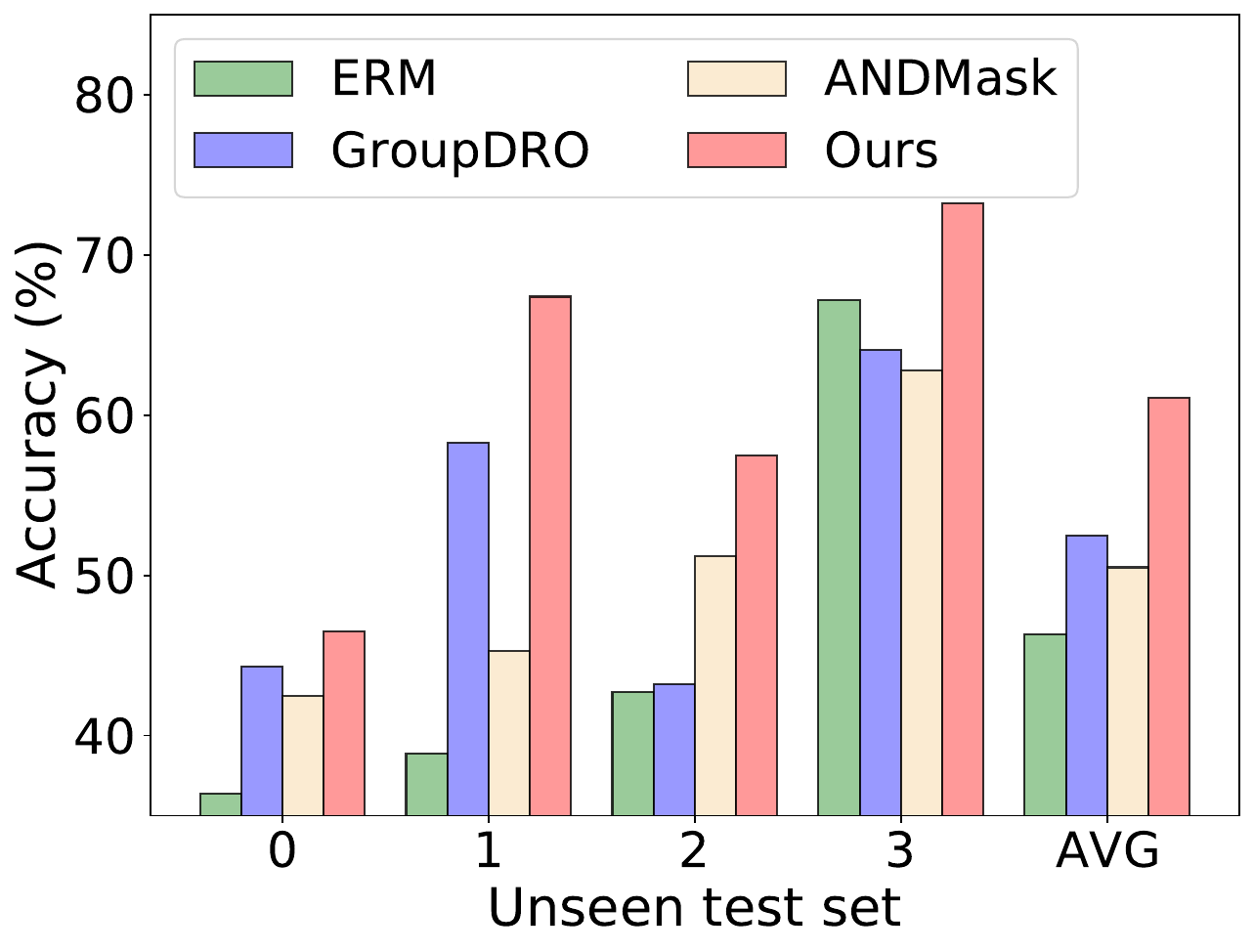}
        \label{fig:wesad}
    }
    
    \caption{Results on EMG, Speech commands and WESAD.}
    \label{fig-cls-res-two}
\end{figure}

Then, we adopt a regular speech recognition task, the Speech Commands dataset~\cite{warden2018speech}.
It consists of one-second audio recordings of both background noise and spoken words such as `left' and `right'.
It is collected from more than 2,000 persons, thus is more complicated.
Following \cite{kidger2020neuralcde}, we use 34,975 time series corresponding to ten spoken words to produce a balanced classification problem.
Since this dataset is collected from multiple persons, the training and test distributions are different, which is also an OOD problem with one training domain.
We do not split samples due to too many subjects and few audios per subject.
\figurename~\ref{fig:speech} shows the results on two different backbones.
Compared with GroupDRO, \method has over $\mathbf{1}\%$ improvement with a basic CNN backbone and over $\mathbf{0.6}\%$ improvement with a strong backbone MatchBoxNet3-1-64~\cite{majumdar2020matchboxnet}.
It demonstrates the superiority of our method on a regular time-series benchmark containing massive distributions.

\subsection{Wearable Stress and Affect Detection}

    

We further evaluate \method on WESAD.
We split 15 subjects into four domains.
Results \figurename~\ref{fig:wesad} showed that our method achieves the best performance compared to other state-of-the-art methods with an improvement of over $\textbf{8}\%$ on this larger dataset.

\subsection{Sensor-based Human Activity Recognition}
Finally, we construct \emph{four} diverse OOD settings on Sensor-based Human Activity Recognition.
Besides the settings mentioned above, we add another difficult setting, One-Person-To-Another.
\textbf{One-Person-To-Another} aims to learn generalized models for different persons from data of a single person.
This setting adopts DSADS, USC-HAD, and PAMAP2.
In each dataset, we randomly select four pairs of persons where one is the training and the other is the test.
Four tasks are $1\rightarrow 0, 3\rightarrow2, 5\rightarrow4,$ and $7\rightarrow6$.
Each number corresponds to one subject.
And the final dimension shape is $45\times 1 \times 125, 6\times 1\times 200,$ and $27\times 1\times 200$ for DSADS, USC-HAD, and PAMAP2 respectively.
\footnote{In the One-Person-To-Another setting, we only report the average accuracy of four tasks on each dataset.}

\tablename~\ref{tab:my-table-crosspeople} and \ref{tab:my-table-otherhar} 
show the results on four settings for HAR, where our method significantly outperforms the second-best baseline by $\mathbf{2.4}\%$, $\mathbf{1.4}\%$, $\mathbf{9.9}\%$, and $\mathbf{5.8}\%$ respectively.
All results show the superiority of \method.

We observe more insightful conclusions similar to OOD detection.
(1) \emph{When the task is difficult:} In the Cross-Person setting, USC-HAD may be the most difficult task.
The results prove the above argument that all methods perform terribly on this benchmark while ours has the largest improvement.
(2) \emph{When datasets are significantly more diverse:} Compared to Cross-Person and Cross-Position settings, Cross-Dataset may be more difficult. 
In this setting, our method is substantially better than others.
(3) \emph{Limited data:} Compared with the Cross-Person setting, One-Person-To-Another is more difficult since it has fewer data samples.
In this case, enhancing diversity can bring a remarkable improvement and our method can boost the performance.

\begin{figure*}[t!]
    \centering
    \subfigure[Class-invariant feat.]{
        \includegraphics[width=0.23\textwidth]{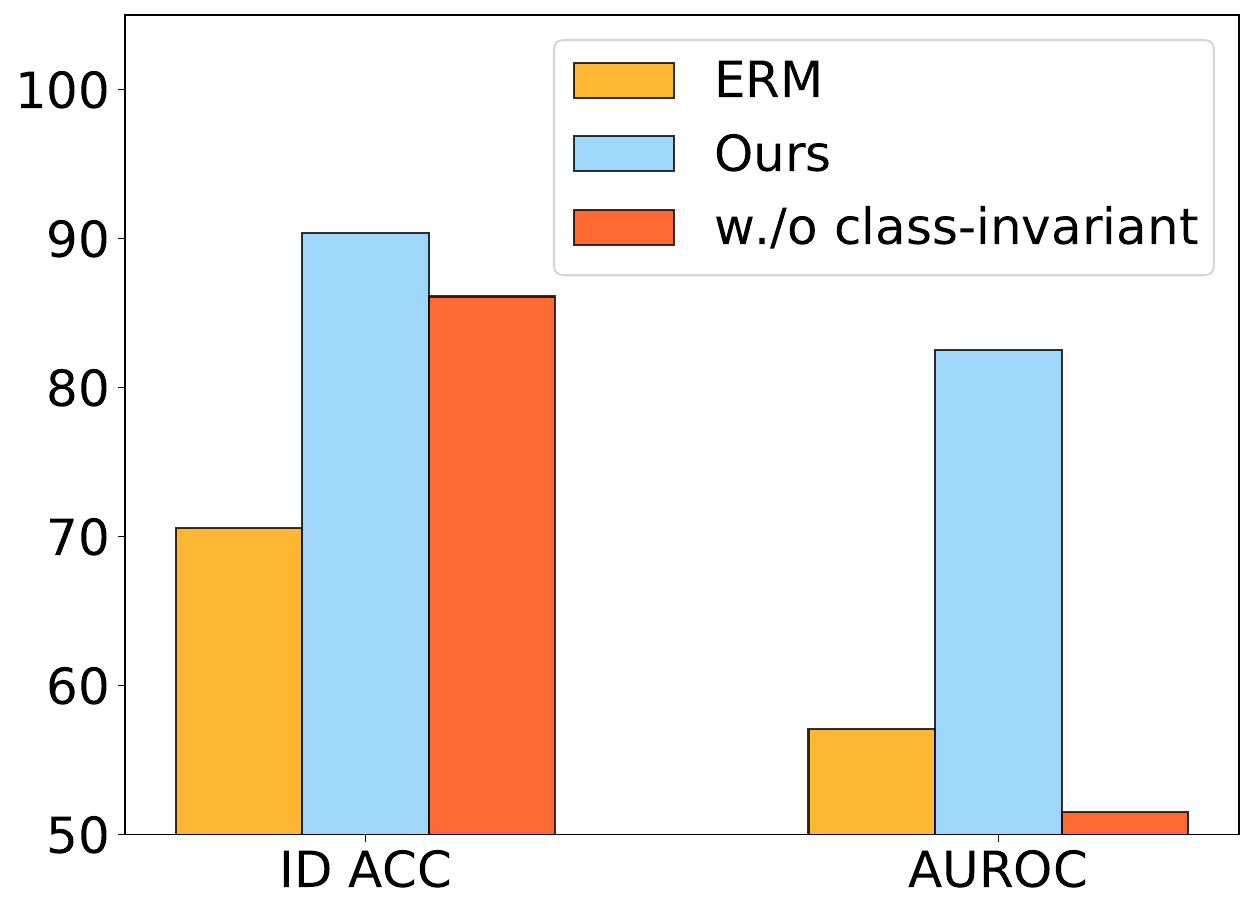}
        \label{fig1-abla-adv}
    }
    \subfigure[Update feat. with $d'$]{
        \includegraphics[width=0.23\textwidth]{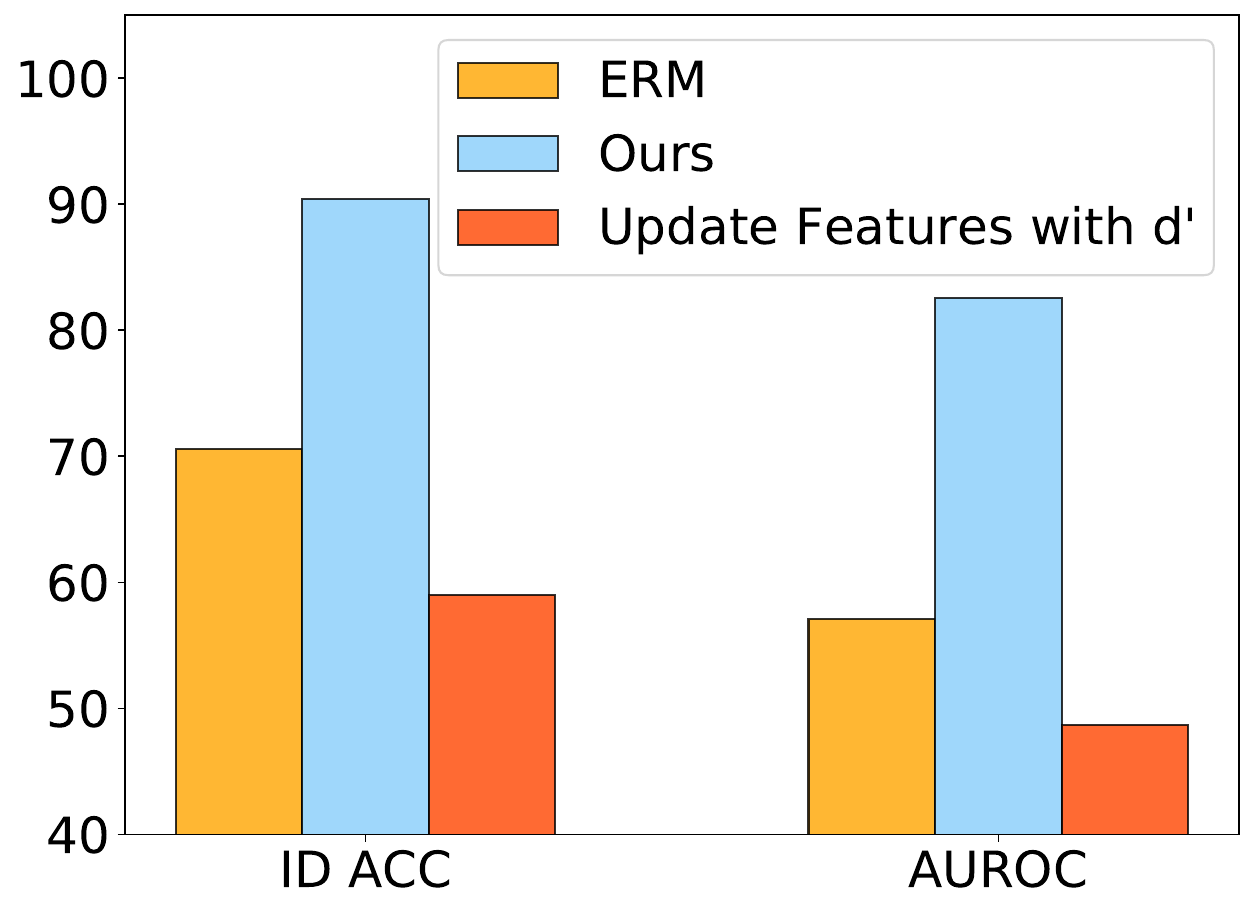}
        \label{fig1-abla-d}
    }  
    \subfigure[Update feature with $y$]{
        \includegraphics[width=0.23\textwidth]{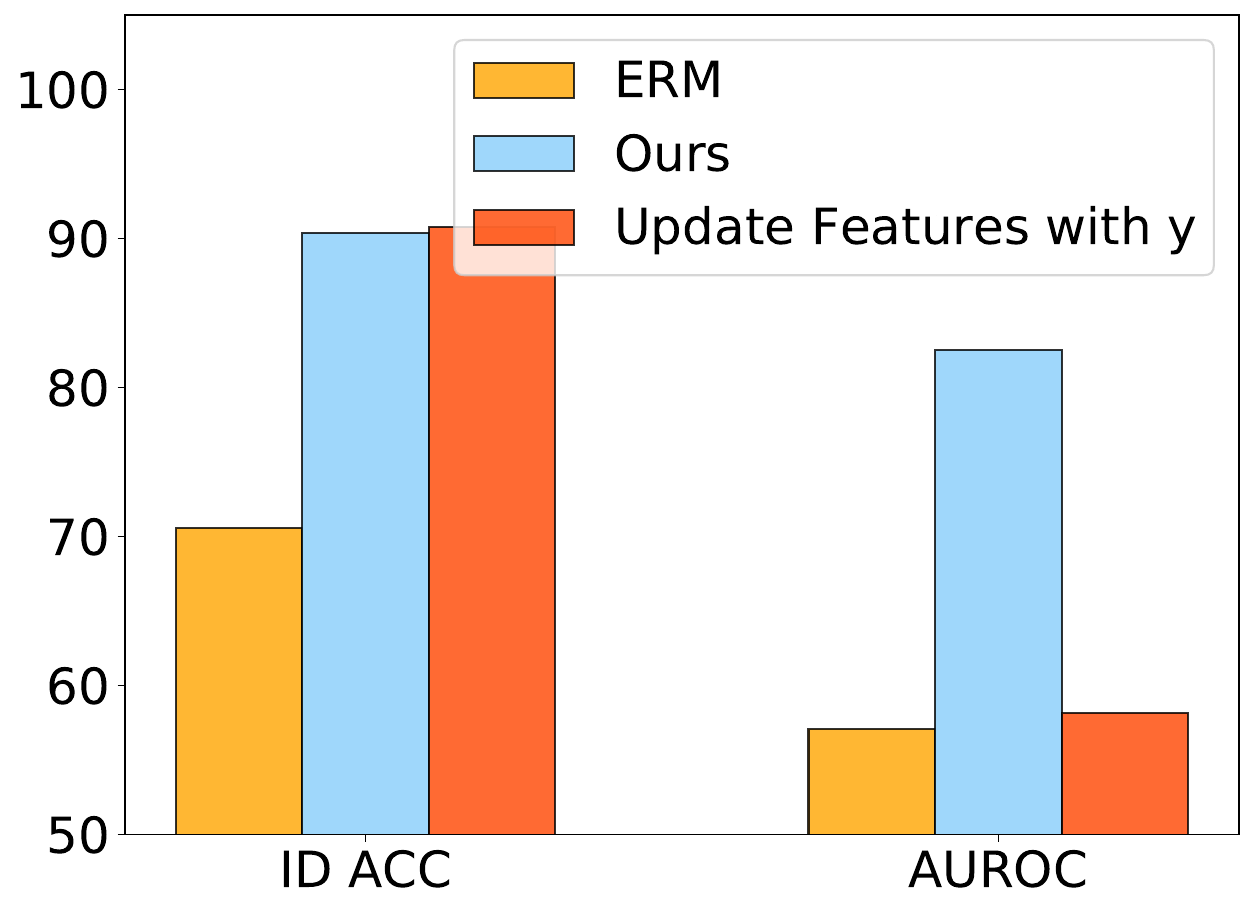}
        \label{fig1-abla-y}
    }
    \subfigure[\#Domain $K$ (EMG)]{
        \includegraphics[width=0.23\textwidth]{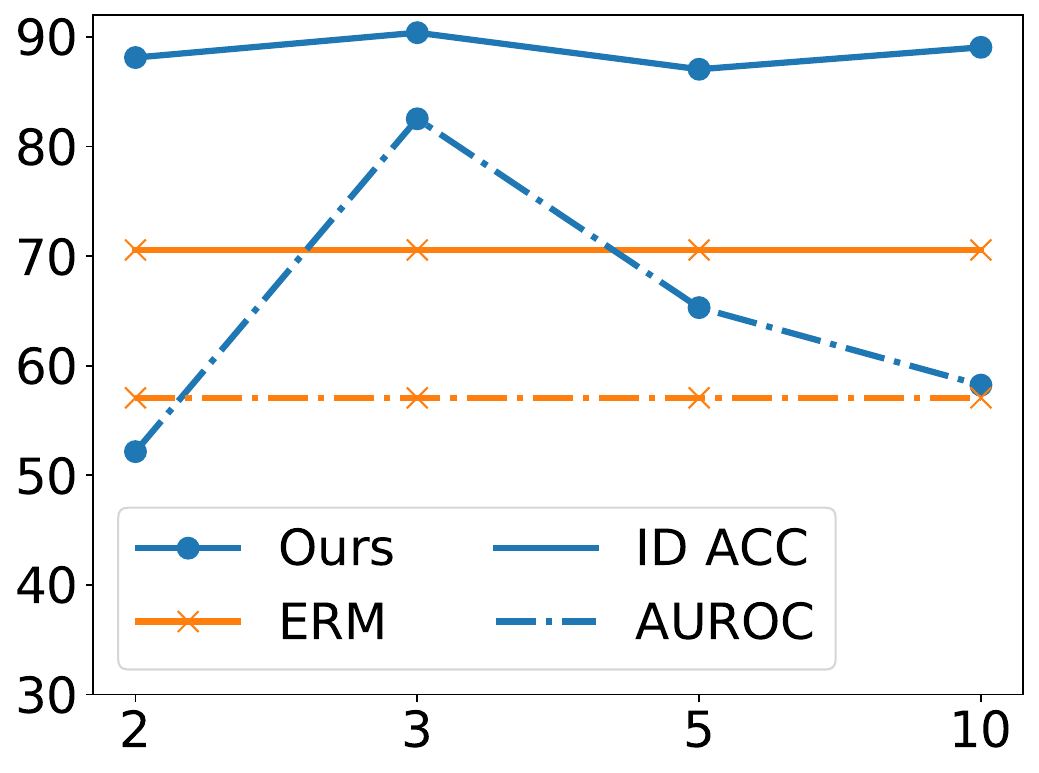}
        \label{fig1-ablat-k}
    }
    \caption{Ablation study of \method for detection.}
    \label{fig1-ablat}
\end{figure*}

\begin{figure*}[t!]
    \centering
    \subfigure[Class-invariant feat.]{
        \includegraphics[width=0.23\textwidth]{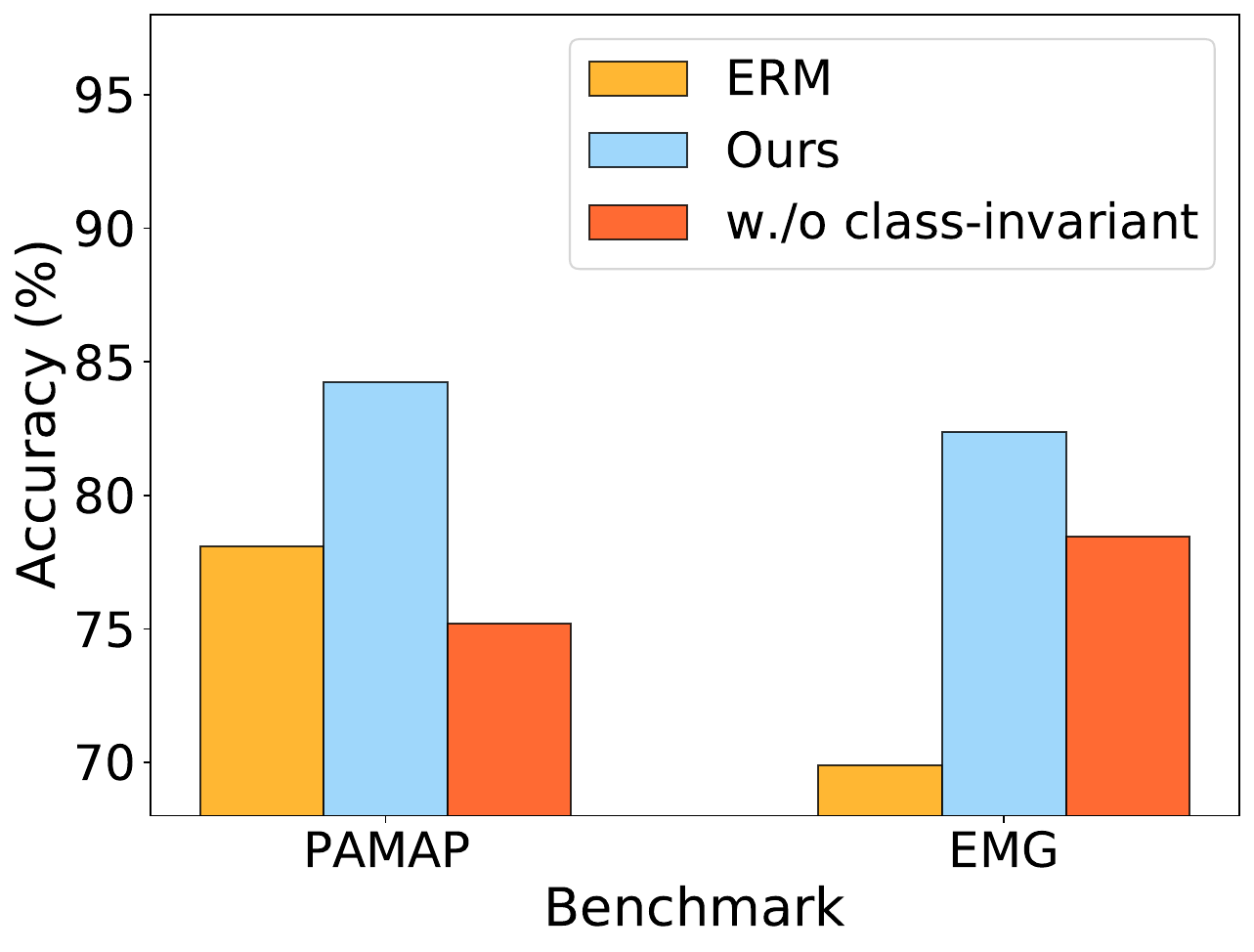}
        \label{fig-abla-adv}
    }
    \subfigure[Update feat. with $d'$]{
        \includegraphics[width=0.23\textwidth]{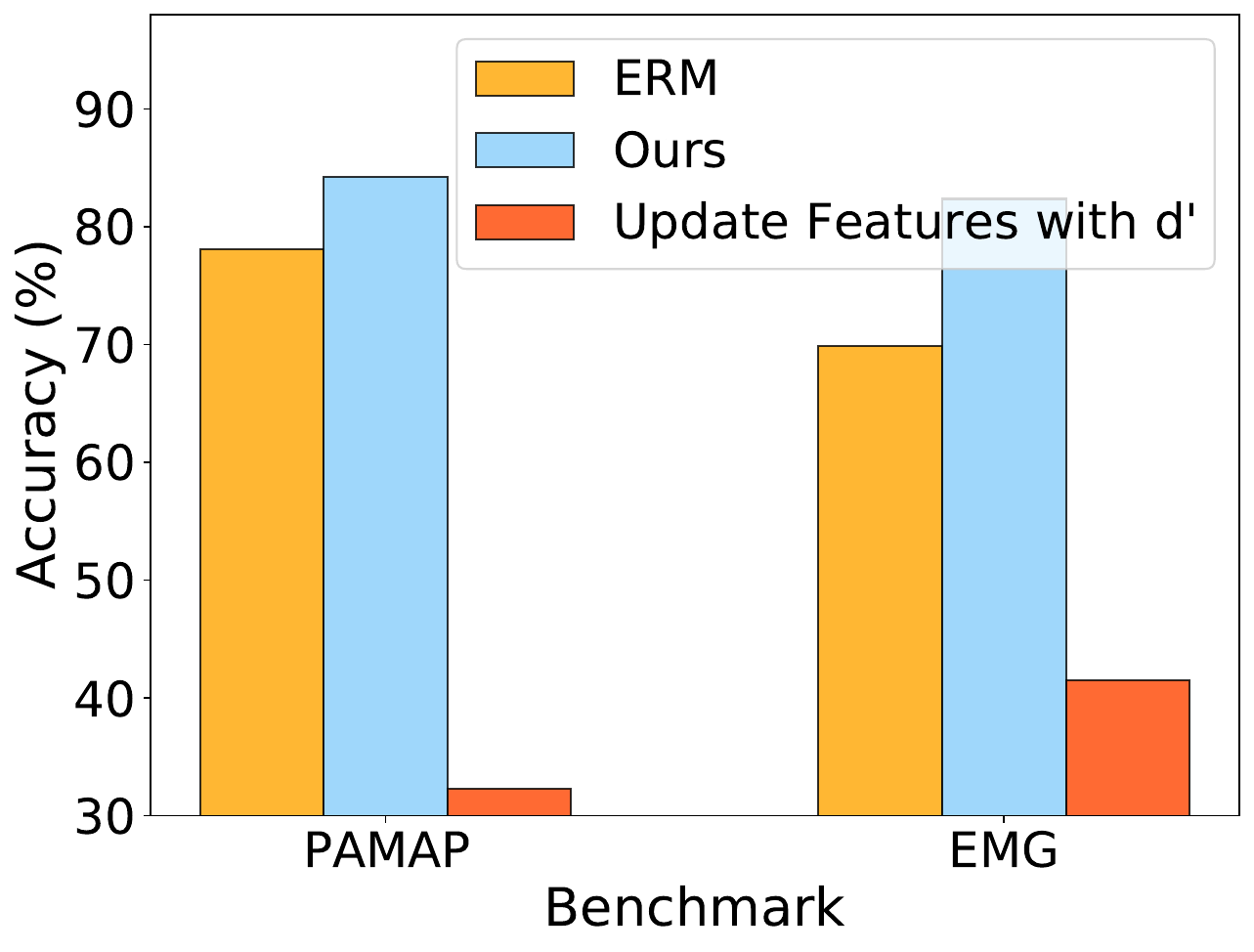}
        \label{fig-abla-d}
    }  
    \subfigure[Update feature with $y$]{
        \includegraphics[width=0.23\textwidth]{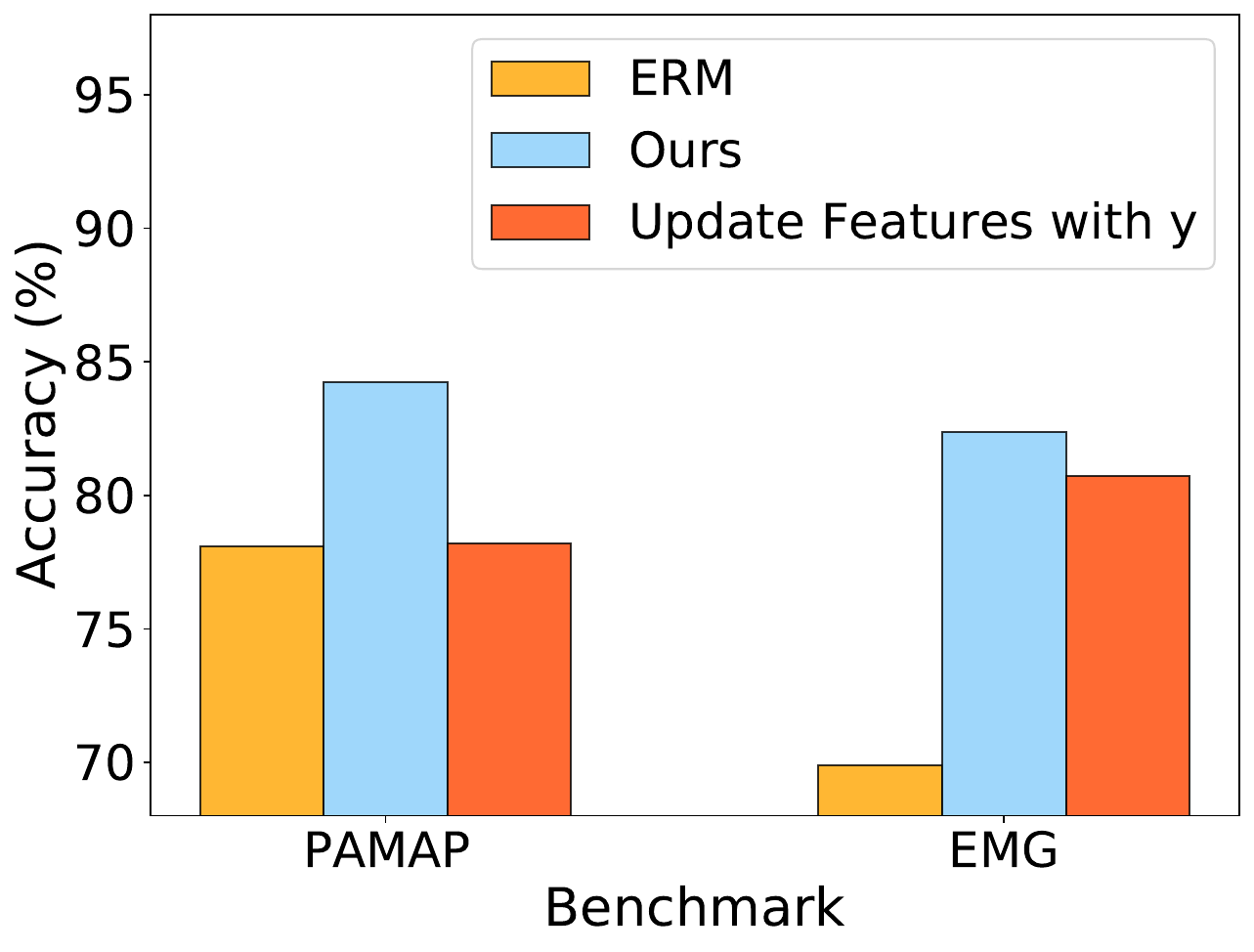}
        \label{fig-abla-y}
    }
    \subfigure[\#Domain $K$ (EMG)]{
        \includegraphics[width=0.23\textwidth]{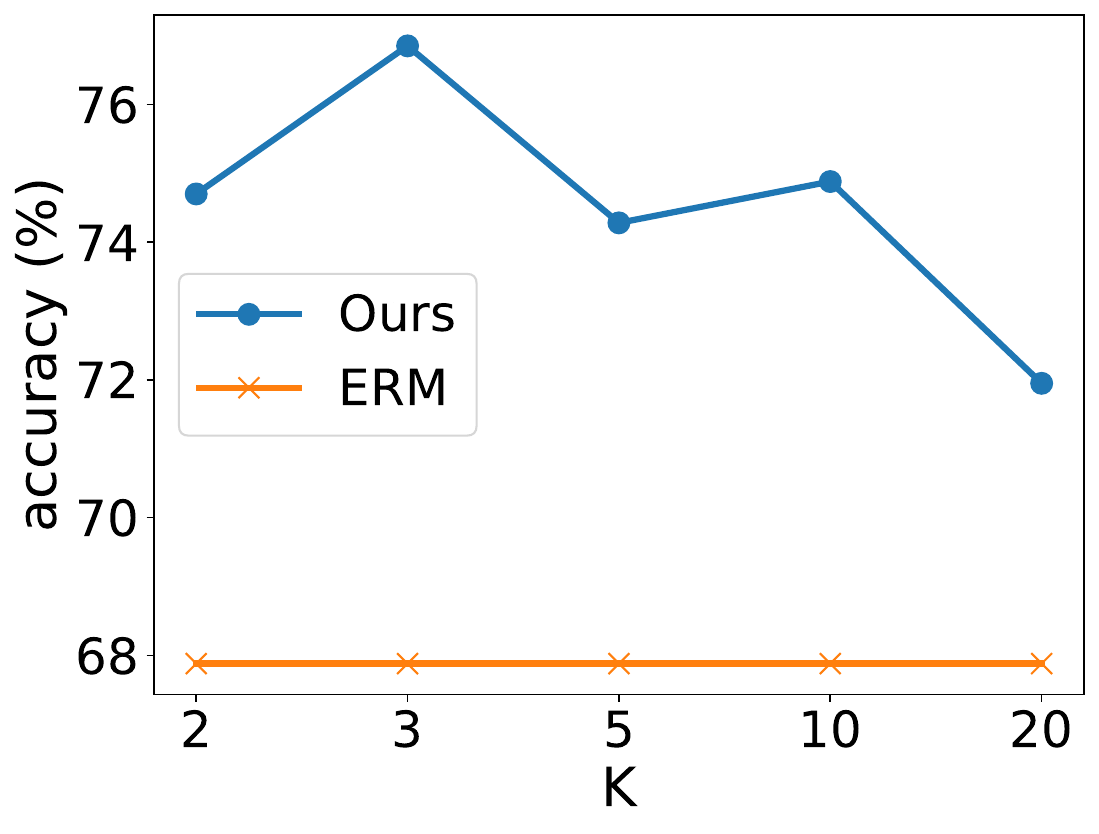}
        \label{fig-ablat-k}
    }
    \caption{Ablation study of \method for generalization.}
    \label{fig-ablat}
\end{figure*}

\section{Analysis}
\label{sec-analay-exp}
\subsection{Ablation study}
We present ablation study to answer the following three questions.
(1) \emph{Why obtaining pseudo domain labels with class-invariant features in step 3?}
If we obtain pseudo domain labels with common features, domain labels may have correlations with class labels, which may introduce contradictions when learning domain-invariant representations and lead to common performance. 
This is certified by the results in \figurename~\ref{fig1-abla-adv}.
(2) \emph{Why using fine-grained domain-class labels in step 2?}
If we utilize pseudo domain labels to update the feature net, it may make the representations seriously biased towards domain-related features and thereby leads to terrible performance on classification, which is proved in \figurename~\ref{fig1-abla-d}.
If we only utilize class labels to update the feature net, it may make representations biased to class-related features, thus \method is unable to obtain true latent sub-domains, as shown in \figurename~\ref{fig1-abla-y}.
Hence, we should employ fine-grained domain-class labels to obtain representations with both domain and class information.
(3) \emph{The more latent domains, the better?}
More latent domains may not bring better results (\figurename~\ref{fig1-ablat-k}) since a dataset may only have a few latent domains and introducing more may contradict its intrinsic data property.
Plus, more latent domains also make it harder to obtain pseudo domain labels and learn domain-invariant features.
For generalization, we can obtain similar observations to OOD detection from \figurename~\ref{fig-ablat}.

\begin{figure*}[htbp]
    \centering
    \subfigure[\#Latent domains $K$]{
        \includegraphics[width=0.23\textwidth]{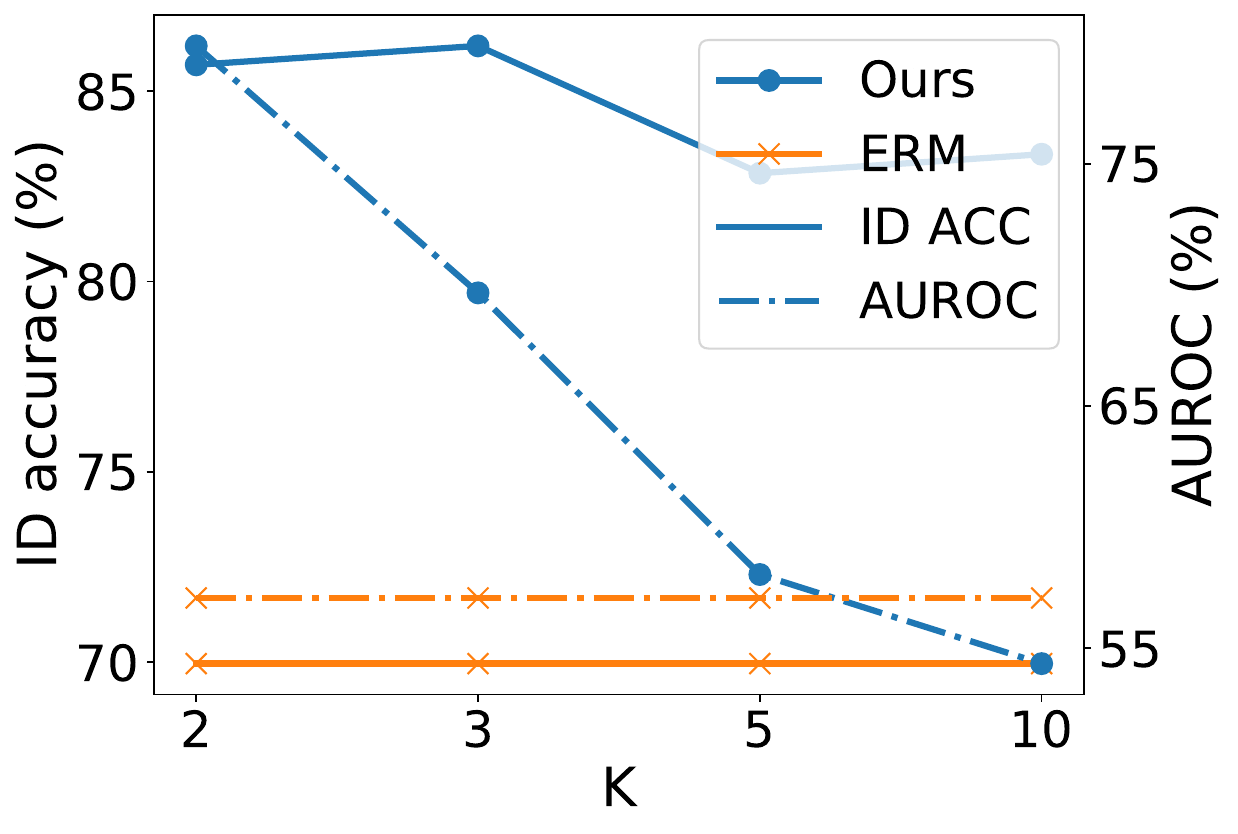}
        \label{fig1-sens-k}
    }
    \subfigure[$\lambda_1$]{
        \includegraphics[width=0.23\textwidth]{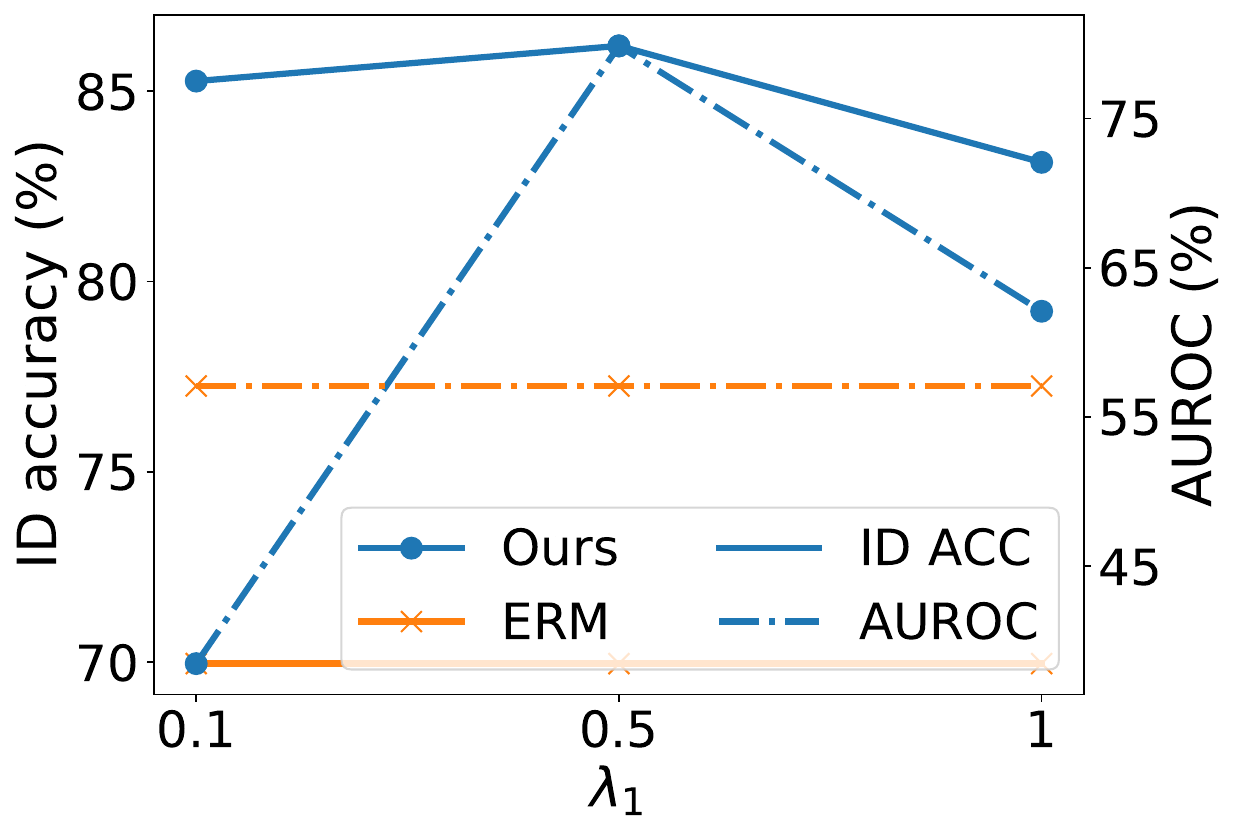}
        \label{fig1-sens-l1}
    }  
    \subfigure[$\lambda_2$]{
        \includegraphics[width=0.23\textwidth]{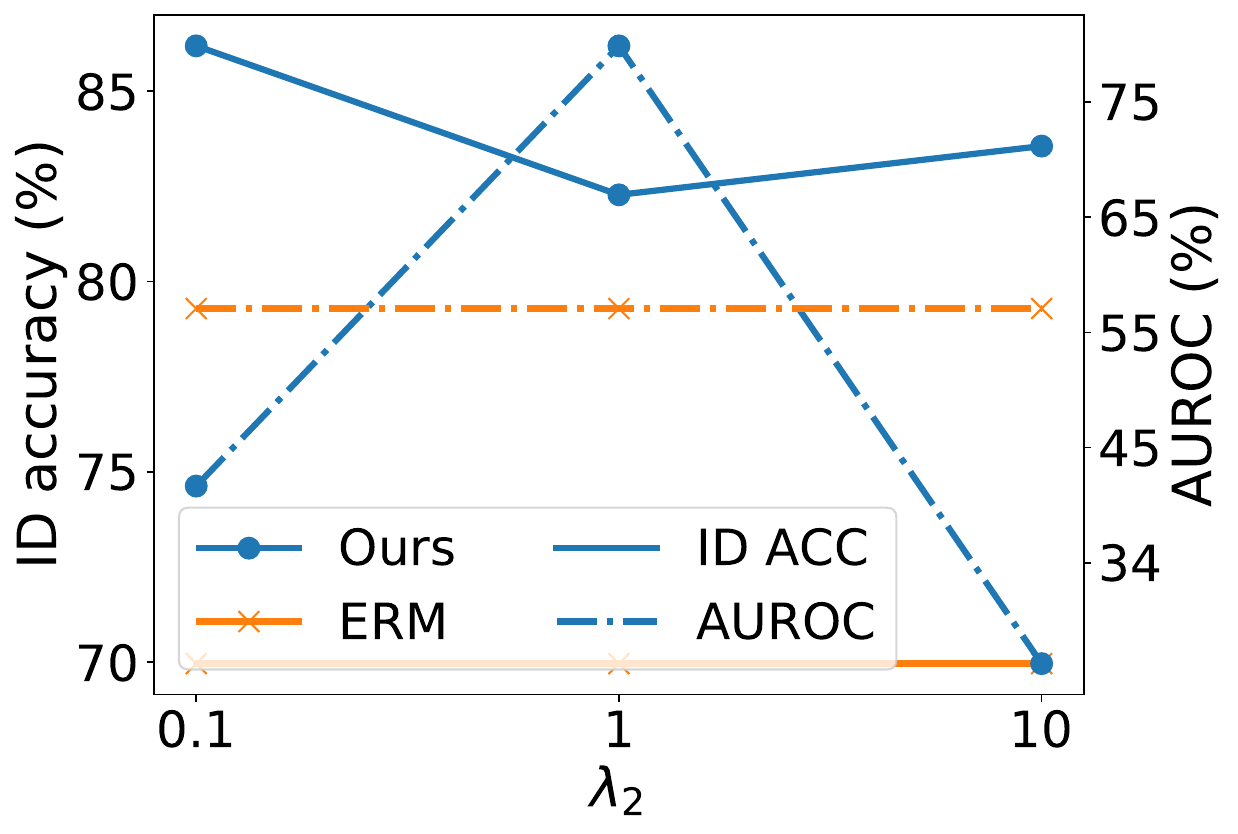}
        \label{fig1-sens-l2}
    }
    \subfigure[Local epoch and Round]{
        \includegraphics[width=0.23\textwidth]{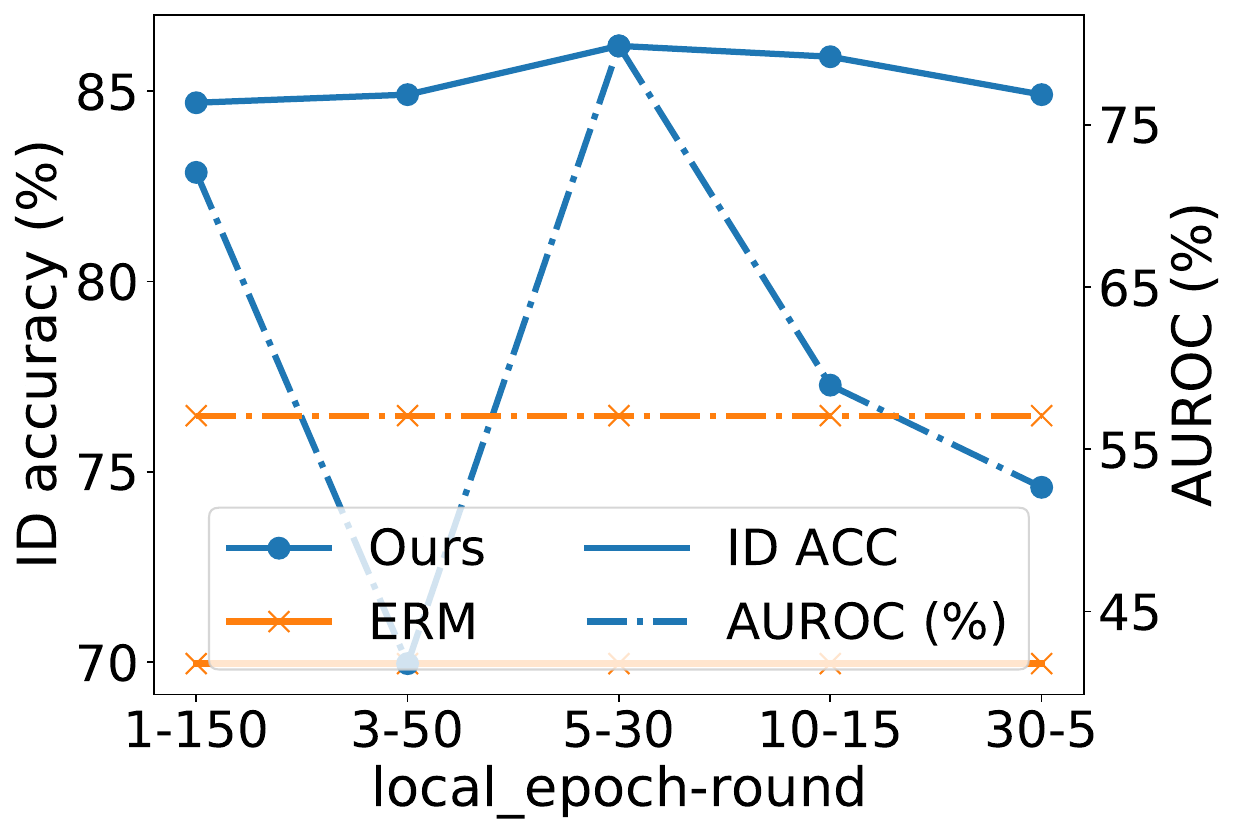}
        \label{fig1-sens-epoch}
    }
    \caption{Parameter sensitivity analysis (EMG) for detection.}
    \label{fig1-sens}
\end{figure*}

\begin{figure*}[htbp]
    \centering
    \subfigure[\#Latent domains $K$]{
        \includegraphics[width=0.22\textwidth]{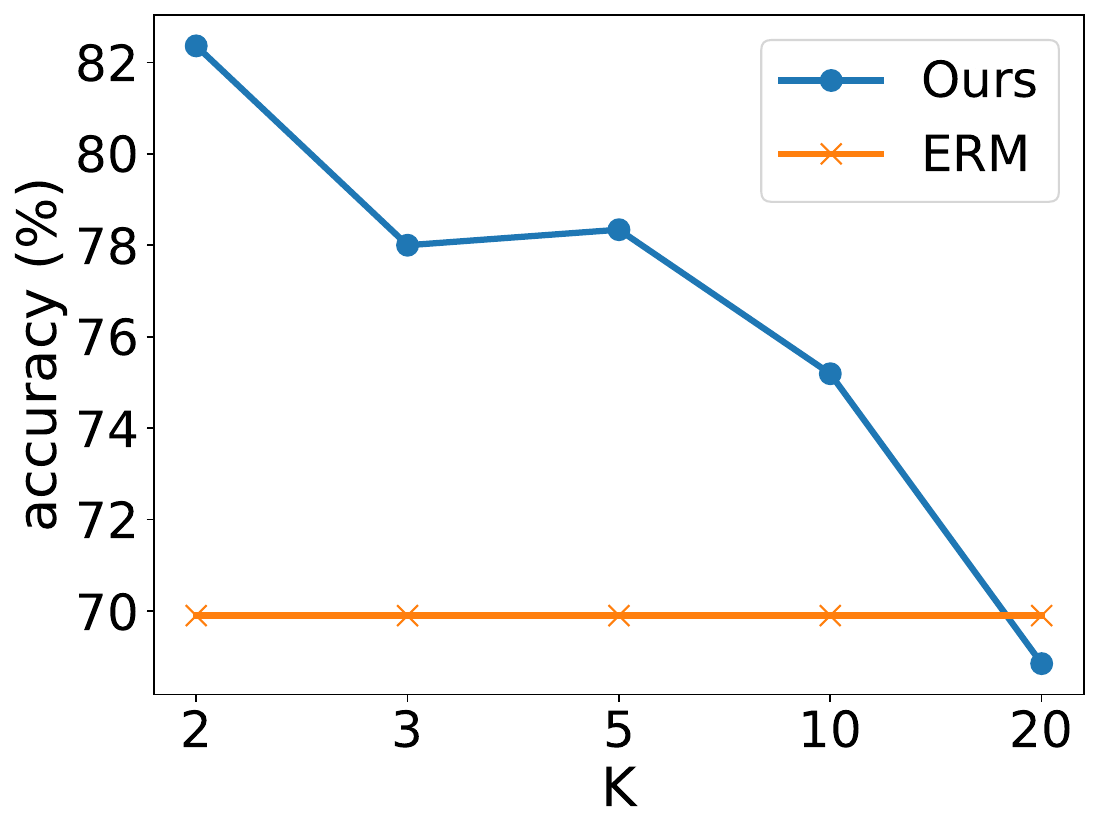}
        \label{fig-sens-k}
    }
    \subfigure[$\lambda_1$]{
        \includegraphics[width=0.22\textwidth]{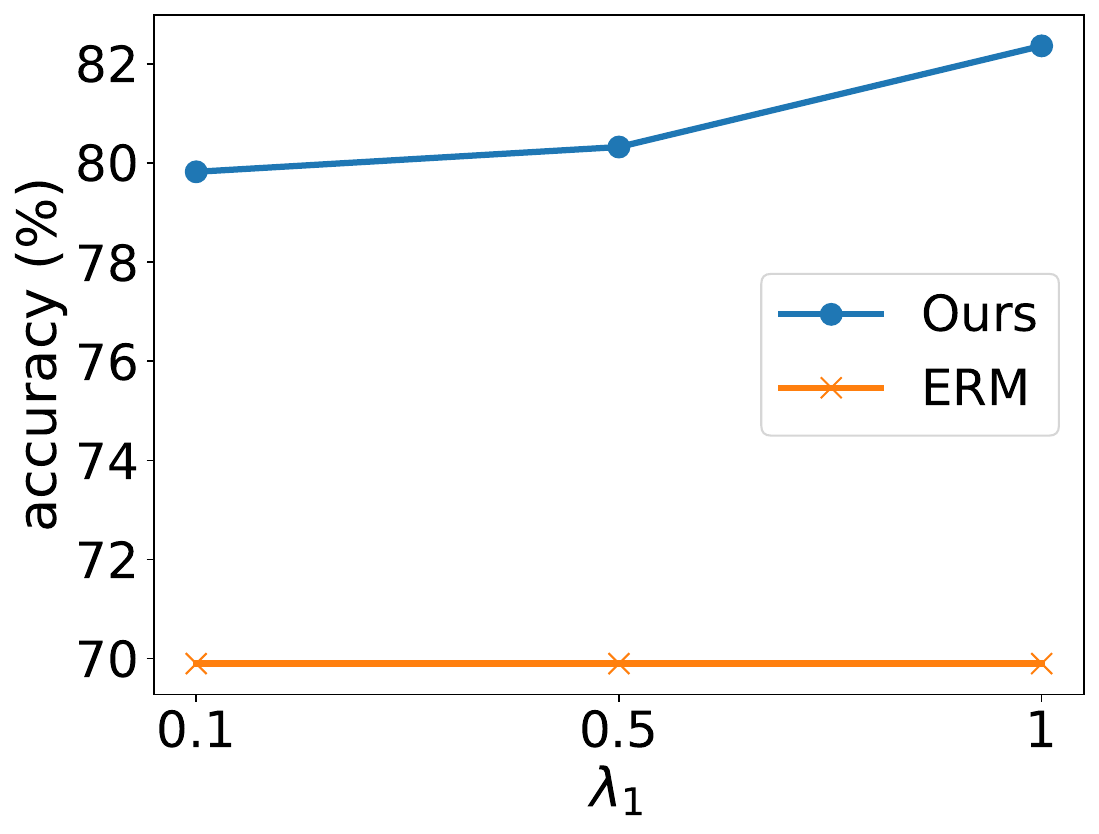}
        \label{fig-sens-l1}
    }  
    \subfigure[$\lambda_2$]{
        \includegraphics[width=0.22\textwidth]{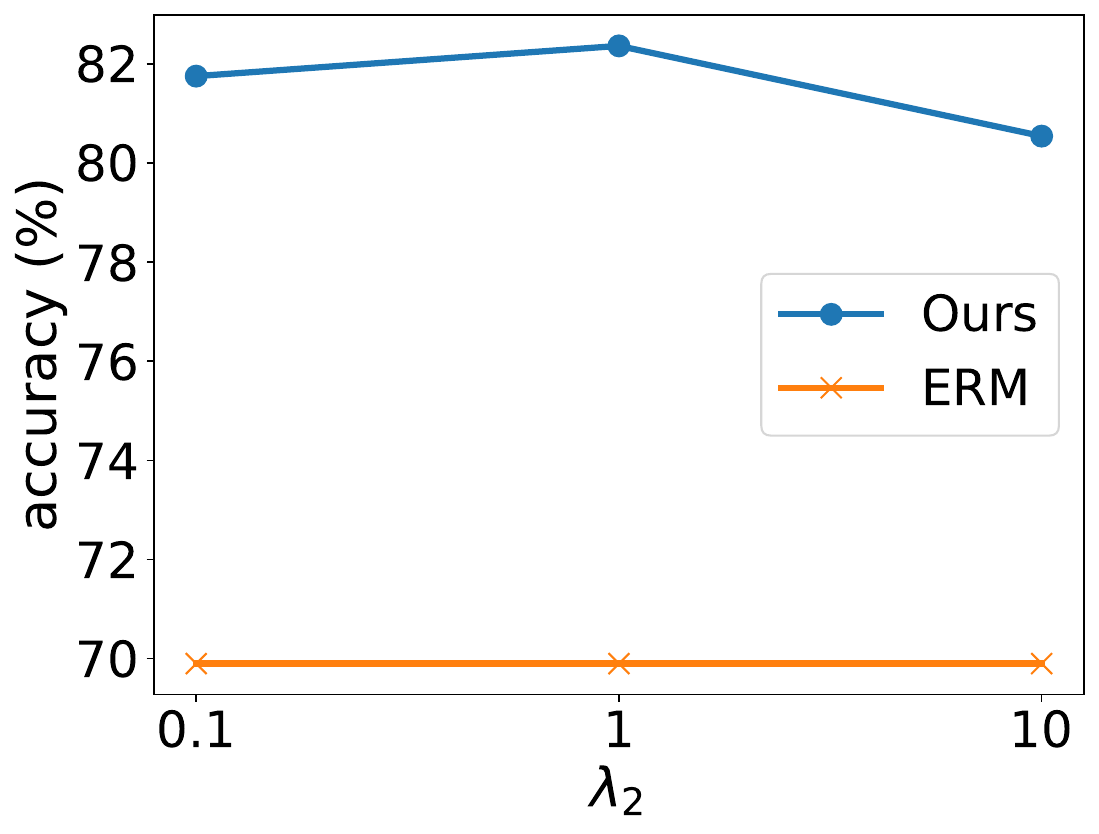}
        \label{fig-sens-l2}
    }
    \subfigure[Local epoch and Round]{
        \includegraphics[width=0.22\textwidth]{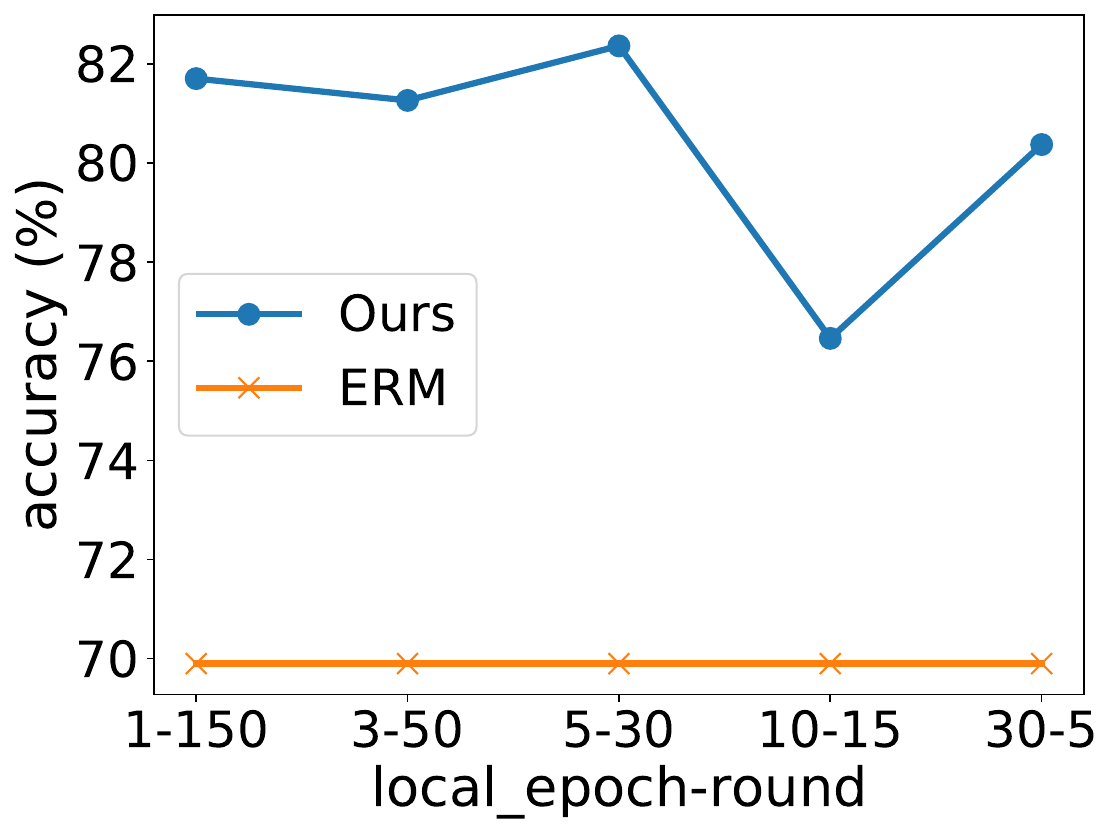}
        \label{fig-sens-epoch}
    }
    \caption{Parameter sensitivity analysis (EMG) for generalization.}
    \label{fig-sens}
\end{figure*}

\begin{figure*}[htbp]
    \centering
    \subfigure[ERM]{
        \includegraphics[width=0.14\textwidth]{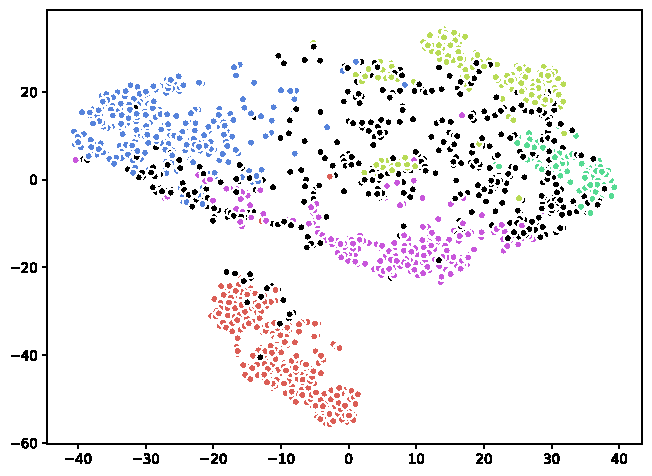}
        \label{fig1-erm-acc}
    }  
    \subfigure[ANDMask]{
        \includegraphics[width=0.14\textwidth]{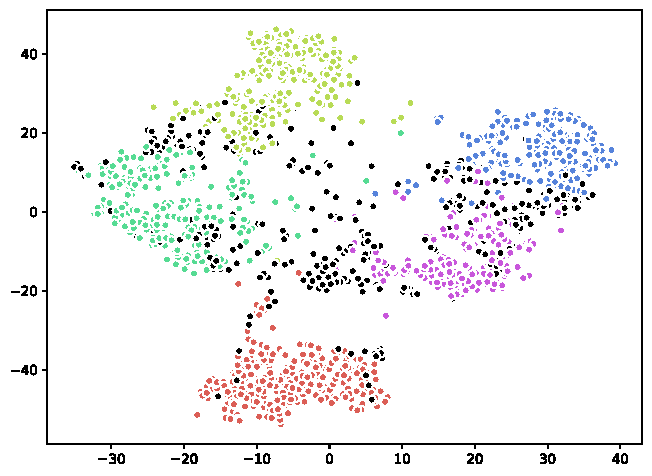}
        \label{fig1-andmask-acc}
    }
    \subfigure[Ours]{
        \includegraphics[width=0.14\textwidth]{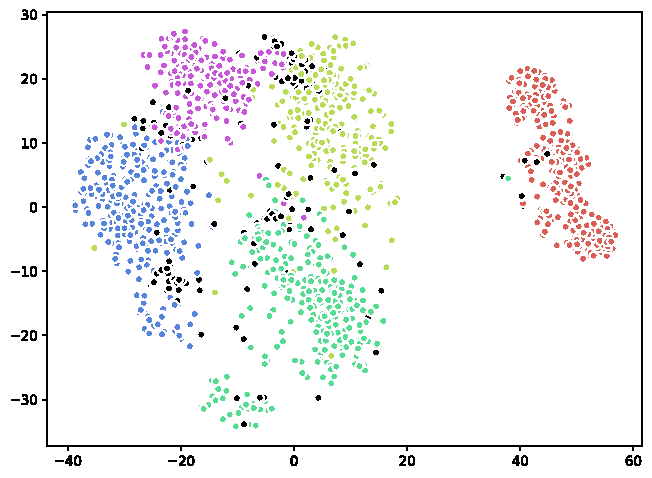}
        \label{fig1-ours-acc}
    }
    \rulesep
    \subfigure[ERM]{
        \includegraphics[width=0.14\textwidth]{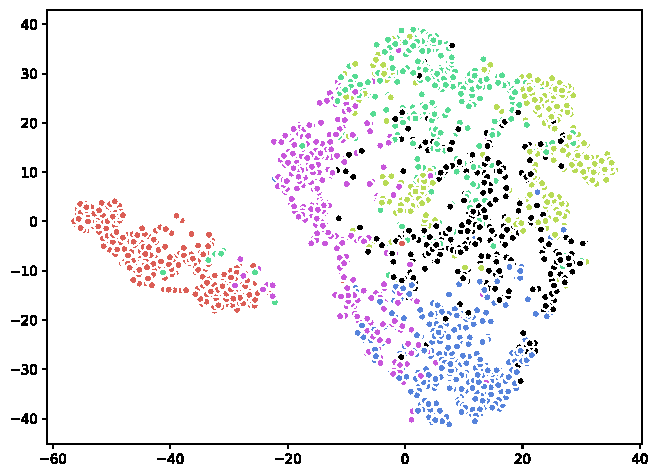}
        \label{fig1-erm-ruc}
    }  
    \subfigure[ANDMask]{
        \includegraphics[width=0.14\textwidth]{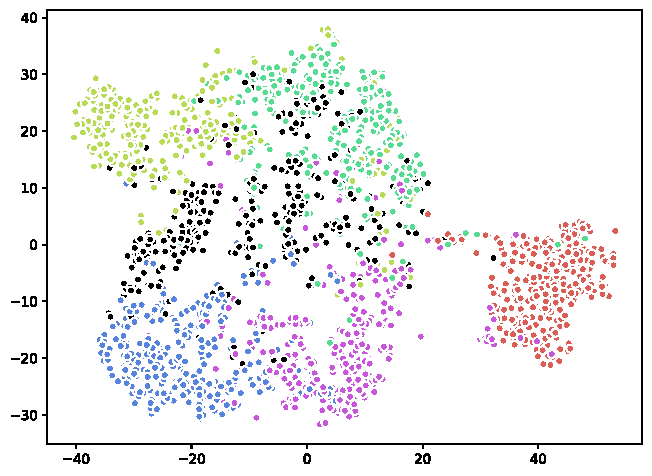}
        \label{fig1-andmask-ruc}
    }
    \subfigure[Ours]{
        \includegraphics[width=0.14\textwidth]{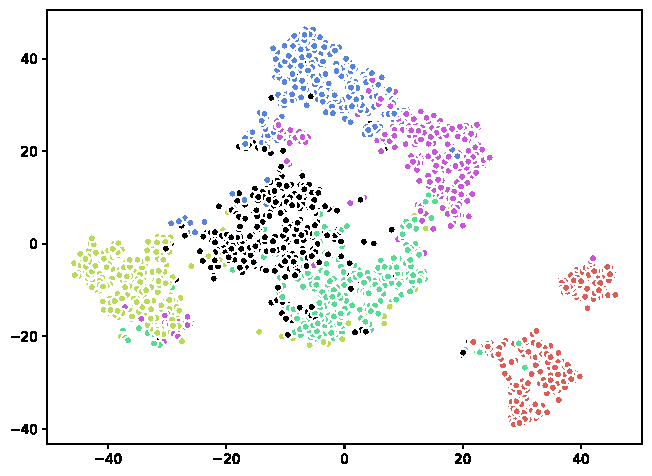}
        \label{fig1-ours-ruc}
    }
    \caption{t-SNE visualization (EMG) for detection. For (a)-(c) black points are misclassified samples. For (d)-(e), black points mean OOD samples.}
    \label{fig1-vis}
\end{figure*}

\begin{figure*}[htbp]
    \centering
    \vspace{-.1in}
    \subfigure[Initial domain split]{
        \includegraphics[width=0.22\textwidth]{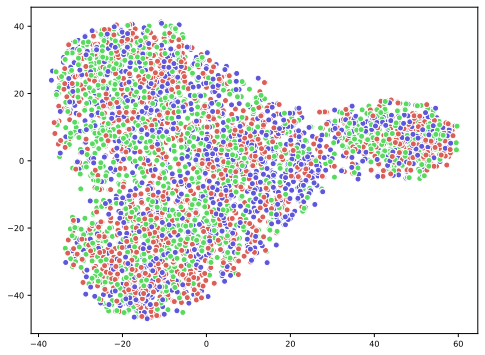}
        \label{fig-vis-emg-i}
    }
    \subfigure[Our domain split]{
        \includegraphics[width=0.22\textwidth]{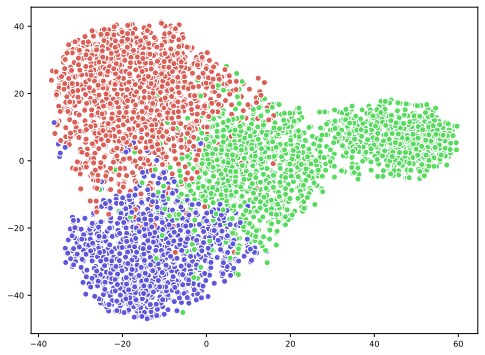}
        \label{fig-vis-emg-o}
    }
    \rulesep
    \subfigure[ANDMask]{
        \includegraphics[width=0.22\textwidth]{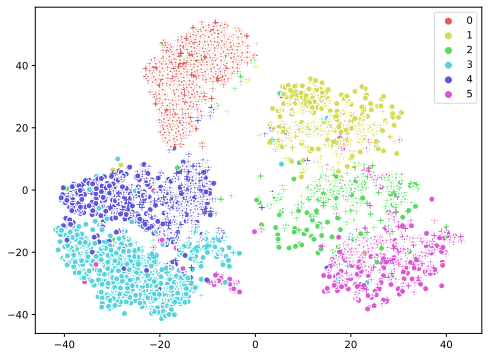}
        \label{fig-vis-andmask}
    }
    \subfigure[Our classification]{
        \includegraphics[width=0.22\textwidth]{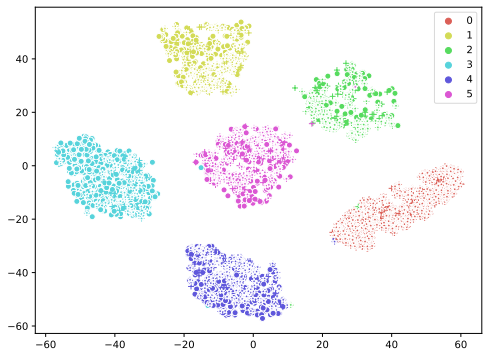}
        \label{fig-vis-ours}
    }
    \caption{t-SNE visualizations for domain splits ((a) (b)) and classification ((c) (d)) on EMG data.}
    \label{fig-visual}
\end{figure*}

\subsection{Parameter sensitivity}
There are mainly four hyperparameters in our method: $K$ which is the number of latent sub-domains, $\lambda_1$ for the adversarial part in step 3, $\lambda_2$ for the adversarial part in step 4, and local epochs and total rounds. 
For fairness, the product of local epochs and total rounds is the same value. We evaluate the parameter sensitivity of our method for detection in \figurename~\ref{fig1-sens} where we change one parameter and fix the other to record the results. From these results, we can obtain the following observations.
1) For ID accuracy, our method achieves better performance in a wide range, demonstrating that our method is robust.
2) For AUROC, our method achieves better performance for a wide range in most situations, demonstrating that our method is robust.
Sometimes, we need to tune hyperparameters for better AUROC carefully.
We also evaluate the parameter sensitivity of \method on OOD generalization and the results are shown in \figurename~\ref{fig-sens}.
From these results, we can see that our method achieves better performance in a wide range, demonstrating that our method is robust.

\subsection{Visualization study}
We present some visualizations to show the rationales of \method.
For detection, \figurename~\ref{fig1-erm-acc} and \ref{fig1-ours-acc} show that \method can learn better margins and generate few misclassifications while \figurename~\ref{fig1-erm-ruc} and \ref{fig1-ours-ruc} show that \method can compact and discriminate OOD samples.
For generalization, data points with different initial domain labels are mixed together in \figurename~\ref{fig-vis-emg-i} while \method can characterize different latent distributions and separate them well in \figurename~\ref{fig-vis-emg-o}.
\figurename~\ref{fig-vis-ours} and \ref{fig-vis-andmask} show that \method can learn better domain-invariant representations compared to the latest method ANDMask.
To sum up, \method can find better representations to enhance generalization.

\begin{figure*}[t!]
    \centering
    \subfigure[Temporal shift]{
\includegraphics[height=0.16\textwidth]{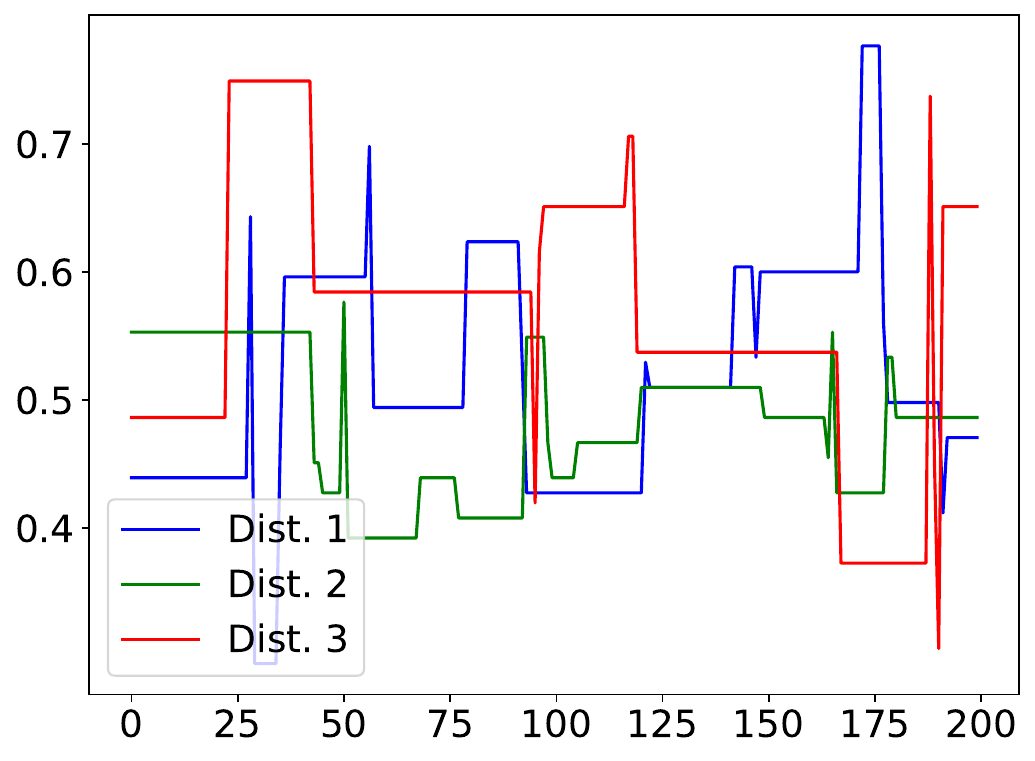}
        \label{fig1-s-one}
    }
    \subfigure[Spatial shift]{
\includegraphics[height=0.16\textwidth]{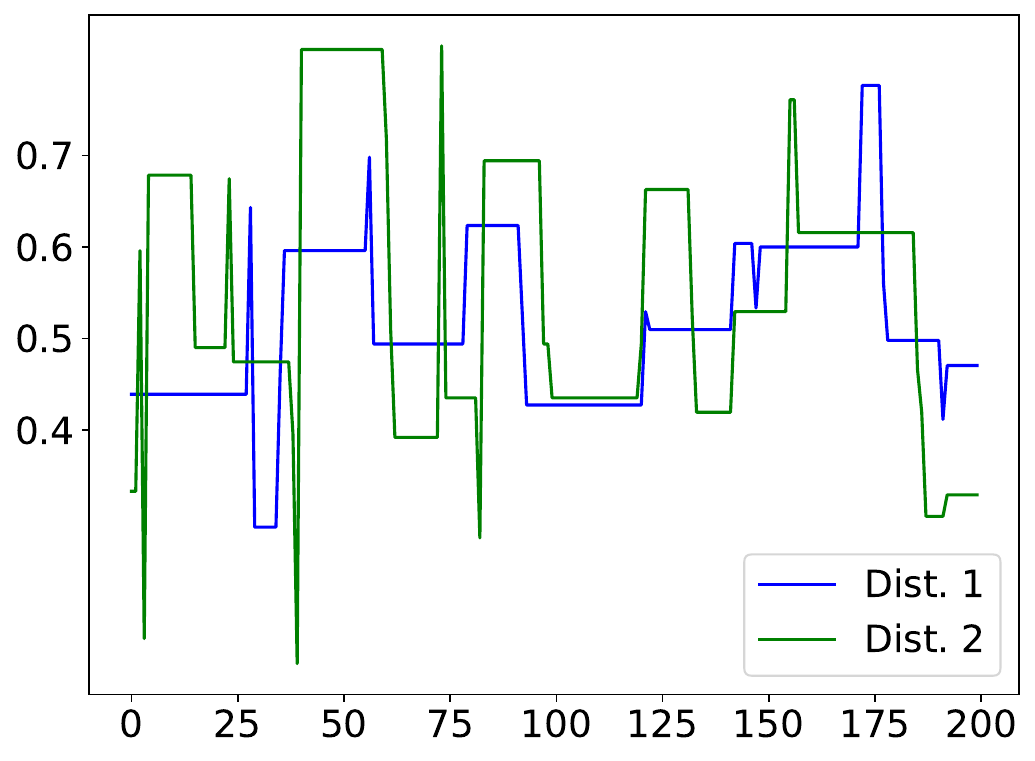}
        \label{fig1-s-mult}
    }  
    \rulesep
    \subfigure[Temporal shift]{
        \includegraphics[height=0.16\textwidth]{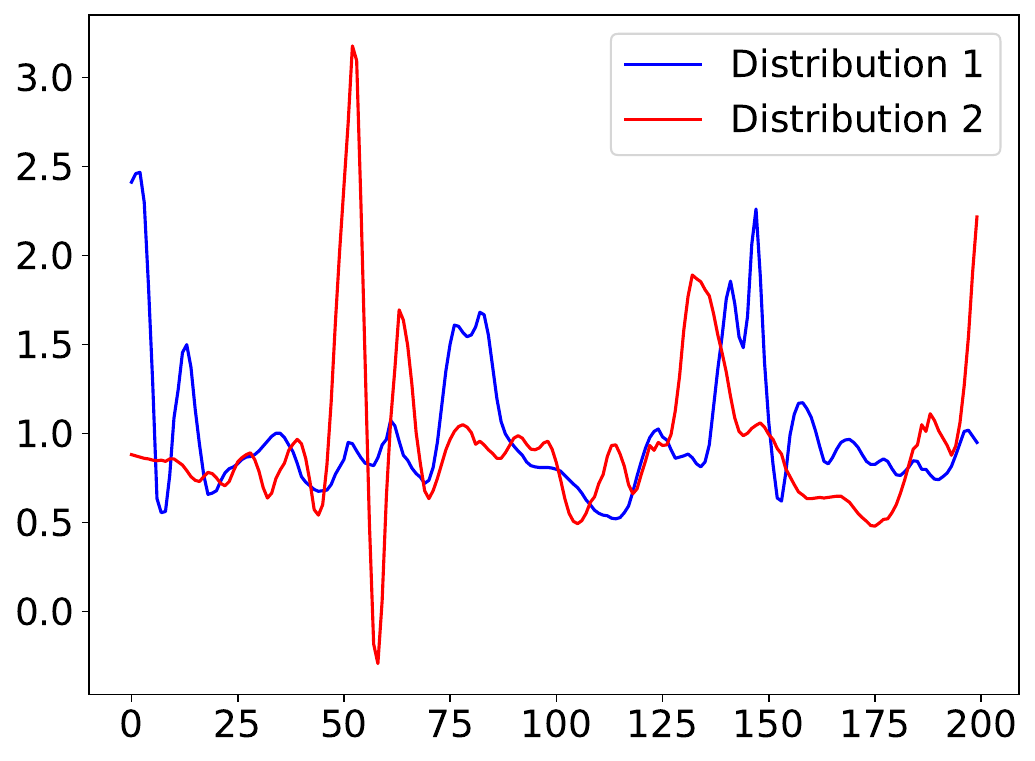}
        \label{fig-s-usc}
    }
    \subfigure[Spatial shift]{
        \includegraphics[height=0.16\textwidth]{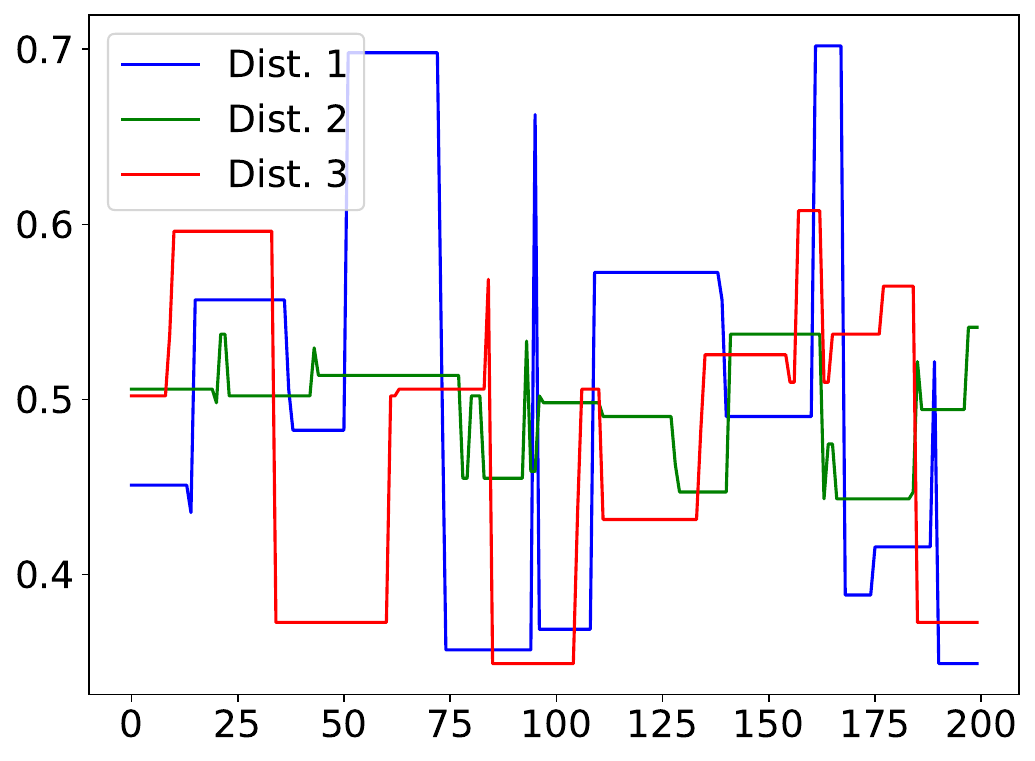}
        \label{fig-s-emg}
    }
    \caption{Latent distributions obtained by our method on two datasets. X-axis is data numbers while Y-axis is its values. (a)(b) are for detection while (c)(d) are for generalization.}
    \label{fig-split-sampe}
    \vspace{-.1in}
\end{figure*}

\subsection{Existence of latent subdomains}
What exactly can our \method learn?
In \figurename~\ref{fig1-s-one}, for the first subject in EMG, there is more than one latent distribution for class wrist extension, showing the existence of temporal distribution shifts: the distribution of the same activity could change.
For spatial distribution shift, \figurename~\ref{fig1-s-mult} on EMG dataset shows that our algorithm found two latent distributions from the EMG data of the first person and the third person.
These results indicate the existence of latent distributions with both temporal and spatial distribution shifts.
For generalization, there exist similar phenomena.

\begin{figure}[t!]
    \centering
    \subfigure[Initial splits]{
\includegraphics[height=0.18\textwidth]{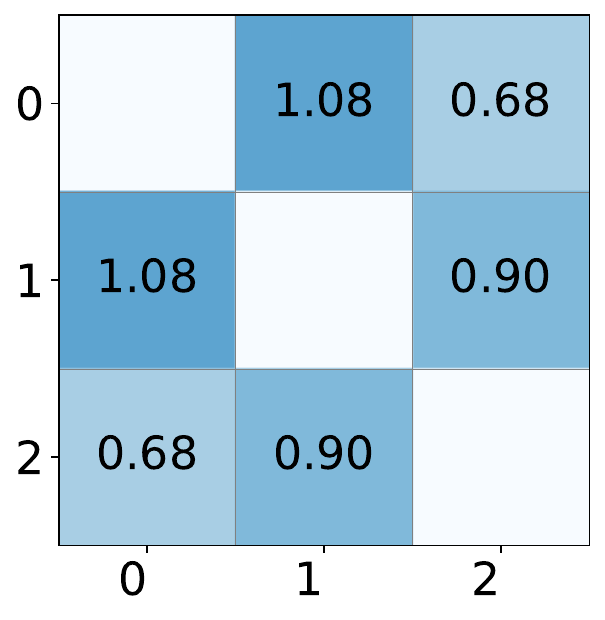}
        \label{fig-p-h}
    }
    \subfigure[Our splits]{
\includegraphics[height=0.18\textwidth]{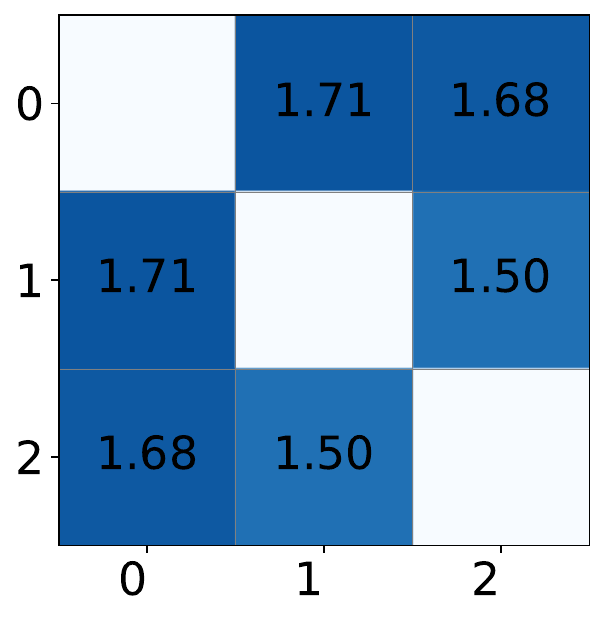}
        \label{fig-p-o}
    }  
    \caption{$\mathcal{H}$-divergence among domains with initial splits and our splits on PAMAP2. Axes are domain numbers.}
    \label{fig-split-div}
    \vspace{-.1in}
\end{figure}

\subsection{Quantitative analysis for `worst-case' distributions}
We present quantitative analysis by computing the $\mathcal{H}$-divergence~\cite{ben2010theory} to show the effectiveness of our `worst-case distribution'.
In \figurename~\ref{fig-p-h} and \ref{fig-p-o}, compared to initial domain splits, latent sub-domains generated by our method have larger $\mathcal{H}$-divergence among each other.
According to \propositionname~\ref{prop:hdist}, larger $\mathcal{H}$-divergence among domains brings better generalization.
This again shows the efficacy of \method in computing the 'worst-case' distribution scenario.

\begin{figure*}[t!]
    \centering
    \subfigure[Extensibility]{
        \includegraphics[height=0.14\textwidth]{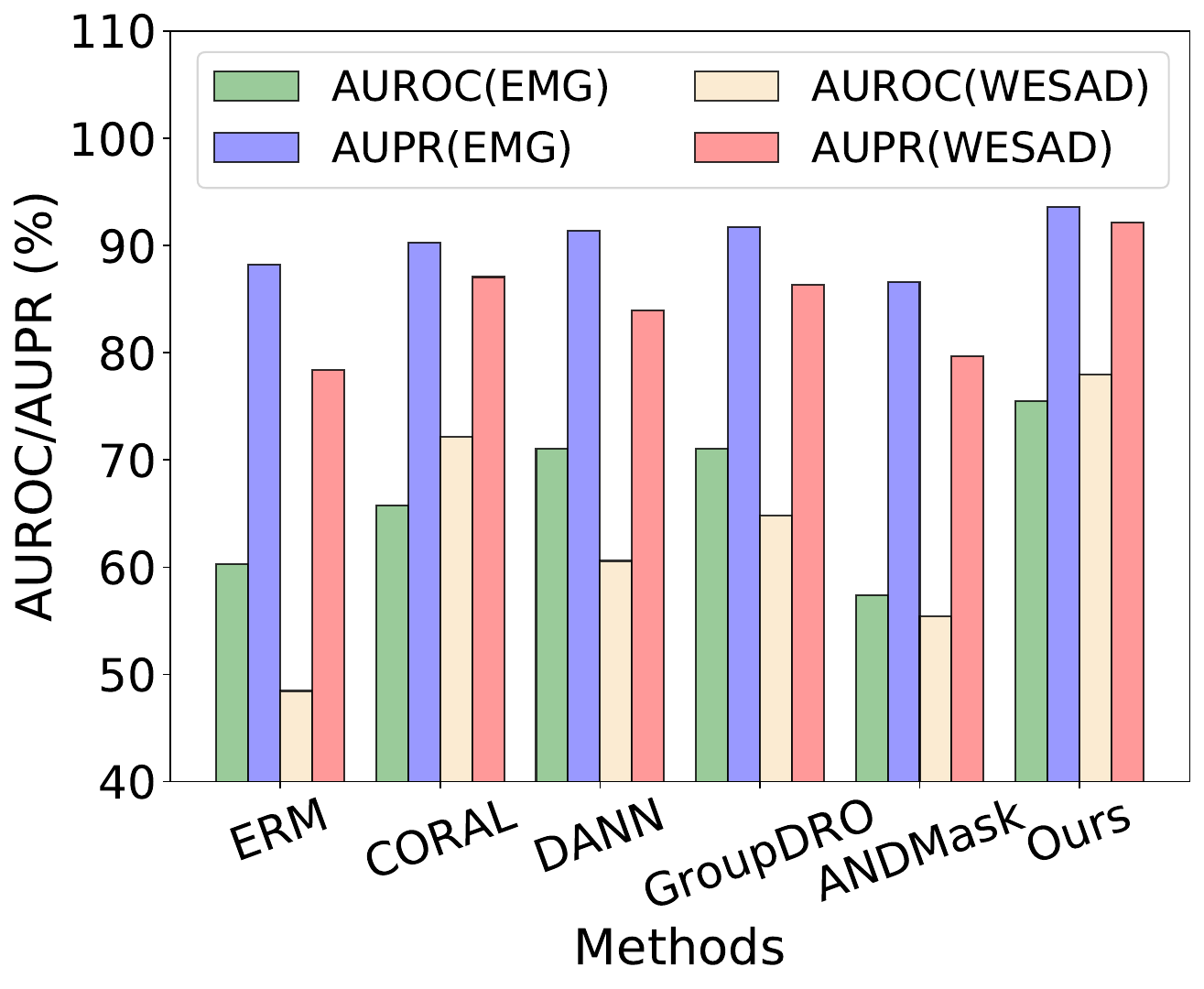}
        \label{fig:odin-new}
    }
    \rulesep
    \subfigure[Different CNNs]{
        \includegraphics[height=0.14\textwidth]{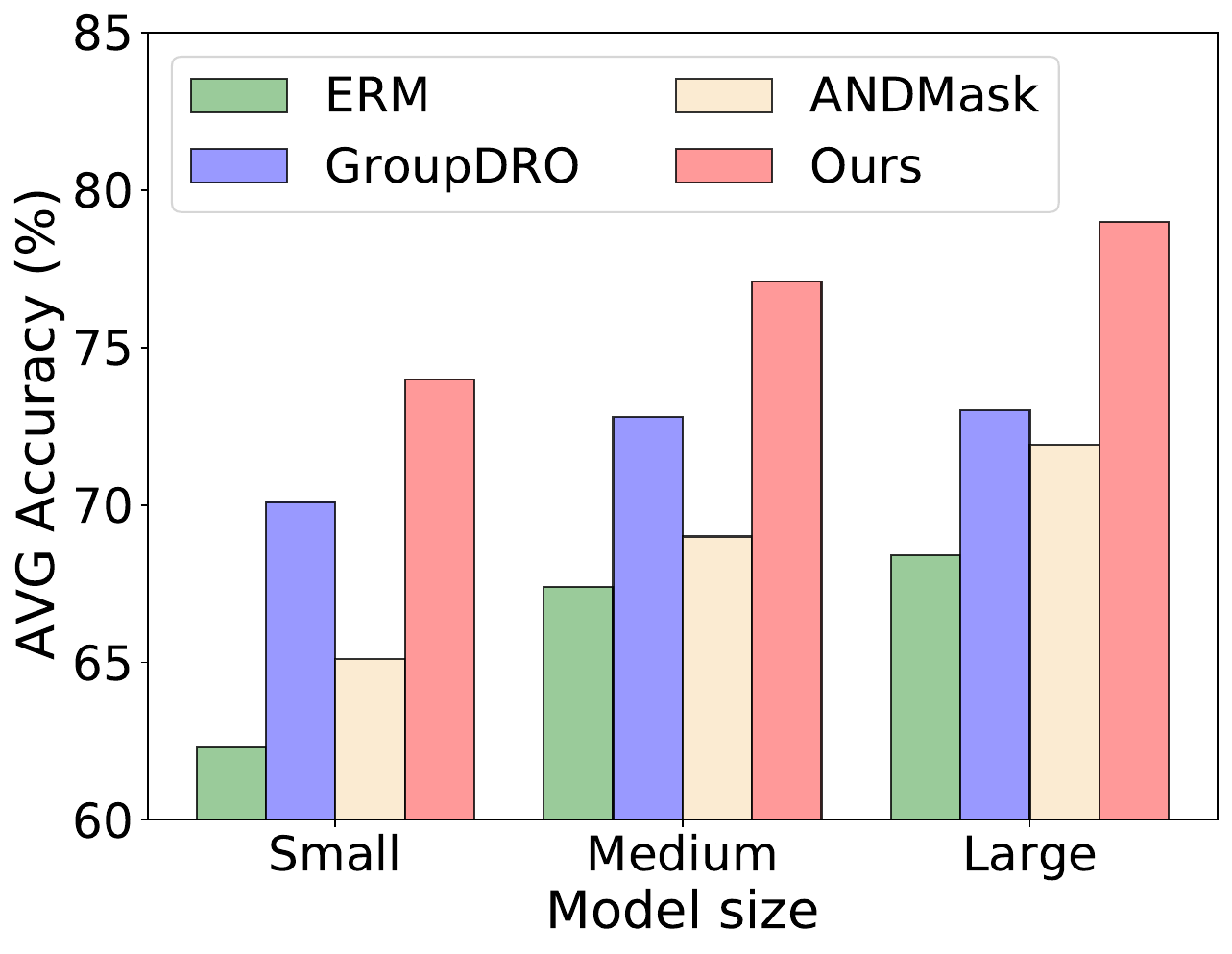}
        \label{fig:msize}
    }
    \subfigure[Transformers]{
        \includegraphics[height=0.14\textwidth]{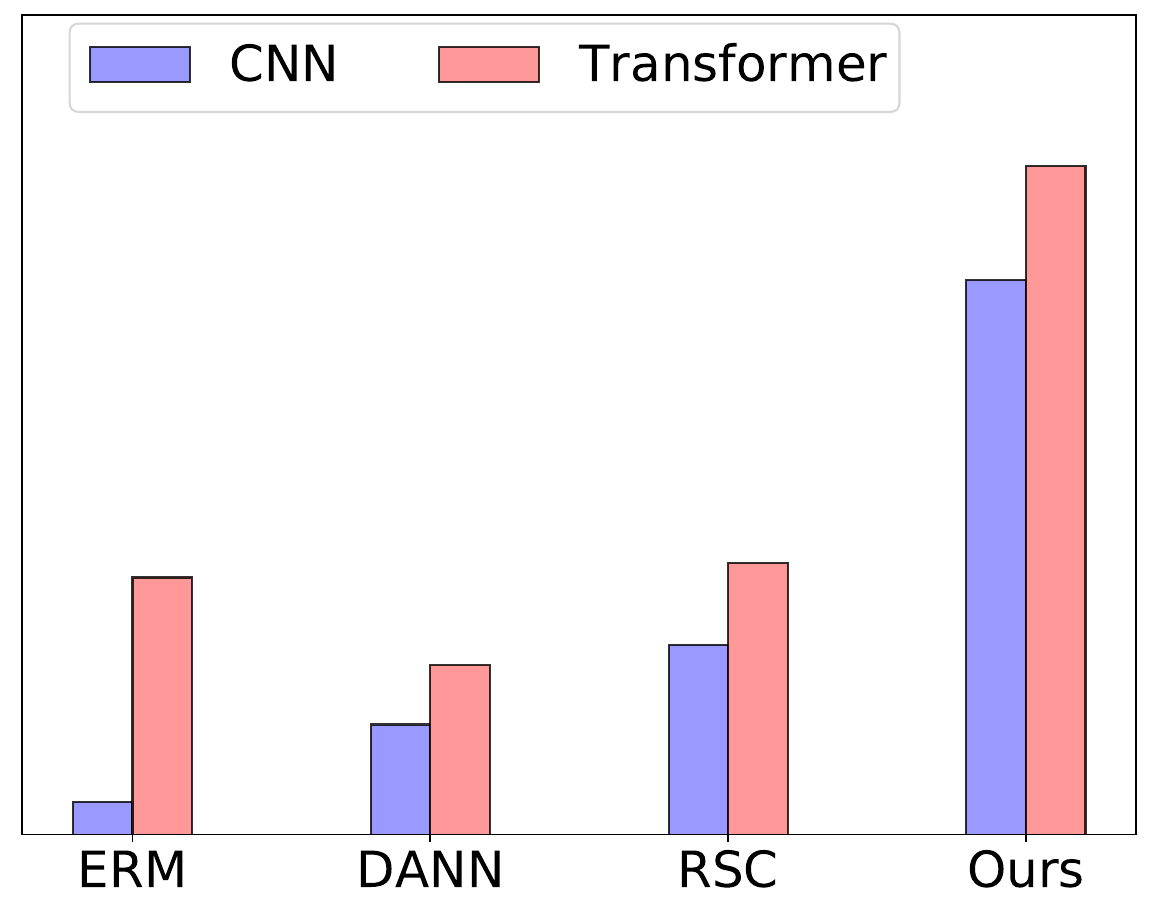}
        \label{fig-transform-xdata}
    }
    \rulesep
    \subfigure[Time complexity]{
    \includegraphics[height=0.14\textwidth]{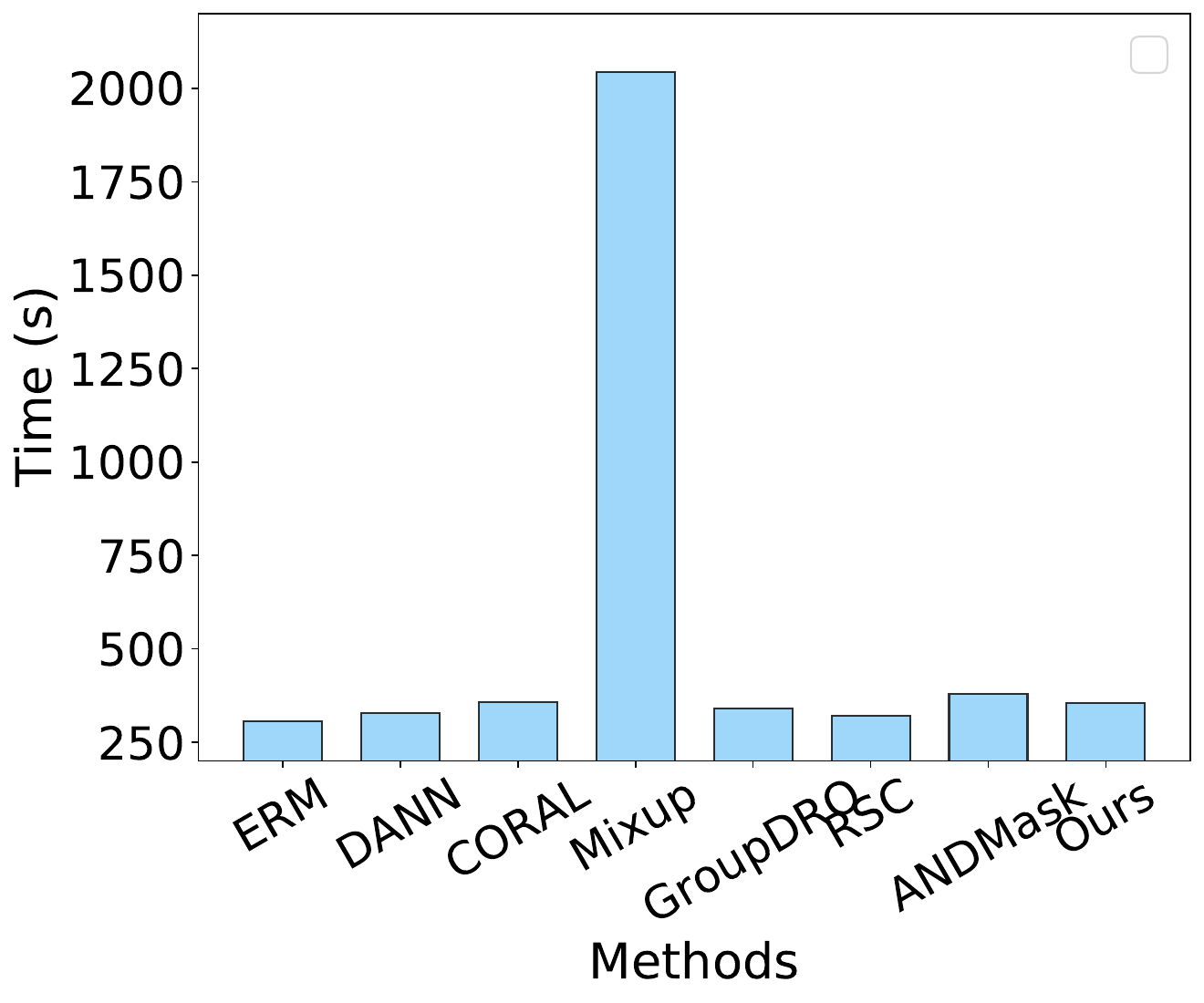}
    \label{fig-time}
    }
    \subfigure[Convergence]{
    \includegraphics[height=0.14\textwidth]{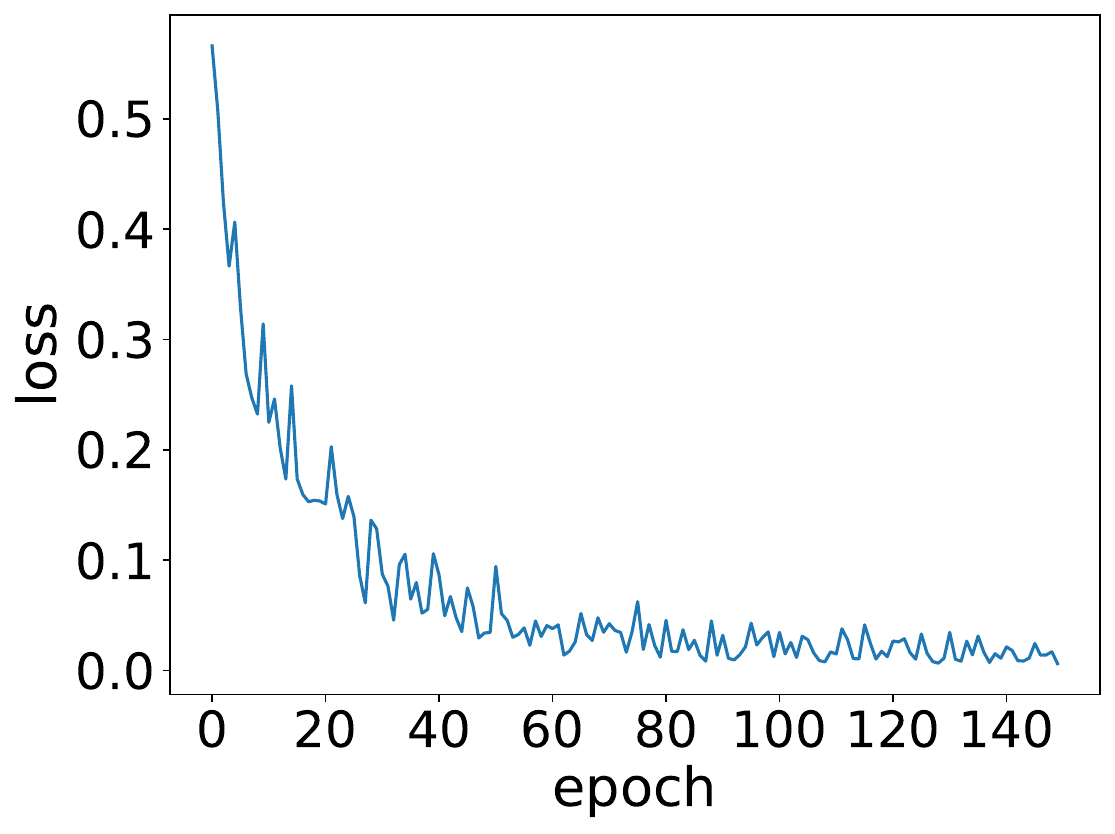}
    \label{fig-conver}
    }
    \label{fig-time-conv}
    \caption{Experimental analysis. (a) Extensibility on EMG and WESAD. (b)-(c) Results with different backbones on EMG and X-dataset. (d)-(e) Time complexity and convergence.  }
    \label{fig-analy-rest}
\end{figure*}

\subsection{Extensibility}
To demonstrate that our method is extensible, we also provide implementations with ODIN~\cite{liang2018enhancing}.
The OOD detection results are shown in \figurename~\ref{fig:odin-new}. 
We have the following observations.
1) Our method still achieves the best average AUROC and AUPR on both EMG and WESAD, which demonstrates the superiority of \method.  
2) Implementations with ODIN perform similarly to implementations with MCP, since both ODIN and MCP are dependent on models' predictions.
For some datasets, e.g. EMG, implementations with ODIN bring improvements compared to MCP but for some datasets, e.g. WESAD, implementations with ODIN perform worse.
Therefore, for specific applications, we need to select the best detection techniques.

\subsection{Varying backbones} 
For generalization, we attempt to demonstrate that \method is robust to varying backbones.
\figurename~\ref{fig:msize} shows the results using small, medium, and large backbones, respectively (we implement them with different numbers of layers.).
Results indicate that larger models tend to achieve better OOD generalization performance.
Our method outperforms others in all backbones, showing that \method presents consistently strong OOD performance in different architectures. 
We also try Transformer~\cite{vaswani2017attention} as the backbone for comparisons.
As shown in \cite{zhang2022delving}, Transformer often has a better generalization ability compared to CNN, which implies improving with Transformer is more difficult.
From \figurename~\ref{fig-transform-xdata}, we can see that each method with Transformer has a remarkable improvement on the first task of cross-dataset.
Compared to ERM, DANN and RSC have no improvements but ours still has further improvements and achieves the best performance.
DANN even performs worse than ERM, which demonstrates the importance of more accurate sub-domain labels.
Overall, for all architectures, our method achieves the best performance.

\subsection{Time Complexity and Convergence Analysis}
\label{ssec-app-timecon}
We also provide some analysis on time complexity and convergence.
Since we only optimize the feature extractor in Step 2, our method does not cost too much time. And the results in \figurename~\ref{fig-time} prove this argument empirically.
The convergence results are shown in \figurename~\ref{fig-conver}.
Our method is convergent.
Although there are some little fluctuations, these fluctuations exist widely in all domain generalization methods due to different distributions of different samples.

\section{Limitation and Discussion}
\label{sec-limit}

\method could be more perfect by pursuing the following avenues.
1) Estimate the number of latent distributions $K$ automatically: we currently treat it as a hyperparameter.
2) Seek the semantics behind latent distributions: can adding more human knowledge obtain better latent distributions?
3) Extend \method beyond detection and generalization but for forecasting problems.

Moreover, we argue that dynamic distributions not only exist in time series but also in general machine learning data such as images and text~\cite{deecke2022visual,xu2023domain}.
Thus, it is of great interest to apply our approach to these domains to further improve their performance.

\section{Conclusion}
\label{sec-conl}
We proposed \method, a universe framework, to learn generalized representation for time series detection and generalization.
\method employs an adversarial game that maximizes the `worst-case' distribution scenario while minimizing their distribution divergence. 
We provide \methodmal and \methodmcp via representations and predictions respectively for detection while we directly utilize \method for generalization.
We demonstrated its effectiveness in different applications.
We are surprised that one dataset can contain several latent distributions.
Characterizing such latent distributions will greatly improve the generalization performance on unseen datasets.

\bibliographystyle{IEEEtran}
\bibliography{IEEEabrv,refs}

\begin{thebibliography}{10}
\providecommand{\url}[1]{#1}
\csname url@samestyle\endcsname
\providecommand{\newblock}{\relax}
\providecommand{\bibinfo}[2]{#2}
\providecommand{\BIBentrySTDinterwordspacing}{\spaceskip=0pt\relax}
\providecommand{\BIBentryALTinterwordstretchfactor}{4}
\providecommand{\BIBentryALTinterwordspacing}{\spaceskip=\fontdimen2\font plus
\BIBentryALTinterwordstretchfactor\fontdimen3\font minus
  \fontdimen4\font\relax}
\providecommand{\BIBforeignlanguage}[2]{{%
\expandafter\ifx\csname l@#1\endcsname\relax
\typeout{** WARNING: IEEEtran.bst: No hyphenation pattern has been}%
\typeout{** loaded for the language `#1'. Using the pattern for}%
\typeout{** the default language instead.}%
\else
\language=\csname l@#1\endcsname
\fi
#2}}
\providecommand{\BIBdecl}{\relax}
\BIBdecl

\bibitem{xiao2022dynamic}
Q.~Xiao, B.~Wu, Y.~Zhang, S.~Liu, M.~Pechenizkiy, E.~Mocanu, and D.~C. Mocanu,
  ``Dynamic sparse network for time series classification: Learning what to
  “see”,'' \emph{Advances in Neural Information Processing Systems},
  vol.~35, pp. 16\,849--16\,862, 2022.

\bibitem{tang2021omni}
W.~Tang, G.~Long, L.~Liu, T.~Zhou, M.~Blumenstein, and J.~Jiang, ``Omni-scale
  cnns: a simple and effective kernel size configuration for time series
  classification,'' in \emph{International Conference on Learning
  Representations}, 2022.

\bibitem{fulcher2014highly}
B.~D. Fulcher and N.~S. Jones, ``Highly comparative feature-based time-series
  classification,'' \emph{IEEE Transactions on Knowledge and Data Engineering},
  vol.~26, no.~12, pp. 3026--3037, 2014.

\bibitem{husken2003recurrent}
M.~H{\"u}sken and P.~Stagge, ``Recurrent neural networks for time series
  classification,'' \emph{Neurocomputing}, vol.~50, pp. 223--235, 2003.

\bibitem{li2019enhancing}
S.~Li, X.~Jin, Y.~Xuan, X.~Zhou, W.~Chen, Y.-X. Wang, and X.~Yan, ``Enhancing
  the locality and breaking the memory bottleneck of transformer on time series
  forecasting,'' \emph{Advances in Neural Information Processing Systems},
  vol.~32, pp. 5243--5253, 2019.

\bibitem{DBLP:conf/icml/DrouinMC22}
A.~Drouin, {\'{E}}.~Marcotte, and N.~Chapados, ``Tactis:
  Transformer-attentional copulas for time series,'' in \emph{International
  Conference on Machine Learning, {ICML}}, vol. 162, 2022, pp. 5447--5493.

\bibitem{wang2023attention}
Y.~Wang, C.~Qian, and S.~J. Qin, ``Attention-mechanism based dipls-lstm and its
  application in industrial process time series big data prediction,''
  \emph{Computers \& Chemical Engineering}, p. 108296, 2023.

\bibitem{du2021adarnn}
Y.~Du, J.~Wang, W.~Feng, S.~Pan, T.~Qin, R.~Xu, and C.~Wang, ``Adarnn: Adaptive
  learning and forecasting of time series,'' in \emph{Proceedings of the 30th
  ACM International Conference on Information \& Knowledge Management}, 2021,
  pp. 402--411.

\bibitem{di2023explainable}
F.~Di~Martino and F.~Delmastro, ``Explainable ai for clinical and remote health
  applications: a survey on tabular and time series data,'' \emph{Artificial
  Intelligence Review}, vol.~56, no.~6, pp. 5261--5315, 2023.

\bibitem{yang2021generalized}
J.~Yang, K.~Zhou, Y.~Li, and Z.~Liu, ``Generalized out-of-distribution
  detection: A survey,'' \emph{arXiv preprint arXiv:2110.11334}, 2021.

\bibitem{wang2021generalizing}
J.~Wang, C.~Lan, C.~Liu, Y.~Ouyang, W.~Zeng, and T.~Qin, ``Generalizing to
  unseen domains: A survey on domain generalization,'' \emph{IEEE Transactions
  on Knowledge and Data Engineering (TKDE)}, 2022.

\bibitem{lehner20223d}
A.~Lehner, S.~Gasperini, A.~Marcos-Ramiro, M.~Schmidt, M.-A.~N. Mahani,
  N.~Navab, B.~Busam, and F.~Tombari, ``3d-vfield: Adversarial augmentation of
  point clouds for domain generalization in 3d object detection,'' in
  \emph{Proceedings of the IEEE/CVF Conference on Computer Vision and Pattern
  Recognition}, 2022, pp. 17\,295--17\,304.

\bibitem{vidit2023clip}
V.~Vidit, M.~Engilberge, and M.~Salzmann, ``Clip the gap: A single domain
  generalization approach for object detection,'' in \emph{Proceedings of the
  IEEE/CVF Conference on Computer Vision and Pattern Recognition}, 2023, pp.
  3219--3229.

\bibitem{yang2022openood}
J.~Yang, P.~Wang, D.~Zou, Z.~Zhou, K.~Ding, W.~Peng, H.~Wang, G.~Chen, B.~Li,
  Y.~Sun \emph{et~al.}, ``Openood: Benchmarking generalized out-of-distribution
  detection,'' \emph{Advances in Neural Information Processing Systems},
  vol.~35, pp. 32\,598--32\,611, 2022.

\bibitem{wang2022out}
Q.~Wang, J.~Ye, F.~Liu, Q.~Dai, M.~Kalander, T.~Liu, H.~Jianye, and B.~Han,
  ``Out-of-distribution detection with implicit outlier transformation,'' in
  \emph{The Eleventh International Conference on Learning Representations},
  2023.

\bibitem{wang2022out2}
Y.~Wang, J.~Zou, J.~Lin, Q.~Ling, Y.~Pan, T.~Yao, and T.~Mei,
  ``Out-of-distribution detection via conditional kernel independence model,''
  \emph{Advances in Neural Information Processing Systems}, vol.~35, pp.
  36\,411--36\,425, 2022.

\bibitem{ren2022out}
J.~Ren, J.~Luo, Y.~Zhao, K.~Krishna, M.~Saleh, B.~Lakshminarayanan, and P.~J.
  Liu, ``Out-of-distribution detection and selective generation for conditional
  language models,'' in \emph{The Eleventh International Conference on Learning
  Representations}, 2023.

\bibitem{fang2022out}
Z.~Fang, Y.~Li, J.~Lu, J.~Dong, B.~Han, and F.~Liu, ``Is out-of-distribution
  detection learnable?'' \emph{Advances in Neural Information Processing
  Systems}, vol.~35, pp. 37\,199--37\,213, 2022.

\bibitem{hendrycks2018deep}
D.~Hendrycks, M.~Mazeika, and T.~Dietterich, ``Deep anomaly detection with
  outlier exposure,'' in \emph{International Conference on Learning
  Representations}, 2019.

\bibitem{zhu2022boosting}
Y.~Zhu, Y.~Chen, C.~Xie, X.~Li, R.~Zhang, H.~Xue, X.~Tian, Y.~Chen
  \emph{et~al.}, ``Boosting out-of-distribution detection with typical
  features,'' \emph{Advances in Neural Information Processing Systems},
  vol.~35, pp. 20\,758--20\,769, 2022.

\bibitem{qian2021latent}
H.~Qian, S.~J. Pan, C.~Miao, H.~Qian, S.~Pan, and C.~Miao, ``Latent independent
  excitation for generalizable sensor-based cross-person activity
  recognition,'' in \emph{Proceedings of the AAAI Conference on Artificial
  Intelligence}, vol.~35, no.~13, 2021, pp. 11\,921--11\,929.

\bibitem{lu2022semantic}
W.~Lu, J.~Wang, Y.~Chen, S.~J. Pan, C.~Hu, and X.~Qin,
  ``Semantic-discriminative mixup for generalizable sensor-based cross-domain
  activity recognition,'' \emph{Proceedings of the ACM on Interactive, Mobile,
  Wearable and Ubiquitous Technologies}, vol.~6, no.~2, pp. 1--19, 2022.

\bibitem{kuznetsov2015learning}
V.~Kuznetsov and M.~Mohri, ``Learning theory and algorithms for forecasting
  non-stationary time series,'' \emph{Advances in neural information processing
  systems}, vol.~28, 2015.

\bibitem{lobov2018latent}
S.~Lobov, N.~Krilova, I.~Kastalskiy, V.~Kazantsev, and V.~A. Makarov, ``Latent
  factors limiting the performance of semg-interfaces,'' \emph{Sensors},
  vol.~18, no.~4, p. 1122, 2018.

\bibitem{liang2018enhancing}
S.~Liang, Y.~Li, and R.~Srikant, ``Enhancing the reliability of
  out-of-distribution image detection in neural networks,'' in
  \emph{International Conference on Learning Representations}, 2018.

\bibitem{dennis2019shallow}
D.~Dennis, D.~A.~E. Acar, V.~Mandikal, V.~S. Sadasivan, V.~Saligrama, H.~V.
  Simhadri, and P.~Jain, ``Shallow rnn: accurate time-series classification on
  resource constrained devices,'' \emph{Advances in Neural Information
  Processing Systems}, vol.~32, 2019.

\bibitem{dempster2021minirocket}
A.~Dempster, D.~F. Schmidt, and G.~I. Webb, ``Minirocket: A very fast (almost)
  deterministic transform for time series classification,'' in
  \emph{Proceedings of the 27th ACM SIGKDD conference on knowledge discovery \&
  data mining}, 2021, pp. 248--257.

\bibitem{nie2022time}
Y.~Nie, N.~H. Nguyen, P.~Sinthong, and J.~Kalagnanam, ``A time series is worth
  64 words: Long-term forecasting with transformers,'' in \emph{The Eleventh
  International Conference on Learning Representations}, 2023.

\bibitem{ismail2019deep}
H.~Ismail~Fawaz, G.~Forestier, J.~Weber, L.~Idoumghar, and P.-A. Muller, ``Deep
  learning for time series classification: a review,'' \emph{Data mining and
  knowledge discovery}, vol.~33, no.~4, pp. 917--963, 2019.

\bibitem{benidis2022deep}
K.~Benidis, S.~S. Rangapuram, V.~Flunkert, Y.~Wang, D.~Maddix, C.~Turkmen,
  J.~Gasthaus, M.~Bohlke-Schneider, D.~Salinas, L.~Stella \emph{et~al.}, ``Deep
  learning for time series forecasting: Tutorial and literature survey,''
  \emph{ACM Computing Surveys}, vol.~55, no.~6, pp. 1--36, 2022.

\bibitem{wen2022transformers}
Q.~Wen, T.~Zhou, C.~Zhang, W.~Chen, Z.~Ma, J.~Yan, and L.~Sun, ``Transformers
  in time series: A survey,'' \emph{arXiv preprint arXiv:2202.07125}, 2022.

\bibitem{pan2009survey}
S.~J. Pan and Q.~Yang, ``A survey on transfer learning,'' \emph{IEEE
  Transactions on knowledge and data engineering}, vol.~22, no.~10, pp.
  1345--1359, 2009.

\bibitem{silver2013lifelong}
D.~L. Silver, Q.~Yang, and L.~Li, ``Lifelong machine learning systems: Beyond
  learning algorithms,'' in \emph{2013 AAAI spring symposium series}, 2013.

\bibitem{zhou2022domain}
K.~Zhou, Z.~Liu, Y.~Qiao, T.~Xiang, and C.~C. Loy, ``Domain generalization: A
  survey,'' \emph{IEEE Transactions on Pattern Analysis and Machine
  Intelligence}, 2022.

\bibitem{shankar2018generalizing}
S.~Shankar, V.~Piratla, S.~Chakrabarti, S.~Chaudhuri, P.~Jyothi, and
  S.~Sarawagi, ``Generalizing across domains via cross-gradient training,'' in
  \emph{International Conference on Learning Representations}, 2018.

\bibitem{xu2021fourier}
Q.~Xu, R.~Zhang, Y.~Zhang, Y.~Wang, and Q.~Tian, ``A fourier-based framework
  for domain generalization,'' in \emph{Proceedings of the IEEE/CVF Conference
  on Computer Vision and Pattern Recognition}, 2021, pp. 14\,383--14\,392.

\bibitem{li2018deep}
Y.~Li, X.~Tian, M.~Gong, Y.~Liu, T.~Liu, K.~Zhang, and D.~Tao, ``Deep domain
  generalization via conditional invariant adversarial networks,'' in
  \emph{Proceedings of the European conference on computer vision (ECCV)},
  2018, pp. 624--639.

\bibitem{zhang2021deep}
X.~Zhang, P.~Cui, R.~Xu, L.~Zhou, Y.~He, and Z.~Shen, ``Deep stable learning
  for out-of-distribution generalization,'' in \emph{Proceedings of the
  IEEE/CVF Conference on Computer Vision and Pattern Recognition}, 2021, pp.
  5372--5382.

\bibitem{rame2022fishr}
A.~Rame, C.~Dancette, and M.~Cord, ``Fishr: Invariant gradient variances for
  out-of-distribution generalization,'' in \emph{International Conference on
  Machine Learning}, 2022, pp. 18\,347--18\,377.

\bibitem{kim2021selfreg}
D.~Kim, Y.~Yoo, S.~Park, J.~Kim, and J.~Lee, ``Selfreg: Self-supervised
  contrastive regularization for domain generalization,'' in \emph{Proceedings
  of the IEEE/CVF International Conference on Computer Vision}, 2021, pp.
  9619--9628.

\bibitem{sun2016deep}
B.~Sun and K.~Saenko, ``Deep coral: Correlation alignment for deep domain
  adaptation,'' in \emph{European conference on computer vision}.\hskip 1em
  plus 0.5em minus 0.4em\relax Springer, 2016, pp. 443--450.

\bibitem{peng2019domain}
X.~Peng, Z.~Huang, X.~Sun, and K.~Saenko, ``Domain agnostic learning with
  disentangled representations,'' in \emph{International Conference on Machine
  Learning}.\hskip 1em plus 0.5em minus 0.4em\relax PMLR, 2019, pp. 5102--5112.

\bibitem{zhang2022towards}
H.~Zhang, Y.-F. Zhang, W.~Liu, A.~Weller, B.~Sch{\"o}lkopf, and E.~P. Xing,
  ``Towards principled disentanglement for domain generalization,'' in
  \emph{Proceedings of the IEEE/CVF Conference on Computer Vision and Pattern
  Recognition}, 2022, pp. 8024--8034.

\bibitem{Matsuura2020DomainGU}
T.~Matsuura and T.~Harada, ``Domain generalization using a mixture of multiple
  latent domains,'' in \emph{AAAI}, 2020.

\bibitem{fan2021adversarially}
X.~Fan, Q.~Wang, J.~Ke, F.~Yang, B.~Gong, and M.~Zhou, ``Adversarially adaptive
  normalization for single domain generalization,'' in \emph{Proceedings of the
  IEEE/CVF Conference on Computer Vision and Pattern Recognition}, 2021, pp.
  8208--8217.

\bibitem{li2021progressive}
L.~Li, K.~Gao, J.~Cao, Z.~Huang, Y.~Weng, X.~Mi, Z.~Yu, X.~Li, and B.~Xia,
  ``Progressive domain expansion network for single domain generalization,'' in
  \emph{Proceedings of the IEEE/CVF Conference on Computer Vision and Pattern
  Recognition}, 2021, pp. 224--233.

\bibitem{wang2021learning}
Z.~Wang, Y.~Luo, R.~Qiu, Z.~Huang, and M.~Baktashmotlagh, ``Learning to
  diversify for single domain generalization,'' in \emph{ICCV}, 2021.

\bibitem{zhu2021crossmatch}
R.~Zhu and S.~Li, ``Crossmatch: Cross-classifier consistency regularization for
  open-set single domain generalization,'' in \emph{International Conference on
  Learning Representations}, 2022.

\bibitem{deecke2022visual}
L.~Deecke, T.~Hospedales, and H.~Bilen, ``Visual representation learning over
  latent domains,'' in \emph{International Conference on Learning
  Representations}, 2022.

\bibitem{wang2020continuously}
H.~Wang, H.~He, and D.~Katabi, ``Continuously indexed domain adaptation,'' in
  \emph{Proceedings of the 37th International Conference on Machine Learning},
  2020, pp. 9898--9907.

\bibitem{xu2023domain}
Z.~Xu, G.-Y. Hao, H.~He, and H.~Wang, ``Domain-indexing variational bayes:
  Interpretable domain index for domain adaptation,'' in \emph{The Eleventh
  International Conference on Learning Representations}, 2023.

\bibitem{rasmussen1999infinite}
C.~E. Rasmussen \emph{et~al.}, ``The infinite gaussian mixture model.'' in
  \emph{NIPS}, vol.~12.\hskip 1em plus 0.5em minus 0.4em\relax Citeseer, 1999,
  pp. 554--560.

\bibitem{koh2021wilds}
P.~W. Koh, S.~Sagawa, S.~M. Xie, M.~Zhang, A.~Balsubramani \emph{et~al.},
  ``Wilds: A benchmark of in-the-wild distribution shifts,'' in \emph{ICML},
  2021, pp. 5637--5664.

\bibitem{delage2010distributionally}
E.~Delage and Y.~Ye, ``Distributionally robust optimization under moment
  uncertainty with application to data-driven problems,'' \emph{Operations
  research}, vol.~58, no.~3, pp. 595--612, 2010.

\bibitem{sagawa2019distributionally}
S.~Sagawa, P.~W. Koh, T.~B. Hashimoto, and P.~Liang, ``Distributionally robust
  neural networks for group shifts: On the importance of regularization for
  worst-case generalization,'' in \emph{International Conference on Learning
  Representations (ICLR)}, 2020.

\bibitem{song2022rankfeat}
Y.~Song, N.~Sebe, and W.~Wang, ``Rankfeat: Rank-1 feature removal for
  out-of-distribution detection,'' \emph{Advances in Neural Information
  Processing Systems}, vol.~35, pp. 17\,885--17\,898, 2022.

\bibitem{djurisic2022extremely}
A.~Djurisic, N.~Bozanic, A.~Ashok, and R.~Liu, ``Extremely simple activation
  shaping for out-of-distribution detection,'' in \emph{The Eleventh
  International Conference on Learning Representations}, 2023.

\bibitem{ming2022exploit}
Y.~Ming, Y.~Sun, O.~Dia, and Y.~Li, ``How to exploit hyperspherical embeddings
  for out-of-distribution detection?'' in \emph{The Eleventh International
  Conference on Learning Representations}, 2023.

\bibitem{tao2022non}
L.~Tao, X.~Du, J.~Zhu, and Y.~Li, ``Non-parametric outlier synthesis,'' in
  \emph{The Eleventh International Conference on Learning Representations},
  2023.

\bibitem{yu2019unsupervised}
Q.~Yu and K.~Aizawa, ``Unsupervised out-of-distribution detection by maximum
  classifier discrepancy,'' in \emph{Proceedings of the IEEE/CVF international
  conference on computer vision}, 2019, pp. 9518--9526.

\bibitem{yang2021semantically}
J.~Yang, H.~Wang, L.~Feng, X.~Yan, H.~Zheng, W.~Zhang, and Z.~Liu,
  ``Semantically coherent out-of-distribution detection,'' in \emph{Proceedings
  of the IEEE/CVF International Conference on Computer Vision}, 2021, pp.
  8301--8309.

\bibitem{yang2023full}
J.~Yang, K.~Zhou, and Z.~Liu, ``Full-spectrum out-of-distribution detection,''
  \emph{International Journal of Computer Vision}, pp. 1--16, 2023.

\bibitem{das1998rule}
G.~Das, K.-I. Lin, H.~Mannila, G.~Renganathan, and P.~Smyth, ``Rule discovery
  from time series.'' in \emph{KDD}, vol.~98, no.~1, 1998, pp. 16--22.

\bibitem{zhang2021robust}
W.~Zhang, M.~Ragab, and R.~Sagarna, ``Robust domain-free domain generalization
  with class-aware alignment,'' in \emph{ICASSP 2021-2021 IEEE International
  Conference on Acoustics, Speech and Signal Processing (ICASSP)}.\hskip 1em
  plus 0.5em minus 0.4em\relax IEEE, 2021, pp. 2870--2874.

\bibitem{ragab2022conditional}
M.~Ragab, Z.~Chen, W.~Zhang, E.~Eldele, M.~Wu, C.-K. Kwoh, and X.~Li,
  ``Conditional contrastive domain generalization for fault diagnosis,''
  \emph{IEEE Transactions on Instrumentation and Measurement}, vol.~71, pp.
  1--12, 2022.

\bibitem{caron2018deep}
M.~Caron, P.~Bojanowski, A.~Joulin, and M.~Douze, ``Deep clustering for
  unsupervised learning of visual features,'' in \emph{Proceedings of the
  European Conference on Computer Vision (ECCV)}, 2018, pp. 132--149.

\bibitem{ganin2016domain}
Y.~Ganin, E.~Ustinova, H.~Ajakan, P.~Germain, H.~Larochelle, F.~Laviolette,
  M.~Marchand, and V.~Lempitsky, ``Domain-adversarial training of neural
  networks,'' \emph{The journal of machine learning research}, vol.~17, no.~1,
  pp. 2096--2030, 2016.

\bibitem{lee2018simple}
K.~Lee, K.~Lee, H.~Lee, and J.~Shin, ``A simple unified framework for detecting
  out-of-distribution samples and adversarial attacks,'' \emph{Advances in
  neural information processing systems}, vol.~31, 2018.

\bibitem{lasserre2006principled}
J.~A. Lasserre, C.~M. Bishop, and T.~P. Minka, ``Principled hybrids of
  generative and discriminative models,'' in \emph{2006 IEEE Computer Society
  Conference on Computer Vision and Pattern Recognition (CVPR'06)},
  vol.~1.\hskip 1em plus 0.5em minus 0.4em\relax IEEE, 2006, pp. 87--94.

\bibitem{murphy2012machine}
K.~P. Murphy, \emph{Machine learning: a probabilistic perspective}.\hskip 1em
  plus 0.5em minus 0.4em\relax MIT press, 2012.

\bibitem{hinton2015distilling}
G.~Hinton, O.~Vinyals, and J.~Dean, ``Distilling the knowledge in a neural
  network,'' \emph{arXiv preprint arXiv:1503.02531}, 2015.

\bibitem{ben2010theory}
S.~Ben-David, J.~Blitzer, K.~Crammer, A.~Kulesza, F.~Pereira, and J.~W.
  Vaughan, ``A theory of learning from different domains,'' \emph{Machine
  learning}, vol.~79, no.~1, pp. 151--175, 2010.

\bibitem{sicilia2021domain}
A.~Sicilia, X.~Zhao, and S.~J. Hwang, ``Domain adversarial neural networks for
  domain generalization: When it works and how to improve,'' \emph{arXiv
  preprint arXiv:2102.03924}, 2021.

\bibitem{parascandolo2020learning}
G.~Parascandolo, A.~Neitz, A.~Orvieto, L.~Gresele, and B.~Sch{\"o}lkopf,
  ``Learning explanations that are hard to vary,'' in \emph{ICLR}, 2021.

\bibitem{wang2019deep}
J.~Wang, Y.~Chen, S.~Hao, X.~Peng, and L.~Hu, ``Deep learning for sensor-based
  activity recognition: A survey,'' \emph{Pattern recognition letters}, vol.
  119, pp. 3--11, 2019.

\bibitem{paszke2019pytorch}
A.~Paszke, S.~Gross, F.~Massa, A.~Lerer, J.~Bradbury, G.~Chanan, T.~Killeen,
  Z.~Lin, N.~Gimelshein, L.~Antiga \emph{et~al.}, ``Pytorch: An imperative
  style, high-performance deep learning library,'' vol.~32, 2019, pp.
  8026--8037.

\bibitem{hendrycks2016baseline}
D.~Hendrycks and K.~Gimpel, ``A baseline for detecting misclassified and
  out-of-distribution examples in neural networks,'' in \emph{International
  Conference on Learning Representations}, 2017.

\bibitem{wilson2020multi}
G.~Wilson, J.~R. Doppa, and D.~J. Cook, ``Multi-source deep domain adaptation
  with weak supervision for time-series sensor data,'' in \emph{Proceedings of
  the 26th ACM SIGKDD International Conference on Knowledge Discovery \& Data
  Mining}, 2020, pp. 1768--1778.

\bibitem{purushotham2016variational}
S.~Purushotham, W.~Carvalho, T.~Nilanon, and Y.~Liu, ``Variational recurrent
  adversarial deep domain adaptation,'' 2016.

\bibitem{schmidt2018introducing}
P.~Schmidt, A.~Reiss, R.~Duerichen, C.~Marberger, and K.~Van~Laerhoven,
  ``Introducing wesad, a multimodal dataset for wearable stress and affect
  detection,'' in \emph{Proceedings of the 20th ACM international conference on
  multimodal interaction}, 2018, pp. 400--408.

\bibitem{barshan2014recognizing}
B.~Barshan and M.~C. Y{\"u}ksek, ``Recognizing daily and sports activities in
  two open source machine learning environments using body-worn sensor units,''
  \emph{The Computer Journal}, vol.~57, no.~11, pp. 1649--1667, 2014.

\bibitem{zhang2012usc}
M.~Zhang and A.~A. Sawchuk, ``Usc-had: a daily activity dataset for ubiquitous
  activity recognition using wearable sensors,'' in \emph{Proceedings of the
  2012 ACM conference on ubiquitous computing}, 2012, pp. 1036--1043.

\bibitem{anguita2012human}
D.~Anguita, A.~Ghio, L.~Oneto, X.~Parra, and J.~L. Reyes-Ortiz, ``Human
  activity recognition on smartphones using a multiclass hardware-friendly
  support vector machine,'' in \emph{International workshop on ambient assisted
  living}.\hskip 1em plus 0.5em minus 0.4em\relax Springer, 2012, pp. 216--223.

\bibitem{reiss2012introducing}
A.~Reiss and D.~Stricker, ``Introducing a new benchmarked dataset for activity
  monitoring,'' in \emph{2012 16th international symposium on wearable
  computers}.\hskip 1em plus 0.5em minus 0.4em\relax IEEE, 2012, pp. 108--109.

\bibitem{zhang2018mixup}
H.~Zhang, M.~Cisse, Y.~N. Dauphin, and D.~Lopez-Paz, ``mixup: Beyond empirical
  risk minimization,'' in \emph{International Conference on Learning
  Representations}, 2018.

\bibitem{huang2020self}
Z.~Huang, H.~Wang, E.~P. Xing, and D.~Huang, ``Self-challenging improves
  cross-domain generalization,'' in \emph{ECCV}, 2020, pp. 124--140.

\bibitem{warden2018speech}
P.~Warden, ``Speech commands: A dataset for limited-vocabulary speech
  recognition,'' \emph{arXiv preprint arXiv:1804.03209}, 2018.

\bibitem{kidger2020neuralcde}
P.~Kidger, J.~Morrill, J.~Foster, and T.~Lyons, ``{N}eural {C}ontrolled
  {D}ifferential {E}quations for {I}rregular {T}ime {S}eries,'' \emph{Advances
  in Neural Information Processing Systems}, 2020.

\bibitem{majumdar2020matchboxnet}
S.~Majumdar and B.~Ginsburg, ``Matchboxnet: 1d time-channel separable
  convolutional neural network architecture for speech commands recognition,''
  in \emph{Interspeech 2020, 21st Annual Conference of the International Speech
  Communication Association, Virtual Event, Shanghai, China, 25-29 October
  2020}.\hskip 1em plus 0.5em minus 0.4em\relax {ISCA}, 2020, pp. 3356--3360.

\bibitem{vaswani2017attention}
A.~Vaswani, N.~Shazeer, N.~Parmar, J.~Uszkoreit, L.~Jones, A.~N. Gomez,
  {\L}.~Kaiser, and I.~Polosukhin, ``Attention is all you need,''
  \emph{Advances in neural information processing systems}, vol.~30, 2017.

\bibitem{zhang2022delving}
C.~Zhang, M.~Zhang, S.~Zhang, D.~Jin, Q.~Zhou, Z.~Cai, H.~Zhao, X.~Liu, and
  Z.~Liu, ``Delving deep into the generalization of vision transformers under
  distribution shifts,'' in \emph{Proceedings of the IEEE/CVF Conference on
  Computer Vision and Pattern Recognition}, 2022, pp. 7277--7286.

\end{thebibliography}

\end{document}